\documentclass[10pt]{article} 
\usepackage[accepted]{tmlr}

\usepackage{amsmath,amsfonts,bm}

\def\eqref#1{equation~\ref{#1}}

\def\1{\bm{1}}

\DeclareMathAlphabet{\mathsfit}{\encodingdefault}{\sfdefault}{m}{sl}
\SetMathAlphabet{\mathsfit}{bold}{\encodingdefault}{\sfdefault}{bx}{n}

\newcommand{\E}{\mathbb{E}}

\newcommand{\R}{\mathbb{R}}

\newcommand{\KL}{D_{\mathrm{KL}}}

\usepackage{hyperref}
\usepackage{url}
\usepackage{amsthm}
\newtheorem{theorem}{Theorem}
\newtheorem{corollary}{Corollary}
\newtheorem{assumption}{Assumption}
\newtheorem{lemma}{Lemma}
\newtheorem{proposition}{Proposition}

\theoremstyle{definition}
\newtheorem{remark}{Remark}
\theoremstyle{plain}
\usepackage{algorithm}
\usepackage{algpseudocode}
\usepackage{graphicx}
\usepackage{amssymb}

\usepackage{multirow}
\usepackage{booktabs}

\usepackage{titletoc}
\usepackage{placeins}
\usepackage{changepage}

\newcommand{\printappendixtoc}{%
\begingroup
\section*{Table of Contents}

\begin{adjustwidth}{0.06\textwidth}{0.008\textwidth}
\hrule

\setcounter{tocdepth}{3}

\titlecontents{section}
  [0pt]
  {\addvspace{0.45em}\bfseries}
  {\contentslabel{2.8em}}
  {}
  {\titlerule*[0.6pc]{.}\contentspage}

\titlecontents{subsection}
  [1.5em]
  {}
  {\contentslabel{3.8em}}
  {}
  {\titlerule*[0.6pc]{.}\contentspage}

\titlecontents{subsubsection}
  [3.5em]
  {}
  {\contentslabel{4.8em}}
  {}
  {\titlerule*[0.6pc]{.}\contentspage}

\printcontents[appendix]{}{1}{\setcounter{tocdepth}{3}}
\end{adjustwidth}

\endgroup
}

\definecolor{darkblue}{RGB}{0,0,100}

\hypersetup{
  colorlinks=true,
  linkcolor=darkblue,
  citecolor=darkblue,
  urlcolor=blue
}

\title{PAC-Bayesian Meta-Learning \\for Few-Shot Identification of Linear Dynamical Systems
}

\author{\name Chenfeng Huang \email chenfenghuang@ucla.edu\\
      \addr Department of Statistics \& Data Science \\
      University of California, Los Angeles
      \AND
      \name George Michailidis \email gmichail@stat.ucla.edu \\
      \addr Department of Statistics \& Data Science \\
      University of California, Los Angeles
 }

\def\month{07}  
\def\year{2026} 
\def\openreview{\url{https://openreview.net/forum?id=CiGFpSLzFv}} 

\begin{document}

\maketitle
\begin{abstract}
Identifying linear time-invariant (LTI) dynamical systems from data is especially challenging when trajectories are short, noisy, or high-dimensional. Traditional system identification methods typically treat each system in isolation and therefore fail to exploit shared structure across related systems. We propose a PAC-Bayesian meta-learning framework for
few-shot LTI system identification (PBML-LTI), which learns a transferable prior over task-specific dynamics while preserving task-level heterogeneity. Each task corresponds to an unknown LTI system, and a meta-learner uses a collection of training trajectories to learn a data-dependent prior over transition matrices. Given a new system with limited
trajectory data, PBML-LTI performs Bayesian adaptation under the learned prior to produce a task-specific posterior, yielding both accurate point estimates and principled uncertainty quantification in the few-shot regime.

A key technical challenge is temporal dependence: trajectories generated by LTI systems violate the i.i.d.\ assumptions underlying most existing PAC-Bayes analyses for meta-learning. To address this, we develop a martingale PAC-Bayes analysis for dependent trajectory losses and use it to motivate a fit--KL surrogate objective for meta-training. The resulting
support--query predictive-risk bound clarifies how empirical fit, posterior complexity, and prior quality interact in few-shot adaptation under sequential dependence. We further show how this predictive bound induces problem-specific corollaries for transition-matrix recovery and multi-step trajectory prediction. Together, these results connect uncertainty-aware meta-identification with finite-sample analysis for dependent dynamical data.
\end{abstract}

\section{Introduction}

Linear time-invariant (LTI) dynamical systems are a foundational modeling framework in control, signal processing, and
time-series analysis. A central challenge in this setting is to recover the unknown transition operator from finite
observations. Although this problem can be written in a regression form, it differs fundamentally from standard linear
regression because the covariates are generated by the system itself. The resulting temporal dependence violates the
independence assumptions underlying classical statistical learning theory and makes estimation sensitive to both trajectory
length and system stability. Recent non-asymptotic analyses have clarified these challenges and established finite-sample
error guarantees across stable, marginally stable, and unstable regimes
\citep{faradonbeh2018finite,simchowitz2018learning,sarkar2019near}.

In many modern applications, however, system identification is not performed once for a single system, but repeatedly
across a collection of related systems, such as different devices, subjects, environments, operating modes, or geographic
regions. Such multi-system settings often exhibit shared latent structure that can be exploited to improve data efficiency
when individual trajectories are short or high-dimensional. Existing joint and multi-task approaches leverage this
structure by pooling information across systems, typically through shared-basis, low-dimensional, or representation-based
assumptions, in order to improve estimation accuracy over the observed collection of tasks
\citep{lin2024a,modi2024joint,chen2023multitask}.

Yet many downstream objectives require a different capability: rapid adaptation to a previously unseen system from only a
small amount of new data. This setting arises when deploying dynamical models across populations of related but
heterogeneous entities, where each system has its own local dynamics and only a short calibration window is available. For
example, one may need to tune a controller for a new aircraft operating condition or configuration
\citep{bosworth1992linearized}, or calibrate a forecasting model for a newly observed region or economic unit in
macroeconomic or financial monitoring settings \citep{pesaran2015time,stockwatson2016dynamic}. In such cases, the primary
challenge is no longer only accurate estimation on the observed systems, but reliable few-shot adaptation to new ones
under strict data constraints.

This challenge is especially important when long trajectories are expensive, unsafe, or infeasible to collect because of
changing operating conditions, nonstationarity, or limited observation horizons
\citep{garcia2015safesurvey,ji2023timevarying}. Similar constraints arise in physiological system identification, where
only short data segments may be practically available \citep{ludvig2012shortsegments}. These considerations motivate a
shift in perspective: rather than only pooling data to improve estimation on the training cohort, we seek to extract
transferable structure that supports principled few-shot adaptation and uncertainty quantification on previously unseen
systems.

To this end, we propose a Bayesian meta-learning framework for few-shot identification of linear dynamical systems. We
treat each system as a task drawn from a shared environment and use trajectories from multiple training tasks to learn a
data-dependent prior over system dynamics. When a new task is encountered, this learned prior enables rapid adaptation
from limited observations through Bayesian inference, yielding both a point estimate and a calibrated measure of
uncertainty. Unlike joint learning approaches that primarily use shared structure to improve estimation on the observed
tasks, our objective is to learn an inductive bias that is useful for adaptation to future systems
\citep{finn2017maml,patacchiola2020dkt,muthirayan2022meta}.

This formulation raises a central theoretical question: how can we characterize generalization to new systems when
adaptation is based on limited, temporally dependent data? PAC-Bayes theory provides a natural framework for this purpose,
as it relates generalization to a tradeoff between empirical fit and a complexity penalty measured by the
Kullback--Leibler divergence to a prior \citep{mcallester1999pacbayes,seeger2002pacbayes,catoni2007pacbayes}. This is
especially appealing in meta-learning, where the quality of the learned prior directly affects the speed and reliability
of adaptation to new tasks \citep{liu2021pacbayesmeta,rezazadeh2022unified,guan2022fastrate,farid2021pacbus}. However,
most existing PAC-Bayes meta-learning analyses assume that observations within each task are independent, whereas system
identification typically provides a single temporally dependent trajectory per task.

To bridge this gap, we develop a martingale PAC-Bayes analysis tailored to dependent dynamical data and use it to justify
a practical fit--KL surrogate for meta-training. The resulting theory yields a support--query predictive-risk perspective
for few-shot adaptation under temporal dependence and clarifies how empirical fit, posterior complexity, and prior quality
interact in this setting. This connects uncertainty-aware meta-identification with finite-sample analysis for dependent
dynamical systems.

Our contributions are summarized as follows:
\begin{itemize}
    \item We propose PBML-LTI, a PAC-Bayesian meta-learning framework for few-shot identification of linear
    time-invariant dynamical systems that learns a transferable prior and enables closed-form task adaptation via
    conjugate Bayesian inference.
    \item We develop a martingale PAC-Bayes analysis for temporally dependent trajectory losses and use it to motivate a
    fit--KL surrogate objective for meta-training.
    \item We show through synthetic and real fMRI-based experiments that PBML-LTI achieves improved data efficiency and
    robust multi-step prediction in high-dimensional, few-shot regimes.
\end{itemize}

\section{Related Work}
\label{sec:related_work}

\paragraph{Finite-sample identification of LTI systems from dependent trajectories.}
Classical system identification typically treats dynamical systems in isolation, focusing on estimation from a single trajectory where covariates are endogenously generated and thus temporally dependent. A rich body of non-asymptotic analysis has established sharp finite-time guarantees across stable, marginally stable, and unstable regimes, characterizing how sample complexity scales with the time horizon, system dimension, and spectral properties \citep{faradonbeh2018finite,simchowitz2018learning,sarkar2019near}. These foundational results underscore two critical constraints that our framework must address: the necessity of rigorously handling intra-task sequential dependence, and the dominant role of stability in determining estimation error when data are scarce.

\paragraph{Joint and multi-task identification across related systems.} 
Recent advances address the identification of multiple LTI systems by exploiting shared structural assumptions, such as common bases, low-dimensional subspaces, or coupled optimization objectives \citep{modi2024joint,chen2023multitask}. These joint-learning formulations excel at improving estimation accuracy across a fixed collection of observed systems by ``borrowing strength'' between tasks, aligning with classical multi-task learning paradigms \citep{argyriou2008convex}. In contrast, our framework targets few-shot generalization to previously \textit{unseen} systems. Here, the goal is not merely to refine estimates for the training set, but to leverage the learned structure to rapidly and reliably calibrate a new system from a short trajectory at deployment time.

\paragraph{Meta-learning, hierarchical Bayes, and Bayesian few-shot adaptation.}
Meta-learning formalizes the problem of learning transferable structure that supports rapid adaptation to new tasks drawn from a shared environment. Gradient-based approaches such as MAML optimize performance after a small number of task-level updates and were studied in both supervised and reinforcement-learning settings \citep{finn2017maml}. Thus, the mere presence of non-i.i.d. data within a task is not, by itself, unique to our setting. A complementary perspective interprets meta-learning through hierarchical Bayes, where training tasks are used to learn a prior and test-time adaptation corresponds to posterior inference \citep{grant2018recasting}. Several Bayesian meta-learning methods instantiate this perspective in general few-shot supervised settings, including PLATIPUS \citep{finn2018platipus}, ABML \citep{ravi2019abml}, and iBAML \citep{zhang2023ibaml}; related uncertainty-aware meta-learning ideas have also been explored in settings such as deep kernel transfer \citep{patacchiola2020dkt} and online control \citep{muthirayan2022meta}. Our approach is complementary but tailored to few-shot LTI identification from temporally dependent trajectories. In particular, we exploit a conjugate matrix-normal/Gaussian model so that task adaptation is available in closed form, rather than through iterative gradient-based, amortized, or implicit-gradient updates.

\paragraph{PAC-Bayes generalization and meta-learning.}
PAC-Bayes theory bounds generalization error through a tradeoff between empirical fit and a complexity term measured by the KL divergence to a reference prior \citep{mcallester1999pacbayes,seeger2002pacbayes,catoni2007pacbayes}. This makes it particularly natural for meta-learning, where the quality of the learned prior directly affects adaptation to new tasks \citep{liu2021pacbayesmeta,rezazadeh2022unified,guan2022fastrate,farid2021pacbus}. Most existing PAC-Bayes meta-learning analyses, however, assume that observations within each task are independent. By contrast, LTI system identification provides a single temporally dependent trajectory per task. Relative to single-task non-asymptotic analyses for dependent LTI data \citep{simchowitz2018learning,sarkar2019near}, the additional challenge in our setting is to move from guarantees for point estimators on one system to a support--query meta-learning problem involving task-specific posterior distributions, a learned prior shared across tasks, and a PAC-Bayes change-of-measure argument under temporal dependence. Our contribution is therefore not simply to handle non-i.i.d. within-task data, but to combine few-shot LTI identification, hierarchical Bayesian adaptation, closed-form conjugate task inference, and martingale PAC-Bayes analysis in a unified framework.


\paragraph{Joint learning versus meta-learning in the experimental evaluation.}
While both joint learning and meta-learning evaluate performance on unseen tasks, they differ fundamentally in their
training objectives and adaptation mechanisms. Joint learning typically extracts a static shared structure (e.g., a common
subspace) from training tasks, which is then applied ``frozen'' or with minimal adjustment to new data
\citep{modi2024joint,chen2023multitask}. In contrast, our experimental protocol \textit{explicitly targets few-shot
adaptation}. We adopt a support--query split: for every test task, the model receives only a short support prefix to
calibrate a task-specific posterior distribution, while evaluation is conducted on a disjoint query window. This
distinction is crucial; unlike joint learning, which optimizes aggregate performance over a cohort, our meta-learning
objective is designed to maximize post-adaptation accuracy by leveraging the learned prior to recover dynamics from
minimal data.
\section{Problem Formulation}
\label{sec:problem_formulation}
We consider the problem of identifying a family of related LTI dynamical systems. The available data comprise trajectories collected from multiple tasks, indexed by $m \in [M] := \{1,2,\dots,M\}$.
For each task $m$ and time $t = 0,1,\dots,T_m-1$, the system evolves according to
\begin{equation}
x_{m,t+1} = A_m x_{m,t} + w_{m,t+1},
\label{eq:lti_dynamics}
\end{equation}
where $x_{m,t} \in \R^d$ denotes the system state, $A_m \in \R^{d\times d}$ is an unknown transition matrix, and $\{w_{m,t}\}_{t\ge 1}$ is a stochastic noise process. We assume the full state trajectory $\{x_{m,t}\}_{t=0}^{T_m}$ is observed for each task.

It is convenient to rewrite \eqref{eq:lti_dynamics} in  regression form. To this end, define the per-task data matrices
\begin{equation}
X_m := \big[x_{m,0},x_{m,1},\dots,x_{m,T_m-1}\big]\in\R^{d\times T_m},
\qquad
Y_m := \big[x_{m,1},x_{m,2},\dots,x_{m,T_m}\big]\in\R^{d\times T_m},
\label{eq:LTI}
\end{equation}
where $X_m$ collects the state vectors serving as predictors at each time step, and $Y_m$ collects the responses, corresponding to the next-step-ahead states. Let $\Xi_m := [w_{m,1},\dots,w_{m,T_m}] \in \R^{d\times T_m}$ denote the stacked process-noise matrix. Then, the
trajectory can be written compactly as
\begin{equation}
Y_m = A_m X_m + \Xi_m.
\label{eq:matrix_regression}
\end{equation}
A critical distinction from standard linear regression is that the columns of $X_m$ exhibit inherent temporal dependence and are adapted to the history of the noise process.

To complete the model formulation, we impose the following assumptions to ensure that the model is well posed and to clearly specify the relationships among the members of the family of LTIs in \eqref{eq:lti_dynamics}.

\begin{assumption}[\textbf{Controlled Growth}]\label{as1}
We assume the dynamics of each system exhibit controlled growth over the observed horizon.
Specifically, there exists a constant $c_\rho>0$ such that for all tasks $m$, 
\begin{equation}
\rho(A_m) \le 1 + c_\rho/T_m,
\label{eq:stability}
\end{equation}

where $\rho(A_m)$ denotes the spectral radius of $A_m$. This type of controlled-growth condition allows near-marginal stability while preventing rapid explosion over the finite horizon, which is standard in non-asymptotic LTI identification and is also adopted in single and multi-system learning settings
\citep{faradonbeh2018finite,modi2024joint,simchowitz2018learning,sarkar2019near}.
\end{assumption}

\begin{assumption}[\textbf{Martingale Difference Noise Process}]\label{as2}

Let $\{\mathcal{F}_{m,t}\}_{t\ge 0}$ denote the natural filtration for task $m$, capturing the history of information available up to time $t$. Formally, $\mathcal{F}_{m,t}$ is the smallest $\sigma$-algebra generated by the initial condition and the noise history up to time $t$: \[ \mathcal{F}_{m,t}:=\sigma(x_{m,0},w_{m,1},\dots,w_{m,t}). \] We assume the noise process forms a martingale difference sequence with respect to this filtration and is conditionally sub-Gaussian. That is,
\begin{equation}
    \E[w_{m,t+1} \mid \mathcal{F}_{m,t}] = 0,
\end{equation}
and for all $u \in \mathbb{R}^d$,
\begin{equation}
    \E\left[\exp(\langle u, w_{m,t+1}\rangle) \mid \mathcal{F}_{m,t}\right] \le \exp\left(\frac{\sigma_w^2}{2}\|u\|_2^2\right).
    \label{eq:subgaussian_noise}
\end{equation}
Additionally, we assume the conditional covariance is uniformly bounded by a fixed positive semidefinite matrix, satisfying $\E[w_{m,t+1}w_{m,t+1}^\top \mid \mathcal{F}_{m,t}] \preceq \Sigma_w$. These conditions place our analysis within the self-normalized martingale framework, a standard setting for analyzing sequential least squares and LTI identification from single trajectories \citep{abbasi2011online,simchowitz2018learning,sarkar2019near}.
\end{assumption}

\begin{assumption}[\textbf{Shared Environment}]\label{as3}
We posit a hierarchical generative process where tasks are related via a shared latent structure. Specifically, there exists an unknown hyper-parameter $\phi_\star$ governing the distribution of system dynamics, such that each task parameter $A_m$ is drawn independently:
\begin{equation}
    A_m \mid \phi_\star \stackrel{\mathrm{i.i.d.}}{\sim} p(\cdot\mid\phi_\star), \qquad m=1,\dots,M.
    \label{eq:task_distribution}
\end{equation}
Conditional on $A_m$, the observed trajectories are generated according to the LTI dynamics in \eqref{eq:lti_dynamics}. This assumption of i.i.d.\ task sampling from a fixed meta-distribution is standard in learning-to-learn theory and serves as the foundation for establishing generalization guarantees to novel tasks \citep{baxter2000model,maurer2005algorithmic,finn2017maml,liu2021pacbayesmeta}.
\end{assumption}

\section{PAC-Bayesian Meta-Learning for LTI Identification (PBML-LTI)}
\label{sec:bayes_meta_lti}

In this section, we introduce PBML-LTI, progressing from the hierarchical Bayesian formulation to the practical
meta-training objective and the corresponding few-shot adaptation procedure. We build directly on the setup in
Section~\ref{sec:problem_formulation}: Assumption~\ref{as3} motivates learning shared structure across tasks through a
common prior, while Assumptions~\ref{as1}--\ref{as2} provide the regularity conditions needed for the martingale
PAC-Bayes analysis developed later in this section.

\subsection{Hierarchical Bayesian Model and Task Adaptation}
\label{subsec:hb_model_adapt}

We treat each LTI system as a task with task-specific transition matrix $A_m$, and we seek shared meta-parameters $\phi$
that define a transferable prior over the family $\{A_m\}$. When a new task is encountered with limited data, adaptation
is performed via Bayesian inference under this learned prior, producing both a point estimate and a task-specific measure
of uncertainty~\citep{grant2018recasting}.

\paragraph{Hierarchical Bayes interpretation of meta-learning.}
Our formulation follows a three-level hierarchical Bayesian model:
\begin{equation}
\phi \sim p(\phi),\qquad
A_m \mid \phi \sim p(A\mid \phi),\qquad
D_m \mid A_m \sim p(D\mid A_m),
\end{equation}
where $\phi$ represents the environment-level structure shared across tasks, $A_m$ specifies the task-specific dynamics,
and $D_m$ denotes the observed trajectory for task $m$. In a fully Bayesian treatment one would place a hyper-prior on
$\phi$ and infer both $\phi$ and $\{A_m\}$ jointly. In contrast, we adopt an empirical-Bayes perspective: $\phi$ is
learned across training tasks, while task-level adaptation on a new system is carried out by exact posterior updating for
$A_m$ conditional on the learned $\phi$. This retains the statistical advantages of a shared prior while avoiding the
computational overhead of full posterior inference over $\phi$~\citep{grant2018recasting}.

\paragraph{Prior Distribution And Working Likelihood.}
Consistent with Assumption~\ref{as3}, we posit a shared prior family $P_\phi$ over task parameters and instantiate it as
a matrix-normal distribution~\citep{dawid1981matrix,gupta2018matrix,west1997bayesian}:
\begin{equation}
A_m \mid \phi ~\sim~ \mathcal{MN}\!\left(W,\; I_d,\; V\right),
\qquad \phi := (W,V,\sigma^2),
\label{eq:mn_prior}
\end{equation}
where $W\in\R^{d\times d}$ is the shared prior mean, and $V\in\R^{d\times d}$ is symmetric positive definite. Equivalently,
\[
\mathrm{vec}(A_m)\sim \mathcal{N}(\mathrm{vec}(W),\,V\otimes I_d),
\]
so $V$ controls how the entries of $A_m$ co-vary across columns. This structure provides a convenient and expressive way
to couple tasks through a shared mean and covariance, while still admitting efficient closed-form adaptation.

For tractable task-level inference, we adopt a conditional Gaussian working likelihood:
\begin{equation}
x_{m,t+1}\mid x_{m,t},A_m ~\sim~ \mathcal{N}\!\left(A_m x_{m,t},\,\sigma^2 I_d\right),
\qquad t=0,\dots,T_m-1.
\label{eq:gaussian_likelihood}
\end{equation}

This matrix-normal/Gaussian specification is adopted not merely because it is conjugate, but because it matches the
structure of the problem. The transition parameter $A_m$ is intrinsically matrix-valued, and the matrix-normal prior
preserves this structure directly rather than imposing it only after vectorization. Under the Gaussian working likelihood
in \eqref{eq:gaussian_likelihood}, it yields closed-form task adaptation and exact expressions for the posterior mean and
covariance, the posterior-expected squared loss, and the posterior-to-prior KL divergence used by PBML-LTI. At the same
time, it provides an interpretable and computationally efficient structured prior: $W$ captures average dynamics shared
across tasks, while $V$ encodes cross-column variability in a compact form. Richer conjugate or nonconjugate priors are
possible, but they would substantially complicate the inner-loop posterior updates and the resulting fit--KL objective.

\paragraph{Closed-Form Task Adaptation.}
Given task data $X_m,Y_m$ from the matrix regression model \eqref{eq:matrix_regression}, combining the Gaussian working
likelihood \eqref{eq:gaussian_likelihood} with the matrix-normal prior \eqref{eq:mn_prior} yields a conjugate
matrix-normal posterior~\citep{dawid1981matrix,gupta2018matrix,west1997bayesian}:
\begin{equation}
Q_{m,\phi}=A_m \mid D_m,\phi ~\sim~ \mathcal{MN}\!\left(M_m,\; I_d,\; V_m\right).
\label{eq:mn_posterior}
\end{equation}
The posterior covariance and mean are
\begin{align}
V_m
&= \left(V^{-1} + \frac{1}{\sigma^2} X_m X_m^\top \right)^{-1},
\label{eq:posterior_cov}
\\
M_m
&= \left(\frac{1}{\sigma^2} Y_m X_m^\top + W V^{-1}\right)V_m.
\label{eq:posterior_mean}
\end{align}
We use the posterior mean $\widehat A_m:=M_m$ as the task-level point estimate, and $V_m$ as the task-level uncertainty
measure.

The same conjugate structure also yields closed-form expressions for the posterior expectation of the Gaussian loss and
for the posterior-to-prior KL divergence. In particular,
\begin{equation}
\E_{A\sim Q_{m,\phi}}\!\left[\|Y_m-AX_m\|_F^2\right]
=
\|Y_m-M_mX_m\|_F^2
+
d\,\mathrm{tr}(V_m X_mX_m^\top),
\label{eq:posterior_sq_loss_closed_form}
\end{equation}
and
\begin{equation}
\KL(Q_{m,\phi}\|P_\phi)
=
\frac12
\left[
d\Big(\mathrm{tr}(V^{-1}V_m)-\log\det(V^{-1}V_m)-d\Big)
+
\mathrm{tr}\!\big((M_m-W)V^{-1}(M_m-W)^\top\big)
\right].
\label{eq:matrix_normal_kl_closed_form}
\end{equation}
Thus, both the posterior expectation of the loss and the KL term can be computed exactly through matrix products, linear
solves, traces, and log-determinants, without Monte Carlo approximation.

Closed-form adaptation therefore specifies how a learned prior $P_\phi$ is converted into a task-specific posterior
$Q_{m,\phi}$. The remaining question is how to learn $\phi$ from many related training tasks in a way that supports
generalization to unseen systems. To address this under temporal dependence, we next use the martingale PAC-Bayes
structure to motivate a practical fit--KL surrogate for meta-training.

\subsection{Martingale PAC-Bayes Motivation for Meta-Training}
\label{sec:mpb_motivation}

Within each task, observations form a single trajectory and are therefore not i.i.d. Classical PAC-Bayes results for
independent samples do not apply directly because the one-step losses depend on the past through the state. A natural way
to handle this dependence is to work with conditional risks and exploit the martingale structure induced by the filtration
in Assumption~\ref{as2}.

For the remainder of this subsection, we fix an arbitrary task $m$ and suppress the task index for notational simplicity.
Write $x_t:=x_{m,t}$ and $T:=T_m$. The trajectory $x_0,\dots,x_T$ induces $T$ one-step prediction losses.

We first state two supporting technical ingredients and then the main martingale PAC-Bayes result. Theorem~\ref{thm:inst_nll}
identifies the instantaneous Gaussian log-loss and the associated martingale-difference structure. Theorem~\ref{thm:exp_supermg_psi}
uses this structure together with an MGF bound to construct an exponential supermartingale. Theorem~\ref{thm:mpb_with_psi}
then applies a PAC-Bayes change-of-measure argument to obtain the predictive-risk bound that motivates the fit--KL
surrogate used by PBML-LTI.

\begin{theorem}[\textbf{Instantaneous Gaussian Log-Loss}]\footnote{The proof is included in Appendix~\ref{sec:proof1}.}
\label{thm:inst_nll}
Fix $\sigma^2>0$ and consider the Gaussian conditional likelihood model
$p_\sigma(x_t\mid x_{t-1},A)=\mathcal{N}(A x_{t-1},\sigma^2 I_d)$.
The associated instantaneous log-loss is
\begin{equation}
\ell_t(A)
:= -\log p_\sigma(x_t\mid x_{t-1},A)
= \frac{1}{2\sigma^2}\|x_t-Ax_{t-1}\|_2^2+\frac{d}{2}\log(2\pi\sigma^2),
\qquad t=1,\dots,T .
\label{eq:inst_nll}
\end{equation}
Let $\{\mathcal{F}_t\}_{t=0}^T$ be any filtration such that $\ell_t(A)$ is $\mathcal{F}_t$-measurable, and define
\begin{equation}
d_t(A):=\E\!\left[\ell_t(A)\mid \mathcal{F}_{t-1}\right]-\ell_t(A).
\label{eq:dt_def}
\end{equation}
Then for every fixed $A$, $\{d_t(A)\}_{t=1}^T$ is a martingale difference sequence:
$\E[d_t(A)\mid\mathcal{F}_{t-1}]=0$ for all $t$.
\end{theorem}

\begin{theorem}[\textbf{Exponential Supermartingale}]\footnote{The proof is included in Appendix~\ref{sec:proof2}.}
\label{thm:exp_supermg_psi}
Suppose that for all $\lambda_{\mathrm{PB}}\in(0,1]$ and all $t=1,\dots,T$,
\begin{equation}
\E\!\left[\exp\!\big(\lambda_{\mathrm{PB}} d_t(A)\big)\mid \mathcal{F}_{t-1}\right]
\le \exp\!\big(\psi_t(\lambda_{\mathrm{PB}})\big).
\label{eq:mgf_bound}
\end{equation}
Let $S_t(A):=\sum_{i=1}^t d_i(A)$ and define
\begin{equation}
Z_t(A;\lambda_{\mathrm{PB}})
:= \exp\!\left(\lambda_{\mathrm{PB}} S_t(A)-\sum_{i=1}^t \psi_i(\lambda_{\mathrm{PB}})\right),
\qquad t=0,1,\dots,T,
\label{eq:zt_def}
\end{equation}
with $Z_0(A;\lambda_{\mathrm{PB}})=1$. Then $\{Z_t(A;\lambda_{\mathrm{PB}})\}_{t=0}^T$ is a nonnegative supermartingale with respect to
$\{\mathcal{F}_t\}_{t=0}^T$.
\end{theorem}

We continue to write $\ell_t(\cdot):=\ell_{m,t}(\cdot)$ and $\mathcal{F}_t:=\mathcal{F}_{m,t}$.

\begin{theorem}[\textbf{Martingale PAC-Bayes Bound}]\footnote{The proof is included in Appendix~\ref{sec:proof3}.}
\label{thm:mpb_with_psi}
Let $P$ be any prior distribution over $A$ independent of the trajectory, and let $Q$ be any posterior distribution over
$A$. Under the setting of Theorems~\ref{thm:inst_nll} and \ref{thm:exp_supermg_psi}, for any $\delta\in(0,1)$ and any
$\lambda_{\mathrm{PB}}\in(0,1]$, with probability at least $1-\delta$,
\begin{equation}
\frac{1}{T}\sum_{t=1}^T \E_{A\sim Q}\!\left[\E\!\left[\ell_t(A)\mid \mathcal{F}_{t-1}\right]\right]
\le
\frac{1}{T}\sum_{t=1}^T \E_{A\sim Q}\!\left[\ell_t(A)\right]
+
\frac{1}{\lambda_{\mathrm{PB}} T}
\Big(\KL(Q\|P)+\log\tfrac{1}{\delta}+\sum_{t=1}^T \psi_t(\lambda_{\mathrm{PB}})\Big).
\label{eq:mpb_with_psi}
\end{equation}
\end{theorem}

Theorems~\ref{thm:inst_nll} and \ref{thm:exp_supermg_psi} are auxiliary technical ingredients for the main result,
Theorem~\ref{thm:mpb_with_psi}. Specifically, Theorem~\ref{thm:inst_nll} shows that the gap between conditional
predictive loss and realized loss forms a martingale difference sequence, while
Theorem~\ref{thm:exp_supermg_psi} turns an MGF control on this sequence into an exponential supermartingale.
Theorem~\ref{thm:mpb_with_psi} then applies a PAC-Bayes change-of-measure argument to lift this concentration statement
from a fixed $A$ to a posterior distribution $Q$ over $A$, which is exactly the object produced by Bayesian adaptation.

For later reference, define the conditional and empirical trajectory risks
\begin{equation}
L_T(Q)
:=
\frac{1}{T}\sum_{t=1}^T \E_{A\sim Q}\!\left[\E\!\left[\ell_t(A)\mid \mathcal{F}_{t-1}\right]\right],
\qquad
\widehat{L}_T(Q)
:=
\frac{1}{T}\sum_{t=1}^T \E_{A\sim Q}\!\left[\ell_t(A)\right].
\label{eq:traj_risks_def}
\end{equation}
Thus, Theorem~\ref{thm:mpb_with_psi} gives a generic predictive-risk statement for dependent trajectories: with high
probability, the conditional predictive risk is controlled by empirical in-trajectory fit, a KL-to-prior complexity term,
and a martingale concentration penalty.

In our hierarchical model, we take $P=P_\phi$ to be the learned matrix-normal prior \eqref{eq:mn_prior}, and we restrict
$Q$ to be the conjugate posterior $Q_{m,\phi}$ in \eqref{eq:mn_posterior}. At this stage, however, the result is still a
generic predictive-risk statement. In the next subsection, we use this structure to motivate a practical fit--KL
surrogate for meta-training. Later, in Section~\ref{theoretical_analysis}, we return to this result and specialize it to
the support--query setting used in few-shot evaluation.

    \subsection{Meta-Training Objective}
    \label{subsec:meta_objective}
    
    The martingale PAC-Bayes bound in Theorem~\ref{thm:mpb_with_psi} contains the term
$\sum_{t=1}^T\psi_t(\lambda_{\mathrm{PB}})$, which captures the conditional fluctuation scale of the centered loss
increments $d_t(A)$. Under a uniform conditional MGF control, this term becomes explicit and independent of the
meta-parameters that define the prior, so it influences the bound only through constants and through the global weighting
of the KL term.

\begin{assumption}[\textbf{Uniform moment generating function (MGF) control}]
\label{as4}
Recall $d_t(A):=\E[\ell_t(A)\mid\mathcal{F}_{t-1}]-\ell_t(A)$. We assume there exists a constant $v>0$ such that, for all
$t\in\{1,\dots,T\}$, all $A$, and all $\lambda_{\mathrm{PB}}\in(0,1]$,
\begin{equation}
\log \E\!\left[
\exp\!\big(\lambda_{\mathrm{PB}} d_t(A)\big)
\,\middle|\,
\mathcal{F}_{t-1}
\right]
\le
\frac{\lambda_{\mathrm{PB}}^2 v}{2}.
\label{eq:uniform_mgf_control}
\end{equation}
\end{assumption}

Equivalently, \eqref{eq:mgf_bound} holds with
$\psi_t(\lambda_{\mathrm{PB}})=\lambda_{\mathrm{PB}}^2v/2$ for
$\lambda_{\mathrm{PB}}\in(0,1]$, so that
\[
\sum_{t=1}^T \psi_t(\lambda_{\mathrm{PB}})
\le
\frac{\lambda_{\mathrm{PB}}^2vT}{2}.
\]
This is the range of $\lambda_{\mathrm{PB}}$ used in the PAC-Bayes theorem and corollaries below. The constant $v$ is a
valid conditional MGF variance proxy: choosing a larger valid $v$ preserves the inequality but loosens the bound, whereas
choosing $v$ too small may invalidate the MGF condition.

\paragraph{Self-normalization and bound tightness.}
The quantity $vT$ can be interpreted as a uniform predictable fluctuation scale for the centered martingale loss process
$d_t(A)$. A sharper fully self-normalized analysis could replace this uniform proxy by a predictable quadratic-variation
or data-dependent variance term. We use the simpler uniform formulation to keep the PAC-Bayes statement transparent and
directly compatible with the fit--KL surrogate. This also clarifies when the bound is expected to be tight: it is tighter
when trajectories remain controlled, posterior mass concentrates on predictors with small one-step residuals, and
conditional loss fluctuations are small. Conversely, if state norms or prediction errors grow, as in unstable systems,
the predictable variance scale becomes large and any valid uniform choice of $v$ necessarily yields a looser bound.

Under Assumption~\ref{as4}, Theorem~\ref{thm:mpb_with_psi} yields a generic predictive-risk bound of the form
\begin{equation}
L_T(Q)
\le
\widehat{L}_T(Q)
+
\frac{1}{\lambda_{\mathrm{PB}}T}\KL(Q\|P)
+
\text{terms depending only on }(\delta,\lambda_{\mathrm{PB}},v,T),
\qquad
\lambda_{\mathrm{PB}}\in(0,1].
\label{eq:generic_pb_shape}
\end{equation}
This result is \emph{not} used as the exact optimization objective in training. Instead, it provides the structural
motivation for the surrogate adopted by PBML-LTI. For fixed $(\delta,\lambda_{\mathrm{PB}},v)$ and $T$, the dependence
on the meta-parameters $\phi$ enters through the empirical fit $\widehat{L}_T(Q)$ and the KL term $\KL(Q\|P)$, suggesting
a fit--KL tradeoff as the core training criterion.

In the main implementation, we use the canonical choice corresponding to $\lambda_{\mathrm{PB}}=1$, which yields the
standard fit--KL tradeoff, preserves the usual conjugate Bayesian update, and avoids introducing an additional
bound-specific tuning parameter in meta-training. In the experiments, we also study an implementation-level training
temperature $\lambda_{\mathrm{tr}}>0$ that rescales the fit--KL tradeoff. This empirical temperature should be
distinguished from the PAC-Bayes parameter $\lambda_{\mathrm{PB}}$: the formal bound is stated only for
$\lambda_{\mathrm{PB}}\in(0,1]$, while values $\lambda_{\mathrm{tr}}>1$ are included only as practical stress tests of
more data-driven adaptation. We report this sensitivity study in Section~\ref{subsec:lambda_sensitivity}.
    For each training task $m\in[M_{\mathrm{tr}}]$, we observe a trajectory $D_m=\{x_{m,t}\}_{t=0}^{T_m}$ and its regression
    matrices $(X_m,Y_m)$ defined in \eqref{eq:LTI}. Given the current meta-parameters
    $\phi=(W,V,\sigma^2)$, we form the conjugate posterior $Q_{m,\phi}$ in \eqref{eq:mn_posterior} with parameters
    $(M_m,V_m)$ from \eqref{eq:posterior_cov} and \eqref{eq:posterior_mean}. We then define the per-task empirical fit term
    \begin{equation}
    \mathcal{L}_{\mathrm{fit},m}(\phi)
    :=
    \frac{1}{T_m}\E_{A\sim Q_{m,\phi}}\!\left[\sum_{t=1}^{T_m}\ell_{m,t}(A)\right]
    =
    \frac{1}{2\sigma^2T_m}
    \left[
    \|Y_m-M_mX_m\|_F^2
    +
    d\,\mathrm{tr}(V_m X_mX_m^\top)
    \right]
    +\frac{d}{2}\log(2\pi\sigma^2),
    \label{eq:fit_term_def}
    \end{equation}
    where the equality follows from \eqref{eq:posterior_sq_loss_closed_form}. Likewise, the per-task complexity term is
    \begin{equation}
    \begin{aligned}
    \mathcal{L}_{\mathrm{kl},m}(\phi)
    :=
    \frac{1}{T_m}\KL\!\left(Q_{m,\phi}\,\|\,P_\phi\right)
    =
    \frac{1}{2T_m}
    \Big[\,&
    d\Big(\mathrm{tr}(V^{-1}V_m)-\log\det(V^{-1}V_m)-d\Big)
    \\
    &+
    \mathrm{tr}\!\big((M_m-W)V^{-1}(M_m-W)^\top\big)
    \Big].
    \end{aligned}
    \label{eq:kl_term_def}
    \end{equation}
    
    Motivated by the PAC-Bayes structure above, we optimize the following surrogate:
    \begin{equation}
    \min_{\phi}\;
    \frac{1}{M_{\mathrm{tr}}}\sum_{m=1}^{M_{\mathrm{tr}}}
    \left[
    \mathcal{L}_{\mathrm{fit},m}(\phi)
    +
    \mathcal{L}_{\mathrm{kl},m}(\phi)
    \right]
    \;+\;
    \mathcal{R}_{\mathrm{hyper}}(W,V)
    \;+\;
    \mathcal{R}_{\mathrm{stab}}(W;\rho_0),
    \label{eq:meta_train_objective_repo}
    \end{equation}
    where the additional regularizers are practical optimization and prior-shaping terms rather than components of the formal PAC-Bayes theorem:
    \begin{equation}
    \mathcal{R}_{\mathrm{hyper}}(W,V)
    :=
    \frac{1}{2\tau_W^2}\|W\|_F^2
    +
    \lambda_V\left(\frac{1}{2}\|V-I_d\|_F^2-\log\det(V)\right),
    \label{eq:hyper_reg_repo}
    \end{equation}
    and
    \begin{equation}
    \mathcal{R}_{\mathrm{stab}}(W;\rho_0)
    :=
    \big(\max\{0,\rho(W)-\rho_0\}\big)^2.
    \label{eq:stab_reg_repo}
    \end{equation}
  Here $\tau_W>0$ controls shrinkage on $W$, $\lambda_V\ge0$ encourages a well-scaled and well-conditioned $V$, and
$\rho_0\in(0,1)$ is a target stability margin for the shared mean dynamics. The stability regularizer should not be
interpreted as directly enforcing Assumption~\ref{as1}, which is a data-generating condition on the true task matrices
$A_m$. Instead, it acts on the learned shared prior mean and encourages the prior to concentrate near dynamically
well-behaved transition matrices, which can indirectly improve posterior adaptation and numerical stability.
    
    To ensure $V$ remains symmetric positive definite during optimization, we parameterize
    \begin{equation}
    V = LL^\top + \epsilon I_d,
    \label{eq:chol_param}
    \end{equation}
    where $L$ is learnable lower-triangular and $\epsilon>0$ is a small diagonal shift.
    
    With the objective \eqref{eq:meta_train_objective_repo} in place, meta-training reduces to repeatedly
    (i) computing conjugate posteriors for minibatches of tasks in closed form and
    (ii) updating $\phi$ through stochastic gradients of the resulting fit--KL surrogate.

\subsection{Task-Level Inference and Adaptation}
\label{sec:inference}

After meta-training, we perform task-level inference on previously unseen tasks by adapting the learned prior to a short
support trajectory. For each new task $m$ with trajectory length $T_m$, we first select a support prefix of length
$S_m=\min\{T_{\mathrm{sup}},T_m\}$ and form the corresponding regression matrices
$X_m^{\mathrm{sup}}$ and $Y_m^{\mathrm{sup}}$ from the observed state transitions.

Given the learned meta-parameters $\phi=(W,V,\sigma^2)$, adaptation is carried out via a closed-form Bayesian update.
Specifically, the matrix-normal prior is combined with the support data to produce a task-specific posterior
\[
Q_{m,\phi}=\mathcal{MN}(M_m,I_d,V_m),
\]
where the posterior mean $M_m$ and covariance $V_m$ are given by \eqref{eq:posterior_cov} and
\eqref{eq:posterior_mean}. We use the posterior mean $\widehat{A}_m:=M_m$ as a point estimate of the task dynamics,
while the posterior covariance quantifies task-specific uncertainty and governs the strength of shrinkage toward the
meta-prior.

The posterior mean is used for the reported point-estimate metrics because both $E_A$ and
$E_{\mathrm{traj}}$ are defined for a single transition matrix and a single deterministic open-loop rollout. Under
squared-error loss, $M_m$ is the natural Bayes point estimate. The posterior distribution itself is not discarded:
the covariance $V_m$ enters the posterior-expected fit term in the meta-training objective and provides task-specific
uncertainty information. In applications where uncertainty-aware prediction is desired, one can sample
$A^{(s)}\sim Q_{m,\phi}$ and propagate each sampled transition matrix to obtain posterior predictive rollout bands or
credible regions for entries of the transition matrix.

We summarize the overall procedure in Algorithm~\ref{alg:meta_train_and_infer}. In particular, meta-training optimizes the
fit--KL surrogate over training tasks, whereas test-time adaptation consists only of the closed-form posterior update on
the support prefix followed by prediction or rollout using the posterior mean.

\begin{algorithm}[t]
\caption{PBML-LTI Meta-Training and Task Adaptation}
\label{alg:meta_train_and_infer}
\begin{algorithmic}[1]

\Statex \textbf{Meta-training}
\Require Training tasks $\mathcal{T}_{\mathrm{train}}$; steps $S$; batch size $B$.
    
\For{$s=1,\dots,S$}
    \State Sample minibatch $\mathcal{B}\subset\mathcal{T}_{\mathrm{train}}$ with $|\mathcal{B}|=\min(B,|\mathcal{T}_{\mathrm{train}}|)$.
    \For{each task $m\in\mathcal{B}$}
        \State Form $(X_m,Y_m)$ from the trajectory as in \eqref{eq:LTI}.
        \State Compute posterior parameters $(M_m,V_m)$ via \eqref{eq:posterior_cov} and \eqref{eq:posterior_mean}.
        \State Evaluate $\mathcal{L}_{\mathrm{fit},m}$ and $\mathcal{L}_{\mathrm{kl},m}$ via \eqref{eq:fit_term_def} and \eqref{eq:kl_term_def}.
    \EndFor
    \State $L_{\mathrm{tr}}\gets \frac{1}{|\mathcal{B}|}\sum_{m\in\mathcal{B}}\big(\mathcal{L}_{\mathrm{fit},m}+\mathcal{L}_{\mathrm{kl},m}\big)
    +\mathcal{R}_{\mathrm{hyper}}(W,V)+\mathcal{R}_{\mathrm{stab}}(W)$.
    \State Adam update of $(W,V,\sigma^2)$ using $\nabla L_{\mathrm{tr}}$.
\EndFor
\State Save learned $\phi=(W,V,\sigma^2)$.

\Statex \textbf{Task adaptation on new tasks}
\Require New tasks $\mathcal{T}_{\mathrm{test}}$; support horizon $T_{\mathrm{sup}}$; (optional) query horizon $T_{\mathrm{qry}}$.
\For{each task $m\in\mathcal{T}_{\mathrm{test}}$}
    \State Choose $S_m\gets \min\{T_{\mathrm{sup}},T_m\}$.
    \State Form $(X_m^{\mathrm{sup}},Y_m^{\mathrm{sup}})$ from the first $S_m$ transitions.
    \State Compute $(M_m,V_m)$ using $(X_m,Y_m)=(X_m^{\mathrm{sup}},Y_m^{\mathrm{sup}})$ via \eqref{eq:posterior_cov} and \eqref{eq:posterior_mean}.
    \State Set point estimate $\widehat{A}_m\gets M_m$.
    \State Roll out over $K_m=\min\{T_{\mathrm{qry}},T_m-S_m\}$ using $\widehat{x}_{t+1}=\widehat{A}_m\widehat{x}_t$.
\EndFor
\State \Return $\{\widehat{A}_m\}_{m\in\mathcal{T}_{\mathrm{test}}}$ and rollout predictions $\{\widehat{x}_{m,t}\}$.
\end{algorithmic}
\end{algorithm}

 \subsection{Theoretical Analysis}
\label{theoretical_analysis}

The previous subsection introduced the fit--KL surrogate used for meta-training. We now make its theoretical origin more
explicit in three steps. First, we state the generic martingale PAC-Bayes predictive-risk result under a sub-Gaussian MGF
condition. Second, we specialize this result to the support--query protocol used in few-shot adaptation. Third, we record
problem-specific consequences of this support--query predictive bound for transition-matrix recovery and open-loop rollout
error. Thus, the theory should be interpreted as providing a PAC-Bayes justification for the fit--KL surrogate and for
the support--query predictive-risk perspective, rather than as a direct statement of the exact training objective
optimized in practice.

\subsubsection{From Martingale PAC-Bayes to Support--Query Prediction}
\begin{corollary}[\textbf{Sub-Gaussian martingale PAC-Bayes excess risk}]\footnote{The proof is included in Appendix~\ref{sec:proof4}.}
\label{cor:subg_excess}
Under Assumption~\ref{as4}, for any prior $P$ independent of the trajectory and any (possibly data-dependent) posterior
$Q$, for any $\delta\in(0,1)$ and any $\lambda_{\mathrm{PB}}\in(0,1]$, with probability at least $1-\delta$,
\begin{equation}
L_T(Q)
\le
\widehat{L}_T(Q)
+
\frac{\KL(Q\|P)+\log(1/\delta)}{\lambda_{\mathrm{PB}} T}
+
\frac{\lambda_{\mathrm{PB}} v}{2}.
\label{eq:pb_uniform_v_short}
\end{equation}
Optimizing the right-hand side over $\lambda_{\mathrm{PB}}\in(0,1]$ gives
\begin{equation}
\lambda_{\mathrm{PB}}^\star
=
\min\!\left\{
1,\;
\sqrt{\frac{2\big(\KL(Q\|P)+\log(1/\delta)\big)}{vT}}
\right\}.
\end{equation}
In the regime where $\lambda_{\mathrm{PB}}^\star<1$, this simplifies to
\begin{equation}
L_T(Q)
\le
\widehat{L}_T(Q)
+
\sqrt{\frac{2v\big(\KL(Q\|P)+\log(1/\delta)\big)}{T}}.
\label{eq:pb_rate_sqrtT_short}
\end{equation}
\end{corollary}
Corollary~\ref{cor:subg_excess} shows that the excess-risk term decays at the canonical rate
$O(T^{-1/2})$ under sub-Gaussian martingale concentration. It also makes explicit the role of meta-learning: for a fixed
task posterior $Q$, a better prior $P$ reduces the divergence $\KL(Q\|P)$ and therefore tightens the gap between
empirical in-trajectory fit and conditional predictive risk. This provides the generic PAC-Bayes structure that motivates
the fit--KL objective used by PBML-LTI.

\paragraph{Specialization to the support--query protocol.}
The result above is stated for a generic trajectory. To connect it directly to few-shot adaptation, we now specialize it
to the support--query protocol used in our experiments. Fix a task $m$ with trajectory
$\{x_{m,t}\}_{t=0}^{T_m}$. Let the support prefix be
\[
D_m^{\mathrm{sup}}:=\{x_{m,t}\}_{t=0}^{S_m},
\]
and form the task-specific posterior using only this prefix:
\[
Q_{m,\phi}^{\mathrm{sup}} := p(A_m\mid D_m^{\mathrm{sup}},\phi).
\]
Conditioning on the support filtration $\mathcal{F}_{m,S_m}$ fixes the posterior. The remaining query suffix can then be
analyzed using the shifted filtration
\[
\mathcal{G}_u := \mathcal{F}_{m,S_m+u}, \qquad u=0,\dots,K_m,
\]
where $K_m\le T_m-S_m$ is the query horizon.

Define the query regression matrices
\begin{equation}
X_{m,q}:=[x_{m,S_m},\dots,x_{m,S_m+K_m-1}] \in \R^{d\times K_m},
\qquad
Y_{m,q}:=[x_{m,S_m+1},\dots,x_{m,S_m+K_m}] \in \R^{d\times K_m}.
\label{eq:query_xy}
\end{equation}
We define the conditional support--query predictive risk of a posterior $Q$ by
\begin{equation}
R_{m,q}^{\mathrm{pred}}(Q)
:=
\frac{1}{K_m}\sum_{u=1}^{K_m}
\E_{A\sim Q}\!\left[
\E\!\left[
\|x_{m,S_m+u}-Ax_{m,S_m+u-1}\|_2^2
\,\middle|\,
\mathcal G_{u-1}
\right]
\right],
\label{eq:support_query_pred_risk}
\end{equation}
with empirical counterpart
\begin{equation}
\widehat R_{m,q}^{\mathrm{pred}}(Q)
:=
\frac{1}{K_m}\E_{A\sim Q}\!\left[
\|Y_{m,q}-AX_{m,q}\|_F^2
\right].
\label{eq:support_query_emp_pred_risk}
\end{equation}
These quantities are nonnegative and correspond directly to prediction on future query time steps after adaptation.

\begin{corollary}[\textbf{Support--query martingale PAC-Bayes bound}]\footnote{The proof is included in Appendix~\ref{sec:proof_support_query_pb}.}
\label{cor:support_query_pb}
Under Assumption~\ref{as4}, for any $\delta\in(0,1)$ and any $\lambda_{\mathrm{PB}}\in(0,1]$, with probability at least $1-\delta$,
\begin{equation}
R_{m,q}^{\mathrm{pred}}(Q_{m,\phi}^{\mathrm{sup}})
\le
\widehat R_{m,q}^{\mathrm{pred}}(Q_{m,\phi}^{\mathrm{sup}})
+
\frac{2\sigma^2}{\lambda_{\mathrm{PB}} K_m}
\Big(
\KL(Q_{m,\phi}^{\mathrm{sup}}\|P_\phi)+\log(1/\delta)
\Big)
+
\sigma^2\lambda_{\mathrm{PB}} v.
\label{eq:support_query_pb}
\end{equation}
\end{corollary}

This result is obtained by applying Corollary~\ref{cor:subg_excess} to the shifted query sequence
$\{\ell_{m,S_m+u}(A)\}_{u=1}^{K_m}$ under the filtration $\{\mathcal G_u\}_{u=0}^{K_m}$, and then removing the constant
term in the Gaussian log-loss and rescaling both sides by $2\sigma^2$. Consequently, the left-hand side becomes the
support-conditioned predictive risk on the held-out query suffix, while the right-hand side contains the empirical query
fit, the posterior-to-prior complexity term, and the martingale concentration penalty.

Importantly, \eqref{eq:support_query_pb} is not a support-only certificate whose right-hand side is computable before
observing the query suffix. The term
$\widehat R_{m,q}^{\mathrm{pred}}(Q_{m,\phi}^{\mathrm{sup}})$ is evaluated on the held-out query trajectory. The result
should therefore be interpreted as a support-conditioned support--query PAC-Bayes diagnostic: it is informative when the
posterior adapted from the support prefix also fits the held-out query suffix. If the empirical query fit is large, this
term can dominate the bound and the derived transition-matrix and rollout consequences can become loose.

\subsubsection{Support Fit Transfer to Query Fit Analysis}
We now make explicit sufficient conditions under which the empirical query term is controlled by quantities associated
with the support-adapted posterior. Let the support regression pair used to form
$Q_{m,\phi}^{\mathrm{sup}}=\mathcal{MN}(M_m,I_d,V_m)$ be denoted by
$(X_m^{\mathrm{sup}},Y_m^{\mathrm{sup}})$, with $S_m$ transitions, and define
\[
\widehat\Sigma_{\mathrm{sup}}
:=
\frac{1}{S_m}X_m^{\mathrm{sup}}(X_m^{\mathrm{sup}})^\top,
\qquad
\widehat\Sigma_q
:=
\frac{1}{K_m}X_{m,q}X_{m,q}^\top ,
\]
together with the stacked support and query process-noise matrices
\[
\Xi_m^{\mathrm{sup}}:=[w_{m,1},\dots,w_{m,S_m}],
\qquad
\Xi_{m,q}:=[w_{m,S_m+1},\dots,w_{m,S_m+K_m}] .
\]
Also define the support empirical posterior-predictive risk
\begin{equation}
\widehat R_{\mathrm{sup}}^{\mathrm{pred}}
:=
\frac{1}{S_m}
\E_{A\sim Q_{m,\phi}^{\mathrm{sup}}}
\left[
\|Y_m^{\mathrm{sup}}-AX_m^{\mathrm{sup}}\|_F^2
\right].
\label{eq:support_pred_risk_for_transfer}
\end{equation}
By \eqref{eq:posterior_sq_loss_closed_form}, this is equal to
\[
\widehat R_{\mathrm{sup}}^{\mathrm{pred}}
=
\frac{1}{S_m}\|Y_m^{\mathrm{sup}}-M_mX_m^{\mathrm{sup}}\|_F^2
+
d\,\mathrm{tr}(V_m\widehat\Sigma_{\mathrm{sup}}).
\]

Define the support-to-query excitation-alignment constant
\begin{equation}
\gamma_m
:=
\inf\{\gamma>0:\widehat\Sigma_q\preceq \gamma\,\widehat\Sigma_{\mathrm{sup}}\}.
\label{eq:gamma_alignment}
\end{equation}
When finite, this constant measures how much the query suffix excites directions already excited by the support prefix.
Equivalently, $\gamma_m<\infty$ only when
$\operatorname{range}(\widehat\Sigma_q)\subseteq\operatorname{range}(\widehat\Sigma_{\mathrm{sup}})$, in which case
\[
\gamma_m
=
\lambda_{\max}\!\left(
\widehat\Sigma_{\mathrm{sup}}^{\dagger/2}
\widehat\Sigma_q
\widehat\Sigma_{\mathrm{sup}}^{\dagger/2}
\right).
\]
Thus, $\gamma_m=O(1)$ formalizes the condition that the query directions are sufficiently aligned with the support
directions.

\begin{proposition}[Support-to-query transfer under excitation alignment]
\label{prop:support_query_transfer}
Assume $\gamma_m<\infty$. Then, pathwise,
\begin{equation}
\widehat R_{m,q}^{\mathrm{pred}}(Q_{m,\phi}^{\mathrm{sup}})
\le
4\gamma_m\,\widehat R_{\mathrm{sup}}^{\mathrm{pred}}
+
\frac{4\gamma_m}{S_m}\|\Xi_m^{\mathrm{sup}}\|_F^2
+
\frac{2}{K_m}\|\Xi_{m,q}\|_F^2 ,
\label{eq:support_query_transfer}
\end{equation}
where $\Xi_m^{\mathrm{sup}}$ and $\Xi_{m,q}$ are the stacked support and query process-noise matrices. Moreover, the
posterior-width contribution satisfies
\begin{equation}
d\,\mathrm{tr}(V_m\widehat\Sigma_q)
\le
\gamma_m\,d\,
\min\left\{
\frac{\sigma^2 d}{S_m},\,
\mathrm{tr}(V\widehat\Sigma_{\mathrm{sup}})
\right\}.
\label{eq:posterior_width_transfer_cap}
\end{equation}
\end{proposition}

Proposition~\ref{prop:support_query_transfer} shows that, under support--query excitation alignment, the held-out query
empirical term is bounded by the support predictive risk plus pure noise energies. The factor $\gamma_m$ is the price of
transferring support fit to query fit under temporal dependence. This is the dependent-data analogue of the role played
by exchangeability in i.i.d. train--test arguments. The refined cap \eqref{eq:posterior_width_transfer_cap} also makes
the meta-learning mechanism visible: when $S_m$ is small, the posterior-width contribution is governed by
$\mathrm{tr}(V\widehat\Sigma_{\mathrm{sup}})$, that is, by the tightness of the \emph{learned} prior covariance along
the realized support excitation, and once the support is well excited it decays at the parametric rate
$\sigma^2 d^2/S_m$.

\subsubsection{Regularized Transfer and Few-Shot Limitations}
However, in the strict few-shot regime $S_m<d$, the support Gram matrix is rank deficient; if the process noise is
nondegenerate, the very first query state already exits $\operatorname{range}(\widehat\Sigma_{\mathrm{sup}})$, so that
$\gamma_m=\infty$ almost surely. This is not merely a proof artifact: the support trajectory cannot identify task-specific
directions it does not excite. We therefore also use a regularized alignment constant
\begin{equation}
\gamma_{m,\alpha}
:=
\lambda_{\max}\!\left[
(\widehat\Sigma_{\mathrm{sup}}+\alpha I_d)^{-1/2}
\widehat\Sigma_q
(\widehat\Sigma_{\mathrm{sup}}+\alpha I_d)^{-1/2}
\right],
\qquad \alpha>0,
\label{eq:regularized_gamma_alignment}
\end{equation}
which is always finite.

\begin{proposition}[Regularized support-to-query transfer]
\label{prop:regularized_support_query_transfer}
For any $\alpha>0$,
\begin{align}
\widehat R_{m,q}^{\mathrm{pred}}(Q_{m,\phi}^{\mathrm{sup}})
\le\;&
4\gamma_{m,\alpha}\,\widehat R_{\mathrm{sup}}^{\mathrm{pred}}
+
\frac{4\gamma_{m,\alpha}}{S_m}\|\Xi_m^{\mathrm{sup}}\|_F^2
+
\frac{2}{K_m}\|\Xi_{m,q}\|_F^2
\nonumber\\
&\quad+
2\gamma_{m,\alpha}\alpha\,
\E_{A\sim Q_{m,\phi}^{\mathrm{sup}}}\|A-A_m\|_F^2 .
\label{eq:regularized_support_query_transfer}
\end{align}
Moreover,
\[
\E_{A\sim Q_{m,\phi}^{\mathrm{sup}}}\|A-A_m\|_F^2
=
\|M_m-A_m\|_F^2+d\,\mathrm{tr}(V_m).
\]
Moreover, $\gamma_{m,\alpha}\le\min\{\gamma_m,\;\lambda_{\max}(\widehat\Sigma_q)/\alpha\}$, and the map
$\alpha\mapsto\gamma_{m,\alpha}$ is nonincreasing.
\end{proposition}

The final term in \eqref{eq:regularized_support_query_transfer} isolates the weakly excited directions. In those
directions, transfer cannot come from the support data alone; it must come from posterior concentration and learned-prior
quality. Indeed, by \eqref{eq:posterior_mean}, the mean-error part obeys the closed-form identity
$M_m-A_m=\big[\tfrac{1}{\sigma^2}\Xi_m^{\mathrm{sup}}(X_m^{\mathrm{sup}})^\top+(W-A_m)V^{-1}\big]V_m$,
so it is controlled by support noise concentration together with the learned-prior quality $\Delta_m(\phi)$ defined in
\eqref{eq:prior_quality_basic_inequality}; this control is quantified in
Appendix~\ref{sec:proof_support_query_transfer}. Thus, the regularized transfer result is the relevant statement in the high-dimensional few-shot regime. A
prior-matched regularization choice replaces $\alpha I_d$ by $(\sigma^2/S_m)V^{-1}$, in which case the corresponding
posterior-scaled alignment diagnostic is
\[
\widetilde\gamma_m
=
\frac{S_m}{\sigma^2}
\lambda_{\max}\!\left(
V_m^{1/2}\widehat\Sigma_q V_m^{1/2}
\right),
\]
and replacing $\widehat\Sigma_q$ by the validation Gram gives the support-window proxy used in the empirical diagnostics
of Appendix~\ref{app:support_query_diagnostics}. With the prior-matched choice, the regularized Gram becomes
$\widehat\Sigma_{\mathrm{sup}}+(\sigma^2/S_m)V^{-1}=(\sigma^2/S_m)V_m^{-1}$ by \eqref{eq:posterior_cov}, so
$\widetilde\gamma_m$ is computable in closed form from the fitted posterior and the query Gram, and the posterior-width
contribution obeys the unconditional bound
$d\,\mathrm{tr}(V_m\widehat\Sigma_q)\le\widetilde\gamma_m\,\sigma^2d^2/S_m$.

\subsubsection{A Conditional Support-to-Query Certificate}
\begin{lemma}[Noise-energy concentration]
\label{lem:noise_energy_concentration}
Under Assumption~\ref{as2}, there is an absolute constant $c>0$ such that, for a segment of length $n$ and any
$\delta\in(0,1)$, with probability at least $1-\delta$,
\[
\frac{1}{n}\sum_{t=1}^{n}\|w_{m,t}\|_2^2
\le
\mathrm{tr}(\Sigma_w)
+
c\sigma_w^2 d
\left(
\sqrt{\frac{\log(1/\delta)}{n}}
+
\frac{\log(1/\delta)}{n}
\right).
\]
\end{lemma}

The leading dimension factor $d$ in the deviation term of Lemma~\ref{lem:noise_energy_concentration} improves to
$\sqrt{d}$ when the noise coordinates are conditionally independent, as for the Gaussian noise generator
\eqref{eq:synthetic_noise} used in the synthetic experiments; we state the general conditionally sub-Gaussian version for
consistency with Assumption~\ref{as2}.

Combining Proposition~\ref{prop:support_query_transfer}, Lemma~\ref{lem:noise_energy_concentration}, and
Corollary~\ref{cor:support_query_pb} gives the following conditional support-to-query certificate.

\begin{corollary}[Conditional support-to-query certificate]
\label{cor:conditional_support_query_certificate}
Fix $\bar\gamma>0$, $\lambda_{\mathrm{PB}}\in(0,1]$, and $\delta\in(0,1)$. On the event
$\widehat\Sigma_q\preceq \bar\gamma\,\widehat\Sigma_{\mathrm{sup}}$, with probability at least $1-3\delta$,
\begin{align}
R_{m,q}^{\mathrm{pred}}(Q_{m,\phi}^{\mathrm{sup}})
\le\;&
4\bar\gamma\,\widehat R_{\mathrm{sup}}^{\mathrm{pred}}
+
(4\bar\gamma+2)\mathrm{tr}(\Sigma_w)
+
4\bar\gamma\,\epsilon_{S_m}(\delta)
+
2\,\epsilon_{K_m}(\delta)
\nonumber\\
&+
\frac{2\sigma^2}{\lambda_{\mathrm{PB}}K_m}
\left(
\KL(Q_{m,\phi}^{\mathrm{sup}}\|P_\phi)+\log(1/\delta)
\right)
+
\sigma^2\lambda_{\mathrm{PB}}v,
\label{eq:conditional_support_query_certificate}
\end{align}
where
\[
\epsilon_n(\delta)
=
c\sigma_w^2 d
\left(
\sqrt{\frac{\log(1/\delta)}{n}}
+
\frac{\log(1/\delta)}{n}
\right).
\]
\end{corollary}

Corollary~\ref{cor:conditional_support_query_certificate} is a conditional certificate: apart from the alignment level
and noise-floor constants, the right-hand side is determined by the support-adapted posterior and the support prefix. It
is not an unconditional support-only guarantee. Such an unconditional guarantee cannot hold in general when the support
does not excite the directions later encountered in the query suffix. In that case, the regularized bound
\eqref{eq:regularized_support_query_transfer} identifies the remaining term that must be controlled by learned-prior
quality. If the effective fit prefix is selected by inner validation from a grid of $G$ candidate prefixes, as in the
adaptive-support evaluation protocol, the certificate holds uniformly over the grid with $\delta$ replaced by
$\delta/G$, via a union bound over the candidates. A regularized variant of
Corollary~\ref{cor:conditional_support_query_certificate} replaces the event
$\widehat\Sigma_q\preceq\bar\gamma\,\widehat\Sigma_{\mathrm{sup}}$ by
$\widehat\Sigma_q\preceq\bar\gamma(\widehat\Sigma_{\mathrm{sup}}+\alpha I_d)$---an event that holds deterministically
for $\bar\gamma=\gamma_{m,\alpha}$---at the price of the additional term
$2\bar\gamma\alpha\,\E_{A\sim Q_{m,\phi}^{\mathrm{sup}}}\|A-A_m\|_F^2$ from
Proposition~\ref{prop:regularized_support_query_transfer}; this term involves the true $A_m$ and is a prior-quality
quantity rather than a support-computable one.

\begin{remark}[Unavoidable support--query linkage assumption]
\label{rem:necessity}
The conditional form of Corollary~\ref{cor:conditional_support_query_certificate} reflects a genuine
information-theoretic obstruction rather than a proof artifact. Any two transition matrices that agree on
$\operatorname{range}(\widehat\Sigma_{\mathrm{sup}})$ induce the same support law, so no bound on
$\widehat R_{m,q}^{\mathrm{pred}}$ that is uniform over $A_m$ can be certified from the support alone once the query
excites $\ker(\widehat\Sigma_{\mathrm{sup}})$. The value $\gamma_m=\infty$ in \eqref{eq:gamma_alignment} flags exactly
this situation, and the term
$2\gamma_{m,\alpha}\alpha\,\E_{A\sim Q_{m,\phi}^{\mathrm{sup}}}\|A-A_m\|_F^2$ in
\eqref{eq:regularized_support_query_transfer} prices it. On the unexcited subspace the posterior coincides with the
prior, so control there is an \emph{environment-level} guarantee about the learned $(W,V)$---precisely the second-stage
meta-generalization question that Section~\ref{subsec:prior_quality_sensitivity} and the Conclusion identify as future
work. The analysis thus makes the division of labor explicit: excitation alignment and the support predictive risk
$\widehat R_{\mathrm{sup}}^{\mathrm{pred}}$ govern the task-level, excited component, while prior quality governs the
remainder.
\end{remark}

\subsubsection{Stable and Unstable Regimes}
We next identify regimes in which the support--query transfer constants are controlled, and regimes in which they necessarily deteriorate.

The preceding results make precise the conditions under which the empirical query term is small. If the support Gram matrix is well excited and query state norms are controlled, namely
\[
\lambda_{\min}(\widehat\Sigma_{\mathrm{sup}})\ge\kappa_{\mathrm{sup}}>0,
\qquad
\max_{0\le u<K_m}\|x_{m,S_m+u}\|_2^2\le B_x^2,
\]
then
\[
\gamma_m\le \frac{B_x^2}{\kappa_{\mathrm{sup}}}.
\]
Indeed, $xx^\top\preceq\|x\|_2^2 I_d$ for every $x\in\R^d$, hence
$\widehat\Sigma_q\preceq B_x^2 I_d\preceq (B_x^2/\kappa_{\mathrm{sup}})\,\widehat\Sigma_{\mathrm{sup}}$.
Thus, controlled states and support excitation imply controlled support-to-query transfer. In a common stationary stable
regime, where support and query empirical Grams concentrate around the same stationary covariance, $\gamma_m=O(1)$ with
high probability. Concretely, suppose $\rho(A_m)\le\bar\rho<1$ and $\Sigma_w\succ0$, and let
$\Gamma_\infty:=\sum_{k\ge0}A_m^k\Sigma_w(A_m^k)^\top$, which is finite and positive definite. If
\[
\frac12\Gamma_\infty\preceq\widehat\Sigma_{\mathrm{sup}},
\qquad
\widehat\Sigma_q\preceq\frac32\Gamma_\infty,
\]
then $\gamma_m\le3$. Two-sided concentration of dependent empirical Gram matrices around $\Gamma_\infty$ at this
accuracy holds with high probability once $S_m$ and $K_m$ exceed a polynomial burn-in, by now-standard arguments for
stable linear systems \citep{abbasi2011online,simchowitz2018learning,sarkar2019near}. Since
$\gamma_{m,\alpha}\le\gamma_m$, both conditions bound the regularized constant as well.

The same conditions also make explicit the role of the learned prior. Since $M_m$ minimizes the conjugate
ridge-form objective
\[
A\mapsto
\|Y_m^{\mathrm{sup}}-AX_m^{\mathrm{sup}}\|_F^2
+
\sigma^2\,\mathrm{tr}\big((A-W)V^{-1}(A-W)^\top\big),
\]
comparison with $A=A_m$ yields
\begin{equation}
\frac{1}{S_m}\|Y_m^{\mathrm{sup}}-M_mX_m^{\mathrm{sup}}\|_F^2
\le
\frac{1}{S_m}\|\Xi_m^{\mathrm{sup}}\|_F^2
+
\frac{\sigma^2}{S_m}\Delta_m(\phi),
\qquad
\Delta_m(\phi)
:=
\mathrm{tr}\big((A_m-W)V^{-1}(A_m-W)^\top\big).
\label{eq:prior_quality_basic_inequality}
\end{equation}
Thus, the support fit is small when the support noise energy is moderate and the learned prior is close to the
task-specific dynamics in the $V^{-1}$ geometry. Combining \eqref{eq:prior_quality_basic_inequality} with the width cap
$d\,\mathrm{tr}(V_m\widehat\Sigma_{\mathrm{sup}})\le\sigma^2d^2/S_m$ from the proof of
Proposition~\ref{prop:support_query_transfer}, Lemma~\ref{lem:noise_energy_concentration}, and
Corollary~\ref{cor:conditional_support_query_certificate} with $\bar\gamma=3$, we obtain: with probability at least
$1-3\delta$, on the stable common-regime event above,
\begin{align}
R_{m,q}^{\mathrm{pred}}(Q_{m,\phi}^{\mathrm{sup}})
\le\;&
26\,\mathrm{tr}(\Sigma_w)
+
\frac{12\,\sigma^2\big(\Delta_m(\phi)+d^2\big)}{S_m}
+
24\,\epsilon_{S_m}(\delta)
+
2\,\epsilon_{K_m}(\delta)
\nonumber\\
&+
\frac{2\sigma^2}{\lambda_{\mathrm{PB}}K_m}
\left(
\KL(Q_{m,\phi}^{\mathrm{sup}}\|P_\phi)+\log(1/\delta)
\right)
+
\sigma^2\lambda_{\mathrm{PB}}v,
\label{eq:explicit_smallness_display}
\end{align}
with constants not optimized. Term by term: the first term is the irreducible one-step noise floor (for the
homoscedastic Gaussian generator \eqref{eq:synthetic_noise} the predictive risk is bounded below by
$\mathrm{tr}(\Sigma_w)=d\sigma_{\mathrm{true}}^2$, so \eqref{eq:explicit_smallness_display} is then order-optimal up to
a universal constant); the second term decays at the parametric rate in $S_m$ and is small precisely when the learned
prior is good, both through $\Delta_m(\phi)$ and through the KL term, whose closed form
\eqref{eq:matrix_normal_kl_closed_form} shrinks for the same reason---this is the meta-learning effect, now explicit on
the right-hand side; the remaining terms vanish as $S_m$ and $K_m$ grow. This gives an explicit sufficient condition
under which the support--query bound is numerically meaningful.

We emphasize the scope: the two sufficient conditions above are deliberately stable, data-rich conditions. In the strict
few-shot regime $S_m<d$ they cannot hold, and the operative statement is the regularized
Proposition~\ref{prop:regularized_support_query_transfer}, in which the weakly excited component is controlled by the
learned prior rather than by the support data.

Conversely, if $A_m$ has an unstable dominant mode, the alignment constants are provably large and the transfer
factor degrades geometrically in the query horizon. The following remark records this failure mode formally.

\begin{remark}[Geometric degradation under an unstable dominant mode]
\label{rem:unstable_failure_mode}
Let $u\in\R^d$ with $\|u\|_2=1$ be a left eigenvector of $A_m$ with real eigenvalue $\rho>1$, i.e.\
$A_m^\top u=\rho u$. Then $s_t:=\langle u,x_{m,t}\rangle$ follows the scalar recursion
$s_{t+1}=\rho s_t+\langle u,w_{m,t+1}\rangle$, and $\rho^{-t}s_t$ converges almost surely to a random limit that is
nonzero almost surely whenever the noise has a nondegenerate component along $u$ (as for the Gaussian process noise
used in the synthetic experiments). Whenever $u^\top\widehat\Sigma_{\mathrm{sup}}u>0$, the variational
characterization of the positive-semidefinite order in \eqref{eq:gamma_alignment} gives
\[
\gamma_m
\;\ge\;
\frac{u^\top\widehat\Sigma_q u}{u^\top\widehat\Sigma_{\mathrm{sup}}u}
\;=\;
\frac{S_m\sum_{r=0}^{K_m-1}s_{S_m+r}^{\,2}}{K_m\sum_{t=0}^{S_m-1}s_t^{\,2}}
\;\asymp\;
\rho^{2K_m}\cdot\frac{S_m}{K_m}\cdot c_{\mathrm{traj}},
\]
where $c_{\mathrm{traj}}$ is a trajectory-dependent factor bounded away from $0$ and $\infty$ on the almost-sure
growth event. The same conclusion holds for the regularized constant \eqref{eq:regularized_gamma_alignment} at any
fixed $\alpha>0$, since
$\gamma_{m,\alpha}\ge u^\top\widehat\Sigma_q u/\bigl(u^\top\widehat\Sigma_{\mathrm{sup}}u+\alpha\bigr)$
and the numerator still grows geometrically in $K_m$. Hence the transfer factor---and with it any certificate built
on it---degrades \emph{geometrically in the query horizon} when the dominant mode is unstable, explaining why the
query empirical term may dominate the bound in unstable low-dimensional settings. This complements the finite-horizon
amplification that appears in the rollout consequence through $H_{m,K_m}$.
\end{remark}

Empirical diagnostics supporting validation-window proxies for support--query comparability are reported in
Appendix~\ref{app:support_query_diagnostics}.

\subsubsection{Consequences for Transition Recovery and Rollout Error}
In the remainder of this subsection, we write $\widehat A_m:=M_m$ for the posterior-mean predictor produced by task
adaptation and record several PAC-Bayes-derived corollaries that connect the support--query predictive bound to the
empirical quantities used in our experiments. For a deterministic predictor $\widehat A_m$, we write
$R_{m,q}^{\mathrm{pred}}(\widehat A_m)$ as shorthand for
$R_{m,q}^{\mathrm{pred}}(\delta_{\widehat A_m})$.

\paragraph{Projected transition-matrix error bound.}
Under the true query dynamics
\[
Y_{m,q}=A_mX_{m,q}+\Xi_{m,q},
\]
and since the squared Frobenius loss is convex in $A$, Jensen's inequality implies
\begin{equation}
\frac{1}{K_m}\|Y_{m,q}-M_mX_{m,q}\|_F^2
\le
\widehat R_{m,q}^{\mathrm{pred}}(Q_{m,\phi}^{\mathrm{sup}}).
\label{eq:jensen_support_query}
\end{equation}
Taking conditional expectation under the true dynamics, the residual at the posterior mean decomposes into a projected
transition-matrix term plus a nonnegative noise floor. Combining this with \eqref{eq:jensen_support_query} yields
\begin{equation}
E_{A,\mathrm{proj}}(m)
:=
\frac{1}{K_m}\|(\widehat A_m-A_m)X_{m,q}\|_F^2
\le
R_{m,q}^{\mathrm{pred}}(\widehat A_m)
\le
R_{m,q}^{\mathrm{pred}}(Q_{m,\phi}^{\mathrm{sup}})
\le
\mathcal B_m^{\mathrm{pred}},
\label{eq:reply_proj_A_bound}
\end{equation}
where $\mathcal B_m^{\mathrm{pred}}$ denotes the right-hand side of \eqref{eq:support_query_pb}. This quantity is always
well defined and measures transition-matrix error along the query excitation directions.

\paragraph{Full Frobenius transition-matrix bound under excitation.}
If the query Gram matrix satisfies the excitation condition
\begin{equation}
\lambda_{\min}\!\left(\frac{1}{K_m}X_{m,q}X_{m,q}^\top\right)\ge \kappa_m > 0,
\label{eq:reply_excitation_cond}
\end{equation}
then the projected bound lifts to a full Frobenius bound:
\begin{equation}
E_A(m)
=
\|\widehat A_m-A_m\|_F^2
\le
\frac{\mathcal B_m^{\mathrm{pred}}}{\kappa_m}.
\label{eq:reply_full_A_bound}
\end{equation}
When $K_m<d$, this bound may be vacuous because the query Gram matrix can be rank deficient; in such regimes,
\eqref{eq:reply_proj_A_bound} remains the meaningful matrix-error consequence.

\paragraph{Trajectory rollout bound.}
Writing the deterministic rollout error as
\[
E_{\mathrm{traj}}(m)
:=
\sum_{u=1}^{K_m}\|\widehat x_{m,S_m+u}-x_{m,S_m+u}\|_2^2,
\]
and defining the finite-horizon growth factor
\[
H_{m,K_m}:=\sum_{r=0}^{K_m-1}\|\widehat A_m\|_2^r,
\]
a standard unrolling argument for the rollout error recursion gives the PAC-Bayes-derived consequence
\begin{equation}
\E\!\left[E_{\mathrm{traj}}(m)\mid \mathcal{F}_{m,S_m}\right]
\le
K_m^2\,H_{m,K_m}^2\,\mathcal B_m^{\mathrm{pred}}.
\label{eq:reply_traj_bound}
\end{equation}
The factor $K_m^2$ appears because $\mathcal B_m^{\mathrm{pred}}$ is a per-step average predictive bound, whereas
$E_{\mathrm{traj}}(m)$ is a cumulative multi-step error.

Together, \eqref{eq:reply_proj_A_bound}--\eqref{eq:reply_traj_bound} show how the support--query PAC-Bayes predictive
bound induces explicit bounds for projected transition-matrix recovery, full Frobenius recovery under query excitation,
and finite-horizon rollout accuracy. These results are PAC-Bayes-derived corollaries of the primary support--query
predictive-risk bound. Their practical tightness depends on the same held-out query fit term
$\widehat R_{m,q}^{\mathrm{pred}}(Q_{m,\phi}^{\mathrm{sup}})$ appearing in
Corollary~\ref{cor:support_query_pb}. Therefore, the matrix and rollout consequences should be viewed as informative
when the query predictive fit remains controlled, and as potentially loose or vacuous when the adapted dynamics produce
large query residuals or unstable rollout amplification. On the alignment event of
Corollary~\ref{cor:conditional_support_query_certificate}, substituting its right-hand side for
$\mathcal B_m^{\mathrm{pred}}$ in \eqref{eq:reply_proj_A_bound}--\eqref{eq:reply_traj_bound} shows that all three
consequences inherit the same conditional support-to-query form; this corollary chain is recorded in
Appendix~\ref{sec:proof_support_query_transfer}.

\section{Performance Evaluation}

We evaluate the proposed framework through a comprehensive series of experiments on both synthetic LTI systems and real-world data sets. We begin by detailing the experimental protocol, including competitive baselines and evaluation metrics. Subsequently, we present a comparative analysis demonstrating the advantages of our approach in terms of parameter identification accuracy, multi-step predictive stability, and data efficiency.

\subsection{Baselines}
\label{sec:baselines}

We compare PBML-LTI against four primary baselines spanning both single-task estimation and cross-task structure sharing. In additional experiments, we also compare against a MAML-style LTI baseline as a stronger gradient-based
meta-learning method. For all
methods, each test trajectory is evaluated under the same outer support--query protocol: the first part of the trajectory
is reserved as the available support window and the following part is held out as the query window. Our primary
evaluation uses an adaptive-support protocol under a common maximum support budget. That is, each method receives the
same available support window, but may choose an effective fit prefix within that budget before being evaluated on the
same held-out query segment. This setup is intended to measure not only post-adaptation accuracy, but also how much
support data each method actually needs in order to adapt well on a given task.

\paragraph{Per-task ordinary least squares (OLS).}
The OLS baseline fits each task independently using the closed-form least-squares estimator
\begin{equation}
\widehat{A}_m^{\mathrm{ols}}
=
Y_m^{\mathrm{sup}} (X_m^{\mathrm{sup}})^\top
\Big(X_m^{\mathrm{sup}} (X_m^{\mathrm{sup}})^\top\Big)^{-1}.
\end{equation}
Here $(X_m^{\mathrm{sup}},Y_m^{\mathrm{sup}})$ denotes the regression pair formed from the support data available to the
method. To ensure numerical stability, we add a small diagonal jitter and use a pseudo-inverse fallback whenever the Gram
matrix is ill-conditioned. OLS has no regularization hyperparameter.

\paragraph{Per-task ridge (Ridge).}
The Ridge baseline fits each task independently by Tikhonov-regularized least squares:
\begin{equation}
\widehat{A}_m^{\mathrm{ridge}}(\lambda_{\mathrm{reg}})
=
Y_m^{\mathrm{sup}} (X_m^{\mathrm{sup}})^\top
\Big(X_m^{\mathrm{sup}} (X_m^{\mathrm{sup}})^\top + \lambda_{\mathrm{reg}} I_d\Big)^{-1}.
\end{equation}
We select $\lambda_{\mathrm{reg}}$ once using only the training tasks by grid search over a predefined candidate set with
a stability-oriented criterion. In synthetic experiments, the stability threshold is fixed from the known regime parameter,
$\rho_{\mathrm{target}}=\rho_0$, and is not tuned on the test set. In settings where such regime information is not
available, $\rho_{\mathrm{target}}$ should be regarded as a baseline hyperparameter or prior stability specification. Specifically, we choose the smallest $\lambda_{\mathrm{reg}}$ for which the fitted
matrices satisfy
\begin{equation}
\mathbb{E}_{m\in\mathcal{T}_{\mathrm{tr}}}\big[\rho(\widehat{A}_m^{\mathrm{ridge}}(\lambda_{\mathrm{reg}}))\big]
\le \rho_{\mathrm{target}},
\label{ridge_regular}
\end{equation}
and if no candidate satisfies this condition, we choose the one minimizing the mean stability violation
\begin{equation}
\mathbb{E}_{m\in\mathcal{T}_{\mathrm{tr}}}
\Big[\max\big\{0,\rho(\widehat{A}_m^{\mathrm{ridge}}(\lambda_{\mathrm{reg}}))-\rho_{\mathrm{target}}\big\}\Big].
\label{ridge_regular2}
\end{equation}
The selected $\lambda_{\mathrm{reg}}$ is then fixed and reused for all test tasks.

\paragraph{Pooled-prior Ridge.}
The Pooled-prior Ridge baseline incorporates cross-task information through a pooled mean dynamics estimate. It first
computes
\begin{equation}
\bar{A}
:=
\frac{1}{M_{\mathrm{tr}}}\sum_{m=1}^{M_{\mathrm{tr}}}\widehat{A}_m^{\mathrm{ols,full}},
\end{equation}
where $\widehat{A}_m^{\mathrm{ols,full}}$ is the OLS estimate fitted on the full training trajectory of task $m$.
Given $\bar{A}$, the task-specific estimate is
\begin{equation}
\widehat{A}_m^{\mathrm{pool}}(\lambda_{\mathrm{reg}})
=
\Big(Y_m^{\mathrm{sup}} (X_m^{\mathrm{sup}})^\top + \lambda_{\mathrm{reg}} \bar{A}\Big)
\Big(X_m^{\mathrm{sup}} (X_m^{\mathrm{sup}})^\top + \lambda_{\mathrm{reg}} I_d\Big)^{-1}.
\end{equation}
We tune $\lambda_{\mathrm{reg}}$ using the same stability-based grid-search procedure as for Ridge, with the same
$\rho_{\mathrm{target}}$ convention, and then keep it fixed for evaluation on validation and test tasks. To verify that
our conclusions are not an artifact of this stability-oriented tuning, Appendix~\ref{app:rho_target_diagnostic} also
reports a diagnostic variant in which Ridge and Pooled-prior Ridge are tuned directly by validation rollout error, without
using the $\rho_{\mathrm{target}}$ criterion.

\paragraph{Shared Subspace.}
The Shared Subspace baseline assumes that the task matrices share a low-dimensional structure in vectorized form. It first
fits OLS estimates on the training tasks, vectorizes them to obtain $\mathrm{vec}(\widehat{A}_m^{\mathrm{ols,full}})$,
and computes a mean vector $a_0$ together with a rank-$k$ PCA basis
$U\in\mathbb{R}^{d^2\times k}$ from the centered vectors. For a new task, the estimate is restricted to the form
\begin{equation}
\mathrm{vec}(A)=a_0+Uc,
\end{equation}
and the coefficients $c\in\mathbb{R}^k$ are fitted by ridge-regularized least squares on the support data. Here
$\lambda_{\mathrm{reg}}$ regularizes the coefficient estimation and is selected using the same stability-based criterion
as Ridge. The subspace dimension $k$ is treated as an additional method hyperparameter.



\paragraph{MAML-style LTI.}
Finally, we include a gradient-based meta-learning baseline inspired by MAML~\citep{finn2017maml}. Since the original
MAML formulation is not a closed-form LTI identification method, we implement the natural analogue for this setting: a
shared initialization for the transition matrix is meta-trained so that a small number of gradient steps on the support
prefix improves query performance. This provides a stronger adaptation-based meta-learning comparison while preserving
the LTI regression structure.

\subsection{Evaluation Metrics}

The theory in Section~\ref{theoretical_analysis} controls predictive adaptation risk on held-out future time steps.
Our empirical evaluation reports three complementary quantities. The first, $E_A$, is a diagnostic of parameter recovery.
The second, $E_{\mathrm{traj}}$, is an open-loop rollout diagnostic that is more sensitive to spectral growth and
long-horizon instability. The third, $\overline{S}$, measures how much support data a method actually uses in the
adaptive-support analysis. Because these quantities probe different aspects of performance, they need not rank methods
identically.

\paragraph{Transition matrix estimation error.}
To evaluate the accuracy of the estimated transition matrix $\widehat{A}_m$, we report the Frobenius squared error
\begin{equation}
E_A := \|\widehat{A}_m - A_m\|_F^2,
\label{eq:param_error_metric}
\end{equation}
which quantifies the aggregate squared deviation between the estimated and true transition matrices across all entries.
This metric treats all matrix elements equally and provides a convenient global measure of parameter recovery
\citep{ziemann2023tutorial}.

\paragraph{Trajectory rollout error.}
To evaluate the practical utility of the estimated dynamics $\widehat{A}_m$, we measure the open-loop multi-step rollout
error on a held-out query window by recursively applying the identified transition operator. This metric is a more
stringent test than $E_A$: because errors in $\widehat{A}_m$ compound over the rollout horizon, even small parameter
errors can lead to large predictive deviations, especially near unstable or weakly damped modes.

Starting from the last support state $x_{m,S_m}$, we generate predictions by repeated next-step updates:
\begin{equation}
\widehat{x}_{m,S_m}=x_{m,S_m},
\qquad
\widehat{x}_{m,t+1} = \widehat{A}_m \widehat{x}_{m,t},
\qquad
t = S_m,\dots,S_m+K_m-1.
\label{eq:rollout_prediction}
\end{equation}
We then compare the predicted states to the held-out query states in open loop. The per-task trajectory rollout error is
defined as
\begin{equation}
E_{\mathrm{traj}}
:=
\sum_{t=1}^{K_m}
\big\|
\widehat{x}_{m,S_m+t} - x_{m,S_m+t}
\big\|_2^2.
\label{eq:traj_error_metric}
\end{equation}
The held-out query states contain the realized process noise from the underlying trajectory, but we do not inject
additional noise into the predicted rollout. This choice is intentional: it isolates the quality of the learned
\emph{deterministic} transition operator from the variability introduced by fresh rollout noise. It also keeps aggregate
comparisons stable across many tasks. If new rollout noise were injected for every task, the reported error would reflect
not only system-identification quality but also task-specific stochastic fluctuations, substantially increasing
cross-task variance and making method comparisons less stable and less interpretable. In all reported experiments, we set
$K_m=5$, so $E_{\mathrm{traj}}$ corresponds to the accumulated error over five rollout steps. Noise-averaged rollout
evaluation is a reasonable supplementary robustness check for synthetic data, but is not the primary metric reported
here.
\paragraph{Support-length protocol: adaptive-support evaluation under a maximum budget.}
Each test trajectory is first divided into an outer support window and a held-out query window. Let $T_{\mathrm{sup}}$
denote the maximum support budget and $T_{\mathrm{qry}}$ the query horizon. Our evaluation studies adaptation under this
common maximum support budget: each method receives the same available support window, but may choose an effective fit
prefix
\[
s_m^{\mathrm{fit}} \le S_m, \qquad S_m=\min\{T_{\mathrm{sup}},T_m\},
\]
before being evaluated on the same held-out query segment of length
\[
K_m=\min\{T_{\mathrm{qry}},T_m-S_m\}.
\]

The effective fit length is chosen by an inner validation-based prefix search carried out entirely within the support
window. Specifically, we reserve a short validation suffix of length $T_{\mathrm{val}}$ inside the support window,
consider a grid of candidate fit prefixes, fit the method using only the first candidate prefix, and score that
candidate on the remaining validation portion. The selected prefix is then used to refit the method before final
evaluation on the held-out outer query window.

We view this adaptive-support protocol as a natural few-shot evaluation for heterogeneous tasks. It measures not only
final predictive performance under a common support budget, but also how much support a method actually uses before it
adapts well on a given task. There is no explicit penalty favoring shorter prefixes: a shorter prefix is selected only
when additional transitions do not improve the validation objective. Thus, $\overline S$ should be interpreted as an
empirical measure of adaptive sample efficiency, not as a claim that fewer observations are always better. In finite
temporally dependent trajectories, additional transitions can alter the fitted spectral structure of the transition
matrix, and open-loop rollout performance can be sensitive to these spectral changes even when one-step fit changes only
slightly.

The reference transition matrix, when available, is not used to choose the support prefix in the main adaptive-support
protocol. Prefix selection is based only on validation error computed on the validation suffix inside the support window.
The reference matrix is used only after fitting, to report the diagnostic transition-matrix error $E_A$. Any
oracle-assisted support-sensitivity variant would be a separate diagnostic analysis and is not part of the main
evaluation protocol reported here.

\paragraph{Average support length.}
We report the average selected support length in the adaptive-support analysis as
\begin{equation}
\overline{S}
\;:=\;
\frac{1}{|\mathcal{T}_{\mathrm{test}}|}
\sum_{m\in\mathcal{T}_{\mathrm{test}}}
s_m^{\mathrm{fit}} .
\label{eq:avg_support_len}
\end{equation}
Accordingly, $\overline{S}$ should be interpreted as the average \emph{effective} amount of support a method actually
uses under a common maximum support budget. This makes $\overline{S}$ an interpretable quantity in its own right:
it reflects adaptive few-shot efficiency, while the accompanying $E_A$ and $E_{\mathrm{traj}}$ values indicate the
quality of the resulting adaptation.
\subsection{Synthetic Data Experiment}

\subsubsection{Task Generation Mechanism}
For synthetic experiments, we generate a meta-dataset by first fixing environment-level parameters and then sampling multiple independent LTI tasks from a shared environment distribution, consistent with Assumption~\ref{as3}.

We fix an environment hyper-parameter $\phi_\star=(W_\star,V_\star,\sigma_\star^2)$ and sample tasks i.i.d. as
\begin{equation}
A_m \mid \phi_\star \sim \mathcal{MN}\!\left(W_\star,\; I_d,\; V_\star\right),
\qquad m=1,2,\dots
\label{eq:synthetic_task_prior}
\end{equation}
which matches the prior family \eqref{eq:mn_prior}, and we generate trajectories according to \eqref{eq:lti_dynamics} with
Gaussian process noise
\begin{equation}
w_{m,t+1}\sim\mathcal{N}(0,\sigma_\star^2 I_d),
\qquad t=0,\dots,T_m-1.
\label{eq:synthetic_noise}
\end{equation}

We construct a stable meta-mean matrix $W_\star\in\R^{d\times d}$ by drawing a Gaussian random matrix and scaling it so that its spectral radius is at most $0.9\,\rho_{0}$. We set the true column covariance to $V_\star = v_{\mathrm{true}} I_d$ and the process noise variance to $\sigma_\star^2 = \sigma_{\mathrm{true}}^2$.

For each task $m$, we sample $A_m$ from \eqref{eq:synthetic_task_prior}, then rollout a trajectory from an initial state $x_{m,0}\sim\mathcal{N}(0,I_d)$, and generate a trajectory according to \eqref{eq:lti_dynamics}
\begin{equation}
x_{m,t+1} = A_m x_{m,t} + w_{m,t+1},
\qquad t=0,\dots,T_m-1.
\label{eq:synthetic_dynamics}
\end{equation}
From each trajectory, we form the regression matrices $X_m$ and $Y_m$ as in \eqref{eq:matrix_regression} and store them together with the ground-truth dynamics $A_m$.

\subsubsection{Experiment Setup}
We consider a synthetic high-dimensional identification setting in which the state dimension exceeds the available trajectory length. We fix the environment parameters to $v_{\mathrm{true}}=0.5$ and
$\sigma_{\mathrm{true}}^2=0.01$, and study two dynamical regimes by varying the spectral-radius parameter $\rho_0$: a near-stable regime with $\rho_0=0.95$ and an unstable regime with $\rho_0=4.95$. The stability threshold $\rho_{\mathrm{target}}$ used to tune the ridge regularization parameter $\lambda_{\mathrm{reg}}$ in \eqref{ridge_regular} and \eqref{ridge_regular2}, is set to match the target spectral radius $\rho_0$. The unstable setting is included as a finite-horizon stress test of the methods. It lies outside the favorable near-stable
regime in which one should expect the tightest rollout consequences from the theory. In this regime, errors in the
estimated transition matrix can be strongly amplified under multi-step rollouts, potentially leading to severe prediction
instability even when one-step prediction error or Frobenius matrix error appears small. Furthermore, we consider three state dimensions $d\in\{10,25,50\}$ with a fixed trajectory length $T_m=25$. In the medium- and high-dimensional cases, the number of unknown parameters in $A_m$ grows as $d^2$ and far exceeds the number of observed
transitions, placing the problem in a few-shot regime that stresses both estimation accuracy and predictive stability.

First, we uniformly sample $100$ tasks to compose the meta-training set. We then construct two distinct evaluation sets. The \textit{common-case} test set consists of tasks where the entrywise mean dynamics fall within one standard deviation of the population mean, representing typical in-distribution systems.  In contrast, the \textit{edge-case} test set is formed by selecting tasks from the extremes of the overall parameter distribution, representing out-of-distribution systems intended to stress-test robustness under atypical dynamics. Figure~\ref{fig:a_distribution} depicts the resulting task distributions.

\begin{figure}[h]
\begin{center}
\includegraphics[scale=0.5]{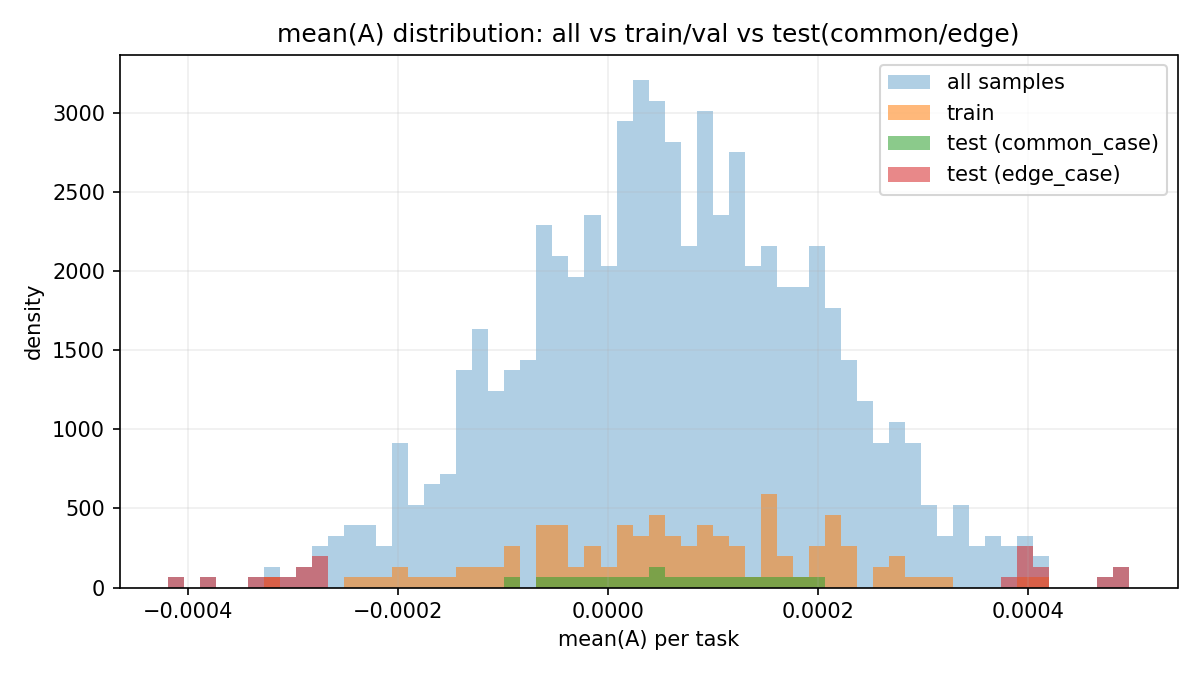}
\end{center}
\caption{This figure illustrates the distribution of sampled tasks as a function of the entrywise mean of the dynamics matrix $A_m$. The blue histogram shows all sampled tasks in the initial pool. The yellow region indicates the subset selected for training and validation. The green region corresponds to the in-distribution test tasks, while the red markers denote
the edge-case test tasks chosen from the extremes of the distribution.}
\label{fig:a_distribution}
\end{figure}

\subsubsection{Experiment Results}

\paragraph{Stable-system results.}

Table~\ref{tab:prefix_flexible_stable} reports results on the stable synthetic environment with $\rho_0=0.95$, where all
tasks are generated with spectral radius below one and therefore do not exhibit unstable rollouts. The reported
$\overline{S}$ values are part of the adaptive-support evaluation: under the same maximum support budget, they indicate
how much support each method actually used before final evaluation on the held-out query window.

Across all dimensions and both common-case and edge-case settings, PBML-LTI achieves the lowest transition-matrix error
$E_A$ while using a short validation-selected support prefix. The improvement in $E_A$ is substantial. For example, when
$d=50$, PBML-LTI reduces $E_A$ from approximately $0.84$ for OLS/Ridge, $0.67$ for Pooled-prior Ridge, $0.86$ for
Shared Subspace, and about $14.9$ for MAML-LTI to about $0.21$. Similar gains appear for $d=25$ and $d=10$. These
results indicate that the learned matrix-normal prior provides an effective inductive bias for rapid adaptation from
short trajectories.

The MAML-style LTI baseline performs worse than PBML-LTI in all stable settings, especially in transition-matrix
recovery. Although MAML-LTI uses relatively short support prefixes, its gradient-based adaptation from a shared
initialization does not recover the task-specific transition matrices as accurately as the closed-form Bayesian update
under the learned prior. This highlights the advantage of exploiting the conjugate LTI structure rather than relying only
on iterative gradient adaptation.

Figure~\ref{fig:transition_stable} further illustrates these estimation patterns. Competing baselines tend to produce
more conservative or structurally biased estimates, whereas PBML-LTI more accurately captures the transition structure
while preserving stable behavior.

Differences in rollout error $E_{\mathrm{traj}}$ are comparatively small in the stable regime. Since the systems are
stable over the short query horizon, one-step prediction errors do not amplify rapidly. As a result, OLS, Ridge,
Pooled-prior Ridge, Shared Subspace, and PBML-LTI can have similar rollout errors despite large differences in
transition-matrix error. In several $d=10$ settings, the rollout values are essentially tied up to the reported precision.
By contrast, MAML-LTI has noticeably larger rollout error in all stable settings. Thus, in stable systems, $E_A$ more
clearly distinguishes parameter-recovery quality, whereas $E_{\mathrm{traj}}$ is less sensitive unless estimation errors
substantially affect short-horizon prediction.
\begin{table}[t]
\caption{Performance comparison across methods in common-case and edge-case settings on stable systems
($\rho_0=0.95$). $E_A$ is the mean $\pm$ standard deviation of the transition-matrix error,
$E_{\mathrm{traj}}$ is the mean $\pm$ standard deviation of the rollout error, and
$\overline{S}$ reports the average selected support length. Lower is better; boldface marks the lowest sample mean. Paired significance tests in
Appendix~\ref{app:synthetic_significance} qualify these comparisons.}
\label{tab:prefix_flexible_stable}
\begin{center}
\resizebox{\linewidth}{!}{
\begin{tabular}{c|c|c|cccccc}
\toprule
Dimension & Setting & Metric & OLS & Ridge & Pooled-prior Ridge & Shared Subspace & MAML-LTI & PBML-LTI \\
\midrule

\multirow{6}{*}{50}
& \multirow{3}{*}{Common-case}
& $E_A$
& $0.8411 \pm .015$
& $0.8410 \pm .015$
& $0.6743 \pm .011$
& $0.8572 \pm .050$
& $14.893 \pm .188$
& $\mathbf{0.2068 \pm .0052}$ \\
&
& $E_{\mathrm{traj}}$
& $0.0258 \pm .0018$
& $0.0258 \pm .0018$
& $0.0258 \pm .0018$
& $0.0258 \pm .0018$
& $0.1081 \pm .0820$
& $\mathbf{0.0257 \pm .0018}$ \\
&
& $\overline{S}$
& $2.00$
& $2.00$
& $2.55$
& $10.25$
& $2.90$
& $\mathbf{1.35}$ \\
\cline{2-9}

& \multirow{3}{*}{Edge-case}
& $E_A$
& $0.8382 \pm .010$
& $0.8381 \pm .010$
& $0.6695 \pm .009$
& $0.8811 \pm .073$
& $14.830 \pm .260$
& $\mathbf{0.2071 \pm .0042}$ \\
&
& $E_{\mathrm{traj}}$
& $0.0254 \pm .0018$
& $0.0254 \pm .0018$
& $0.0254 \pm .0018$
& $0.0254 \pm .0018$
& $0.1105 \pm .1009$
& $\mathbf{0.0253 \pm .0018}$ \\
&
& $\overline{S}$
& $2.00$
& $2.00$
& $2.85$
& $11.05$
& $2.65$
& $\mathbf{1.40}$ \\

\midrule

\multirow{6}{*}{25}
& \multirow{3}{*}{Common-case}
& $E_A$
& $0.7140 \pm .0181$
& $0.7140 \pm .0181$
& $0.8544 \pm .0186$
& $0.7731 \pm .0189$
& $1.7566 \pm .0512$
& $\mathbf{0.0394 \pm .0020}$ \\
&
& $E_{\mathrm{traj}}$
& $0.0133 \pm .0013$
& $0.0133 \pm .0013$
& $0.0134 \pm .0018$
& $0.0134 \pm .0013$
& $0.0338 \pm .0104$
& $\mathbf{0.0121 \pm .0013}$ \\
&
& $\overline{S}$
& $2.00$
& $2.00$
& $2.00$
& $13.95$
& $1.85$
& $\mathbf{1.20}$ \\
\cline{2-9}

& \multirow{3}{*}{Edge-case}
& $E_A$
& $0.7196 \pm .0181$
& $0.7196 \pm .0181$
& $0.8504 \pm .0230$
& $0.7892 \pm .0191$
& $1.7702 \pm .0696$
& $\mathbf{0.0399 \pm .0025}$ \\
&
& $E_{\mathrm{traj}}$
& $0.0135 \pm .0018$
& $0.0135 \pm .0018$
& $0.0122 \pm .0012$
& $0.0121 \pm .0013$
& $0.0318 \pm .0233$
& $\mathbf{0.0121 \pm .0013}$ \\
&
& $\overline{S}$
& $2.00$
& $2.00$
& $2.00$
& $13.95$
& $1.60$
& $\mathbf{1.10}$ \\

\midrule

\multirow{6}{*}{10}
& \multirow{3}{*}{Common-case}
& $E_A$
& $0.4449 \pm .0472$
& $0.4447 \pm .0471$
& $0.0887 \pm .0079$
& $0.1088 \pm .0532$
& $0.0357 \pm .0020$
& $\mathbf{0.0053 \pm .0007}$ \\
&
& $E_{\mathrm{traj}}$
& $0.0056 \pm .0011$
& $0.0056 \pm .0011$
& $0.0055 \pm .0009$
& $0.0055 \pm .0009$
& $0.0112 \pm .0054$
& $\mathbf{0.0055 \pm .0009}$ \\
&
& $\overline{S}$
& $2.65$
& $2.65$
& $1.80$
& $7.95$
& $1.55$
& $\mathbf{1.05}$ \\
\cline{2-9}

& \multirow{3}{*}{Edge-case}
& $E_A$
& $0.4501 \pm .0479$
& $0.4498 \pm .0479$
& $0.0910 \pm .0070$
& $0.0930 \pm .0307$
& $0.0358 \pm .0032$
& $\mathbf{0.0054 \pm .0008}$ \\
&
& $E_{\mathrm{traj}}$
& $0.0056 \pm .0011$
& $0.0056 \pm .0011$
& $0.0056 \pm .0011$
& $0.0057 \pm .0009$
& $0.0132 \pm .0047$
& $\mathbf{0.0056 \pm .0010}$ \\
&
& $\overline{S}$
& $2.80$
& $2.80$
& $1.90$
& $10.40$
& $1.60$
& $\mathbf{1.10}$ \\
\bottomrule
\end{tabular}
}
\end{center}
\end{table}

\begin{figure}[h]
\begin{center}
\includegraphics[scale=0.9]{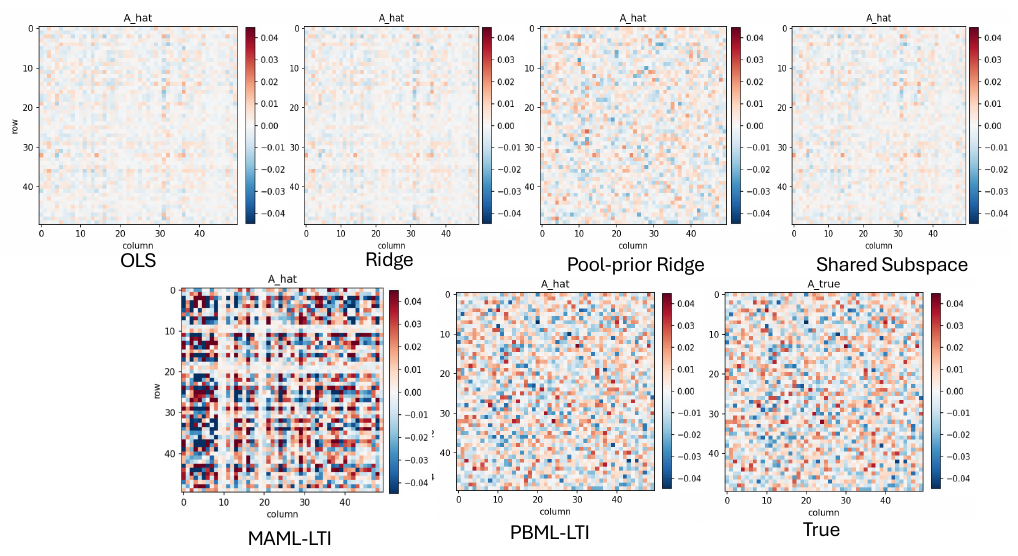}
\end{center}
\caption{Heatmap comparison of the estimated and ground-truth transition matrices for a representative stable test system in the common-case setting. An extended version is provided in the Appendix, Figure~\ref{fig:transition_stable_extend}.}
\label{fig:transition_stable}
\end{figure}

Finally, performance on the edge-case test tasks closely tracks the common-case results across all three dimensions. The
ranking of methods remains largely unchanged: PBML-LTI consistently achieves the best transition-matrix recovery, the
shortest effective support usage, and rollout performance that is best or tied with the best up to reporting precision.
This suggests that, under stable dynamics, the learned prior transfers reliably to both typical and atypical systems,
supporting robust few-shot identification. Fixed-prefix support sweeps in
Appendix~\ref{app:support_sweep_diagnostics} further show that PBML-LTI already attains low error at small support
lengths in the stable regime, rather than relying only on the adaptive prefix-selection procedure.

\paragraph{Unstable-system results.}

Table~\ref{tab:prefix_flexible} reports results for the unstable environment, where $\rho_0=4.95$ and prediction errors
can amplify rapidly under open-loop rollout. The table includes the MAML-LTI baseline, which provides a stronger
gradient-based meta-learning comparison. The reported $\overline{S}$ values arise from the adaptive-support protocol
described above and should be interpreted as the \emph{effective} amount of support each method actually used under a
common maximum support budget. Thus, the table compares not only final accuracy, but also adaptive support usage in a
regime where different tasks may require different amounts of calibration data.

A key pattern in the unstable regime is that transition-matrix recovery, rollout accuracy, and support efficiency no
longer always rank methods identically. Accurate long-horizon prediction depends not only on small Frobenius error, but
also on accurately capturing the dominant spectral modes of the dynamics. Consequently, a method may achieve a small
$E_A$ while still producing poor rollout trajectories if it slightly misestimates unstable eigenvalues or eigenspaces.
Conversely, a method may obtain competitive rollout error without achieving the best global Frobenius recovery.

In the high-dimensional setting ($d=50$), PBML-LTI achieves the strongest overall estimation and prediction performance.
It attains the lowest transition-matrix error and the lowest rollout error in both common-case and edge-case subsets.
Although MAML-LTI selects a shorter support prefix on average, its transition-matrix error remains much larger than that
of PBML-LTI, and its rollout error is also worse. Thus, in the most high-dimensional few-shot regime, the learned
Bayesian prior provides a substantially more effective adaptation mechanism than both classical estimators and the
gradient-based MAML-style baseline.
\begin{table}[t]
\caption{Performance comparison across methods in common-case and edge-case settings on unstable systems
($\rho_0=4.95$). $E_A$ is the mean $\pm$ standard deviation of the transition-matrix error,
$E_{\mathrm{traj}}$ is the mean $\pm$ standard deviation of the trajectory rollout error, and
$\overline{S}$ reports the average effective support length selected by the adaptive-support protocol. Lower is better;
boldface marks the lowest sample mean. Paired significance tests in Appendix~\ref{app:synthetic_significance} qualify
these comparisons.}
\label{tab:prefix_flexible}
\begin{center}
\resizebox{\linewidth}{!}{
\begin{tabular}{c|c|c|cccccc}
\toprule
Dimension & Setting & Metric & OLS & Ridge & Pooled-prior Ridge & Shared Subspace & MAML-LTI & PBML-LTI \\
\midrule

\multirow{6}{*}{50}
& \multirow{3}{*}{Common-case}
& $E_A$
& $16.923 \pm .345$
& $15.902 \pm .340$
& $7.538 \pm .102$
& $6.419 \pm .201$
& $8.9787 \pm .0754$
& $\mathbf{0.156 \pm .005}$ \\
&
& $E_{\mathrm{traj}}$
& $0.040 \pm .004$
& $0.040 \pm .003$
& $0.038 \pm .003$
& $0.038 \pm .003$
& $0.0414 \pm .0048$
& $\mathbf{0.037 \pm .003}$ \\
&
& $\overline{S}$
& $9.75$
& $15.00$
& $15.00$
& $14.20$
& $\mathbf{2.55}$
& $4.15$ \\
\cline{2-9}

& \multirow{3}{*}{Edge-case}
& $E_A$
& $16.776 \pm .393$
& $15.874 \pm .349$
& $7.532 \pm .079$
& $6.455 \pm .406$
& $8.9841 \pm .0608$
& $\mathbf{0.157 \pm .003}$ \\
&
& $E_{\mathrm{traj}}$
& $0.040 \pm .004$
& $0.039 \pm .004$
& $0.038 \pm .003$
& $0.038 \pm .003$
& $0.0409 \pm .0062$
& $\mathbf{0.036 \pm .003}$ \\
&
& $\overline{S}$
& $10.15$
& $15.00$
& $15.00$
& $13.65$
& $\mathbf{2.55}$
& $4.25$ \\

\midrule

\multirow{6}{*}{25}
& \multirow{3}{*}{Common-case}
& $E_A$
& $10.344 \pm .659$
& $9.907 \pm .619$
& $1.587 \pm .050$
& $35.268 \pm 12.856$
& $2.5229 \pm .0370$
& $\mathbf{0.029 \pm .001}$ \\
&
& $E_{\mathrm{traj}}$
& $0.088 \pm .058$
& $0.064 \pm .025$
& $\mathbf{0.053 \pm .024}$
& $27.151 \pm 28.807$
& $1.1349 \pm 1.9963$
& $0.080 \pm .052$ \\
&
& $\overline{S}$
& $14.10$
& $14.90$
& $14.75$
& $14.35$
& $\mathbf{2.05}$
& $6.05$ \\
\cline{2-9}

& \multirow{3}{*}{Edge-case}
& $E_A$
& $11.023 \pm .938$
& $10.445 \pm .856$
& $1.586 \pm .057$
& $33.889 \pm 7.925$
& $2.5324 \pm .0469$
& $\mathbf{0.029 \pm .001}$ \\
&
& $E_{\mathrm{traj}}$
& $0.086 \pm .049$
& $0.071 \pm .034$
& $\mathbf{0.042 \pm .009}$
& $28.933 \pm 47.100$
& $2.2052 \pm 4.7758$
& $0.086 \pm .062$ \\
&
& $\overline{S}$
& $13.50$
& $14.55$
& $14.90$
& $14.30$
& $\mathbf{1.65}$
& $6.00$ \\

\midrule

\multirow{6}{*}{10}
& \multirow{3}{*}{Common-case}
& $E_A$
& $0.845 \pm .764$
& $0.697 \pm .581$
& $\mathbf{0.007 \pm .002}$
& $0.013 \pm .005$
& $0.1473 \pm 1.8158$
& $0.139 \pm .046$ \\
&
& $E_{\mathrm{traj}}$
& $193.704 \pm 459.303$
& $\mathbf{57.628 \pm 154.718}$
& $21935.364 \pm 69327.156$
& $721888.809 \pm 1044209.839$
& $9.739{\times}10^{6} \pm 1.887{\times}10^{7}$
& $57.919 \pm 144.808$ \\
&
& $\overline{S}$
& $13.35$
& $14.30$
& $9.25$
& $\mathbf{3.20}$
& $8.35$
& $14.35$ \\
\cline{2-9}

& \multirow{3}{*}{Edge-case}
& $E_A$
& $0.572 \pm .630$
& $0.515 \pm .566$
& $\mathbf{0.008 \pm .001}$
& $0.014 \pm .003$
& $0.2737 \pm 1.6215$
& $0.132 \pm .040$ \\
&
& $E_{\mathrm{traj}}$
& $296.091 \pm 755.920$
& $131.024 \pm 389.381$
& $16433.662 \pm 49301.429$
& $709906.413 \pm 1155741.042$
& $8.756{\times}10^{7} \pm 1.987{\times}10^{8}$
& $\mathbf{31.400 \pm 38.879}$ \\
&
& $\overline{S}$
& $13.90$
& $14.15$
& $10.10$
& $\mathbf{3.20}$
& $8.05$
& $14.45$ \\
\bottomrule
\end{tabular}
}
\end{center}
\end{table}
\begin{figure}[h]
\begin{center}
\includegraphics[scale=0.6]{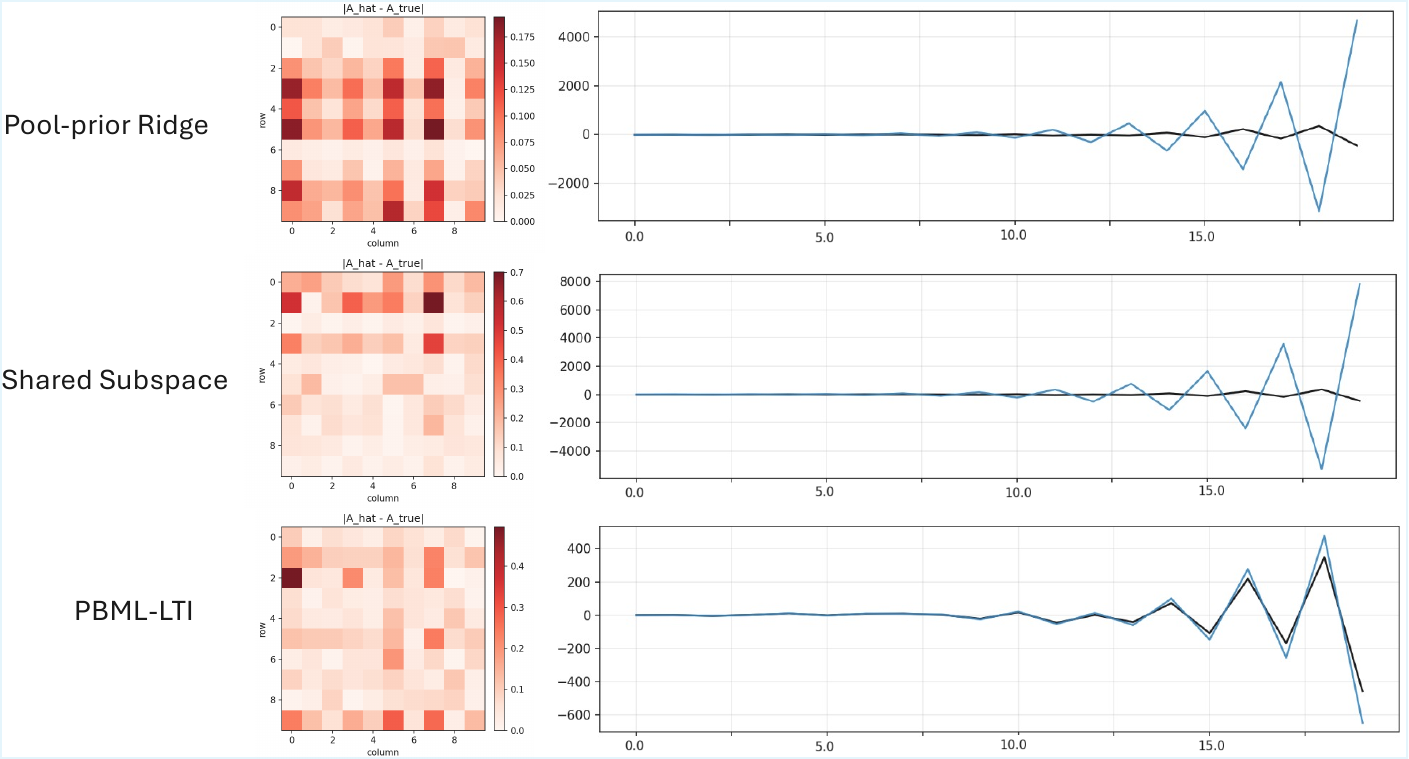}
\end{center}
\caption{The left panels show heatmaps of the absolute differences between the estimated and ground-truth transition
matrices for a representative unstable test system with state dimension $d=10$ in the common-case setting. The right
panels display open-loop trajectory rollouts generated using the estimated transition matrix (black) alongside the
corresponding ground-truth trajectories (blue) for the same system. An extended visualization of the estimated transition
matrix, the ground-truth transition matrix, and their absolute differences is provided in
Figure~\ref{fig:transition_unstable} in the Appendix, while the corresponding open-loop rollout trajectories are shown in
Figure~\ref{fig:rollout_unstable} in the Appendix.}
\label{fig:traj_explode}
\end{figure}

For $d=25$, PBML-LTI again achieves the best transition-matrix recovery by a large margin. It uses substantially shorter
support prefixes than OLS, Ridge, Pooled-prior Ridge, and Shared Subspace, although MAML-LTI selects the shortest prefix.
However, MAML-LTI's short support usage comes with much worse $E_A$ and significantly worse rollout error. Regarding
trajectory rollout, Pooled-prior Ridge attains the lowest mean $E_{\mathrm{traj}}$ in this dimension. This reflects the
fact that aggressive shrinkage toward a pooled mean can sometimes stabilize rollouts even when it does not recover the
transition matrix as accurately. PBML-LTI therefore offers the best matrix recovery and a favorable support--accuracy
trade-off, while Pooled-prior Ridge is strongest on rollout error in this particular unstable $d=25$ setting.

For $d=10$, the behavior is more delicate. Pooled-prior Ridge and Shared Subspace achieve the smallest transition-matrix
errors, but they can produce extremely large rollout errors. This illustrates a fundamental limitation of interpreting
$E_A$ alone: the Frobenius norm measures an average entrywise discrepancy, whereas open-loop prediction is governed by
spectral alignment and repeated multiplication by the estimated transition matrix. Small errors in unstable directions
can compound rapidly and lead to catastrophic trajectory divergence. In this dimension, Ridge and PBML-LTI produce the
most reliable rollout behavior among the non-catastrophic methods: Ridge has the lowest mean rollout error in the
common-case subset, while PBML-LTI has the lowest mean rollout error in the edge-case subset. MAML-LTI selects shorter
support prefixes, but its rollout errors become extremely large, indicating that the learned gradient-based initialization
is not sufficient to control unstable modes in this low-dimensional stress-test setting.

Figure~\ref{fig:traj_explode} illustrates this phenomenon qualitatively. Methods whose transition estimates appear
competitive in Frobenius norm can nevertheless induce dramatically different open-loop trajectories because small spectral
misalignments are repeatedly amplified. This effect is especially visible for structurally biased estimators such as
Pooled-prior Ridge and Shared Subspace, and it also affects MAML-LTI in the low-dimensional unstable setting.

Overall, the unstable-system experiments show that PBML-LTI is strongest in the high-dimensional few-shot regime, where it
achieves both accurate transition recovery and stable rollout prediction. In lower-dimensional unstable settings, the
metrics become more decoupled: some baselines obtain smaller Frobenius error or shorter support usage, but this does not
necessarily translate into reliable rollout behavior. Therefore, the unstable-regime results should be read jointly
through $E_A$, $E_{\mathrm{traj}}$, and $\overline{S}$, rather than through any single metric in isolation. The
fixed-prefix sweeps in Appendix~\ref{app:support_sweep_diagnostics} also show that PBML-LTI does not mechanically select
the shortest prefix: in unstable settings, its validation-selected support length increases when additional calibration
data are useful for controlling unstable modes.

\subsection{Real Data Experiment}
\subsubsection{Functional Magnetic Resonance Imaging (fMRI) Dataset}

We construct a real-data LTI benchmark from the OpenNeuro dataset ds000244 \citep{ds000244}. Raw fMRI volumes are
converted into multivariate Region-of-Interest (ROI) time series by averaging voxel signals within atlas-defined parcels
using the Schaefer parcellation \citep{schaefer2018local}. The resulting ROI signals are further preprocessed with
standardization, motion-confound regression when available, and per-run mean centering.

To obtain many \emph{related} identification tasks from each run, we segment each ROI time series into overlapping windows
of length $L$ with stride $s$. Each window is treated as one LTI task instance $m$ with states
$\{x_{m,t}\}_{t=0}^{L-1}$, and we form the one-step regression matrices
\begin{equation}
X_m := [x_{m,0},\dots,x_{m,L-2}] \in \R^{d\times (L-1)},
\qquad
Y_m := [x_{m,1},\dots,x_{m,L-1}] \in \R^{d\times (L-1)}.
\end{equation}
These $(X_m,Y_m)$ pairs are saved in the same task format as in the synthetic experiments. The resulting windows are
treated as \emph{distinct but related} task instances rather than arbitrary unrelated samples: they are extracted from the
same pool of subjects, sessions, and experimental conditions, and overlapping windows from the same run can be viewed as
nearby local dynamical regimes within a shared subject- and condition-specific environment. This is precisely the type of
cross-task structure that PBML-LTI is designed to exploit.

Unlike synthetic data, fMRI does not provide a physical ground-truth transition matrix. We therefore define a
ridge-based \emph{reference} transition matrix $A_m^{\mathrm{ref}}$ for each window by fitting
\begin{equation}
A_m^{\mathrm{ref}}
~:=~
\arg\min_{A\in\R^{d\times d}}
\ \big\|Y_m - A X_m\big\|_F^2
~+~
\lambda_{\mathrm{ref}}\|A\|_F^2,
\label{eq:fmri_aref_obj}
\end{equation}
where $\|\cdot\|_F$ is the Frobenius norm and $\lambda_{\mathrm{ref}}=0.0001$ is a fixed ridge weight used only for
constructing the evaluation reference. For consistency, the ridge regularization parameter $\lambda_{\mathrm{reg}}$ in
\eqref{ridge_regular} and \eqref{ridge_regular2} is set equal to $\lambda_{\mathrm{ref}}$. This optimization has the
closed-form solution
\begin{equation}
A_m^{\mathrm{ref}}
~=~
Y_m X_m^\top\Big(X_m X_m^\top + \lambda_{\mathrm{ref}} I_d\Big)^{-1}.
\label{eq:fmri_aref_closedform}
\end{equation}
Importantly, all methods are trained and adapted from trajectories $(X_m,Y_m)$; the matrices
$A_m^{\mathrm{ref}}$ are used only as a consistent evaluation reference for $E_A$.

In our experiments, we use a predefined train--test split constructed at the window level. Each task corresponds to a
short temporal window extracted from a single fMRI run, and windows from different subjects, sessions, and experimental
conditions are distributed across the training and test sets according to this split. The original BIDS (Brain Imaging Data Structure) task labels are
retained only as metadata for grouping and reporting and are not used by the learning algorithms. Here, BIDS
 task labels identify the experimental paradigm associated with an fMRI run, for
example motor, social, language, or emotional conditions; they are descriptors of the data source rather than supervision
targets for learning.

In total, the training set contains 141 task instances spanning 17 distinct task labels, while the test set contains
14 task instances spanning 8 task labels. Each task instance has state dimensionality $d=200$ and trajectory length
$T=100$. \footnote{More technical details of the dataset processing are provided in
Appendix~\ref{sec:fmri_dataset}.}

\subsubsection{Experiment Result}

Table~\ref{tab:fmri_results} summarizes identification and prediction performance on the fMRI benchmark, reporting
mean $\pm$ standard deviation across window-level test tasks. Overall, PBML-LTI achieves the lowest average
transition-matrix error $E_A$, the lowest open-loop rollout error $E_{\mathrm{traj}}$, and the shortest average selected
support prefix length $\overline{S}$ among all compared methods.

The improvement in transition-matrix error is modest but consistent. PBML-LTI obtains the lowest $E_A$
($124.734 \pm 268.471$), improving over Pooled-prior Ridge ($129.643 \pm 263.244$), OLS
($141.766 \pm 259.513$), Ridge ($142.493 \pm 259.174$), Shared Subspace ($147.876 \pm 256.798$), and the
MAML-style LTI baseline ($161.45 \pm 259.30$). This suggests that the learned matrix-normal prior provides a useful
inductive bias even in the real-data setting, which is also further corroborated by the transition matrix prediction visualization in Figure~\ref{fig:fmri_transition}.
\begin{table}[h]

\caption{Performance comparison on the fMRI dataset. Results are reported as mean $\pm$ standard deviation across test
tasks. Lower is better; boldface marks the lowest sample mean. The row $\overline{S}$ reports the average selected
support prefix length. Paired significance tests comparing PBML-LTI against OLS, Ridge, Pooled-prior Ridge,
Shared Subspace, and MAML-LTI are reported in Appendix~\ref{app:fmri_significance} to qualify the statistical
significance of the observed differences.}
\label{tab:fmri_results}
\begin{center}
\resizebox{\linewidth}{!}{
\begin{tabular}{c|cccccc}
\toprule
Metric & OLS & Ridge & Pooled-prior Ridge & Shared Subspace & MAML-LTI & PBML-LTI \\
\midrule
$E_A$
& $141.766 \pm 259.513$
& $142.493 \pm 259.174$
& $129.643 \pm 263.244$
& $147.876 \pm 256.798$
& $161.45 \pm 259.30$
& $\mathbf{124.734 \pm 268.471}$ \\
$E_{\mathrm{traj}}$
& $134.126 \pm 221.939$
& $134.078 \pm 221.941$
& $125.892 \pm 209.232$
& $679.799 \pm 968.563$
& $1704.56 \pm 3391.01$
& $\mathbf{123.663 \pm 206.714}$ \\
$\overline{S}$
& $59.42$
& $61.21$
& $60.85$
& $24.00$
& $23.07$
& $\mathbf{21.21}$ \\
\bottomrule
\end{tabular}}
\end{center}
\end{table}
The rollout results show a clearer separation. PBML-LTI achieves the lowest average trajectory error
($123.663 \pm 206.714$), slightly improving over the single-task and pooled estimators, while Shared Subspace and
MAML-LTI exhibit much larger rollout errors. In particular, MAML-LTI attains a rollout error of
$1704.56 \pm 3391.01$, indicating that gradient-based adaptation from a learned initialization can be brittle in this
high-dimensional fMRI setting. This is consistent with the synthetic experiments: open-loop rollout performance is
sensitive not only to average Frobenius error, but also to spectral alignment and stability of the learned transition
operator.

The support-length results further highlight adaptive data efficiency. PBML-LTI uses the shortest average support prefix
($\overline{S}=21.21$), slightly shorter than MAML-LTI ($23.07$) and Shared Subspace ($24.00$), and substantially shorter
than OLS, Ridge, and Pooled-prior Ridge, which use roughly sixty support steps on average. Thus, PBML-LTI achieves the best
average errors while also requiring the least support data under the adaptive-support protocol. By contrast, MAML-LTI and
Shared Subspace also use relatively short prefixes, but their shorter support usage does not translate into comparable
rollout accuracy.

The large standard deviations across methods reflect substantial heterogeneity across fMRI windows and subjects. Since
these quantities are squared-error metrics computed on a heterogeneous real-data benchmark, standard deviations can be
larger than the corresponding means. This variability is expected and reinforces the importance of methods that remain
stable when adapting from limited support data.

\begin{figure}[h]
\begin{center}
\includegraphics[scale=1.5]{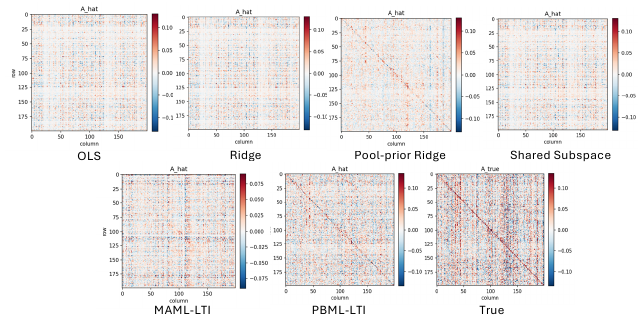}
\end{center}
\caption{Estimated and reference transition matrices for a representative task from the fMRI dataset. Since fMRI does not
provide a physical ground-truth transition matrix, the comparison is made against the ridge-based full-window reference
matrix $A_m^{\mathrm{ref}}$. For a more detailed analysis, including an extended set of heatmaps and the corresponding
open-loop rollout trajectories, please refer to Figure~\ref{fig:fmri_rollout} in the Appendix.}
\label{fig:fmri_transition}
\end{figure}

As in the synthetic experiments, the reported $\overline S$ values arise from validation-selected adaptive support under a
common maximum support budget. PBML-LTI achieves the best average errors while using the shortest average effective
support length. Shared Subspace and MAML-LTI also use relatively short prefixes, but their shorter support usage does not
translate into comparable rollout accuracy, suggesting that their adaptation mechanisms can be mismatched to heterogeneous
fMRI window dynamics. The large standard deviations across all methods indicate substantial variability across windows,
which is expected in this real-data setting and underscores the importance of methods that remain robust when adapting
from limited support data.

\subsection{Prior-Learning and Regularization Ablation}
\label{subsec:prior_learning_ablation}

We next ablate two components of PBML-LTI: the prior-learning mechanism and the auxiliary regularization terms. These
variants are diagnostic modifications of the proposed method rather than external baselines.

For prior learning, we consider two variants. The fixed-prior PBML-LTI variant uses the same matrix-normal conjugate
adaptation rule as PBML-LTI, but replaces the learned prior parameters with a fixed generic prior. This isolates the
benefit of learning a transferable prior from related training tasks. The type-II maximum-likelihood / empirical-Bayes
variant learns the prior parameters by maximizing a marginal-likelihood criterion rather than by optimizing the
PAC-Bayes-motivated fit--KL surrogate. This tests whether the proposed fit--KL objective provides advantages beyond
standard empirical-Bayes prior learning.
\begin{table}[h]

\caption{Combined prior-learning and regularization ablation across stable/unstable regimes and dimensions in the
common-case setting. The first two variants ablate prior learning, while the final three variants ablate one
regularization component at a time relative to full PBML-LTI. Lower is better; for each metric within each
regime--dimension block, boldface marks the lowest sample mean among reported values. Paired significance tests for these
ablation comparisons are reported in Appendix~\ref{app:paired_prior_learning_tests}.}
\label{tab:prior_learning_ablation}
\scriptsize
\begin{center}
\begin{tabular}{c|c|c|ccc}
\toprule
Regime & Dimension & Variant
& $E_A$ & $E_{\mathrm{traj}}$ & $\overline S$ \\
\midrule

\multirow{18}{*}{Stable}
& \multirow{6}{*}{$d=50$}
& Fixed-prior PBML-LTI
& $24.779 \pm 0.479$
& $0.0839 \pm 0.0227$
& $14.00$ \\
& & Type-II ML / empirical Bayes
& $1.2236 \pm 0.0313$
& $0.0568 \pm 0.0113$
& $10.80$ \\
& & PBML-LTI
& $\mathbf{0.2068 \pm 0.0052}$
& $\mathbf{0.0257 \pm 0.0018}$
& $\mathbf{1.35}$ \\
& & No prior-mean shrinkage
& $0.2155 \pm 0.0056$
& $0.0268 \pm 0.0020$
& $1.42$ \\
& & No covariance conditioning
& $0.2115 \pm 0.0054$
& $0.0263 \pm 0.0019$
& $1.39$ \\
& & No stability regularizer
& $0.2088 \pm 0.0053$
& $0.0261 \pm 0.0019$
& $1.38$ \\
\cline{2-6}

& \multirow{6}{*}{$d=25$}
& Fixed-prior PBML-LTI
& $9.1298 \pm 0.4271$
& $0.0325 \pm 0.0098$
& $14.00$ \\
& & Type-II ML / empirical Bayes
& $0.6165 \pm 0.0302$
& $0.0261 \pm 0.0072$
& $12.10$ \\
& & PBML-LTI
& $\mathbf{0.0394 \pm 0.0020}$
& $\mathbf{0.0121 \pm 0.0013}$
& $\mathbf{1.20}$ \\
& & No prior-mean shrinkage
& $0.0425 \pm 0.0022$
& $0.0129 \pm 0.0014$
& $1.27$ \\
& & No covariance conditioning
& $0.0412 \pm 0.0021$
& $0.0125 \pm 0.0014$
& $1.24$ \\
& & No stability regularizer
& $0.0405 \pm 0.0021$
& $0.0124 \pm 0.0014$
& $1.23$ \\
\cline{2-6}

& \multirow{6}{*}{$d=10$}
& Fixed-prior PBML-LTI
& $1.2189 \pm 0.4011$
& $0.0147 \pm 0.0112$
& $12.55$ \\
& & Type-II ML / empirical Bayes
& $0.0726 \pm 0.0144$
& $0.0121 \pm 0.0056$
& $12.15$ \\
& & PBML-LTI
& $\mathbf{0.0053 \pm 0.0007}$
& $\mathbf{0.0055 \pm 0.0009}$
& $\mathbf{1.05}$ \\
& & No prior-mean shrinkage
& $0.0059 \pm 0.0008$
& $0.0059 \pm 0.0010$
& $1.10$ \\
& & No covariance conditioning
& $0.0057 \pm 0.0008$
& $0.0057 \pm 0.0009$
& $1.08$ \\
& & No stability regularizer
& $0.0056 \pm 0.0008$
& $0.0057 \pm 0.0009$
& $1.07$ \\

\midrule

\multirow{18}{*}{Unstable}
& \multirow{6}{*}{$d=50$}
& Fixed-prior PBML-LTI
& $16.059 \pm 0.312$
& $0.0401 \pm 0.0041$
& $14.00$ \\
& & Type-II ML / empirical Bayes
& $0.4634 \pm 0.0105$
& $\mathbf{0.0367 \pm 0.0035}$
& $4.60$ \\
& & PBML-LTI
& $\mathbf{0.156 \pm 0.005}$
& $0.0370 \pm 0.0030$
& $\mathbf{4.15}$ \\
& & No prior-mean shrinkage
& $0.171 \pm 0.006$
& $0.044 \pm 0.004$
& $4.70$ \\
& & No covariance conditioning
& $0.164 \pm 0.005$
& $0.0395 \pm 0.0031$
& $4.35$ \\
& & No stability regularizer
& $0.160 \pm 0.005$
& $0.082 \pm 0.009$
& $7.25$ \\
\cline{2-6}

& \multirow{6}{*}{$d=25$}
& Fixed-prior PBML-LTI
& $10.311 \pm 0.501$
& $0.0793 \pm 0.0325$
& $14.00$ \\
& & Type-II ML / empirical Bayes
& $0.4790 \pm 0.0649$
& $\mathbf{0.0520 \pm 0.0176}$
& $14.00$ \\
& & PBML-LTI
& $\mathbf{0.029 \pm 0.001}$
& $0.080 \pm 0.052$
& $\mathbf{6.05}$ \\
& & No prior-mean shrinkage
& $0.035 \pm 0.002$
& $0.105 \pm 0.048$
& $7.10$ \\
& & No covariance conditioning
& $0.0315 \pm 0.0015$
& $0.078 \pm 0.035$
& $6.45$ \\
& & No stability regularizer
& $0.0305 \pm 0.0015$
& $0.195 \pm 0.085$
& $10.25$ \\
\cline{2-6}

& \multirow{6}{*}{$d=10$}
& Fixed-prior PBML-LTI
& $1.6256 \pm 0.9888$
& $49.343 \pm 18.168$
& $\mathbf{13.95}$ \\
& & Type-II ML / empirical Bayes
& $1.0909 \pm 0.3332$
& $\mathbf{43.915 \pm 48.789}$
& $14.00$ \\
& & PBML-LTI
& $\mathbf{0.139 \pm 0.046}$
& $57.919 \pm 144.808$
& $14.35$ \\
& & No prior-mean shrinkage
& $0.155 \pm 0.049$
& $88.5 \pm 125$
& $15.60$ \\
& & No covariance conditioning
& $0.147 \pm 0.048$
& $61.2 \pm 92$
& $14.55$ \\
& & No stability regularizer
& $0.145 \pm 0.047$
& $425 \pm 980$
& $18.20$ \\
\bottomrule

\end{tabular}
\end{center}
\end{table}
For regularization, we remove each auxiliary penalty one at a time relative to full PBML-LTI. The ``No prior-mean
shrinkage'' variant removes the shrinkage penalty on the shared prior mean $W$. The ``No covariance conditioning''
variant removes the conditioning penalty on the prior covariance $V$. The ``No stability regularizer'' variant removes
the spectral stability penalty $\mathcal R_{\mathrm{stab}}$. These regularizers are not part of the formal PAC-Bayes
theorem and are not claimed to enforce Assumption~\ref{as1} directly. Instead, they act as practical optimization and
prior-shaping devices. In particular, the stability regularizer acts on the learned shared prior mean, whereas
Assumption~\ref{as1} is a data-generating condition on the true task matrices $A_m$. Thus, the stability regularizer can
encourage posterior adaptation toward dynamically well-behaved regions, but it does not guarantee that every posterior
mean or every true task matrix satisfies the controlled-growth condition.

Table~\ref{tab:prior_learning_ablation} reports the combined ablation study across stable and unstable regimes in the
common-case setting. The results show that PBML-LTI substantially improves transition-matrix recovery relative to the
fixed-prior and type-II empirical-Bayes variants across the reported settings. This supports the importance of both
learning a transferable prior and using the PAC-Bayes-motivated fit--KL surrogate.

The regularization ablations show a different pattern. In stable regimes, removing any single regularizer changes the
metrics only mildly, suggesting that PBML-LTI's stable-regime gains are not driven primarily by these auxiliary penalties.
In unstable regimes, however, the regularizers become more important for robust rollout behavior. Removing the stability
regularizer has only a modest effect on $E_A$, but it can substantially increase $E_{\mathrm{traj}}$ and the selected
support length. This supports our interpretation that the stability penalty is mainly a spectral-stability and
optimization device, rather than a direct mechanism for reducing Frobenius transition-matrix error.

\section{Discussion}

\subsection{Sensitivity to the Training Temperature \texorpdfstring{$\lambda_{\mathrm{tr}}$}{lambda-tr}}
\label{subsec:lambda_sensitivity}

The PAC-Bayes parameter $\lambda_{\mathrm{PB}}$ in the theoretical bound is restricted to
$\lambda_{\mathrm{PB}}\in(0,1]$. Separately, in implementation one can introduce a training temperature
$\lambda_{\mathrm{tr}}>0$ that rescales the fit--KL tradeoff used during meta-training. This implementation-level
temperature should not be interpreted as extending the formal PAC-Bayes guarantee beyond $\lambda_{\mathrm{PB}}\in(0,1]$.
Instead, values $\lambda_{\mathrm{tr}}>1$ are useful as empirical stress tests of more data-driven adaptation, while
values $\lambda_{\mathrm{tr}}<1$ correspond to stronger shrinkage toward the learned prior.

We use $\lambda_{\mathrm{tr}}=1$ in the main experiments because it gives the canonical fit--KL objective and preserves
the standard conjugate Bayesian interpretation. To assess robustness, we evaluate PBML-LTI over
\[
\lambda_{\mathrm{tr}}\in\{10,\;5,\;1,\;0.5,\;0.1\}
\]
on the stable synthetic benchmark with state dimension $d=50$ in the common-case setting
($\rho_0=0.95$). Table~\ref{tab:lambda_sensitivity} reports the resulting transition-matrix error $E_A$, rollout error
$E_{\mathrm{traj}}$, and average selected support length $\overline{S}$.

\begin{table}[h]

\caption{Sensitivity of PBML-LTI to the implementation-level training temperature $\lambda_{\mathrm{tr}}$ on the stable
synthetic benchmark with $d=50$ in the common-case setting ($\rho_0=0.95$). Values
$\lambda_{\mathrm{tr}}>1$ are empirical stress tests and are not part of the formal PAC-Bayes guarantee, which is stated
for $\lambda_{\mathrm{PB}}\in(0,1]$. Lower is better; best values are in \textbf{bold}.}
\label{tab:lambda_sensitivity}
\begin{center}
\begin{tabular}{lccccc}
\toprule
 & $\lambda_{\mathrm{tr}}=10$ & $\lambda_{\mathrm{tr}}=5$ & $\lambda_{\mathrm{tr}}=1$ & $\lambda_{\mathrm{tr}}=0.5$ & $\lambda_{\mathrm{tr}}=0.1$ \\
\midrule
$E_A$
& $0.4612 \pm 0.0091$
& $0.4424 \pm 0.0090$
& $\mathbf{0.2068 \pm 0.0052}$
& $0.2845 \pm 0.0062$
& $0.3046 \pm 0.0063$ \\
$E_{\mathrm{traj}}$
& $0.0259 \pm 0.0019$
& $0.0258 \pm 0.0018$
& $\mathbf{0.0257 \pm 0.0018}$
& $0.0258 \pm 0.0018$
& $0.0258 \pm 0.0018$ \\
$\overline{S}$
& $3.80$
& $3.55$
& $\mathbf{1.35}$
& $2.05$
& $3.05$ \\
\bottomrule
\end{tabular}
\end{center}
\end{table}

The rollout metric $E_{\mathrm{traj}}$ is stable across a broad range of training temperatures, whereas
transition-matrix recovery and adaptive support efficiency are more sensitive. Among the tested values,
$\lambda_{\mathrm{tr}}=1$ gives the best overall trade-off, achieving the lowest $E_A$, essentially the best
$E_{\mathrm{traj}}$, and the shortest average support length $\overline{S}$. This supports our default
implementation choice while keeping the formal PAC-Bayes guarantee restricted to $\lambda_{\mathrm{PB}}\in(0,1]$.

\subsection{Prior Quality and Sensitivity to the Number of Training Tasks}
\label{subsec:prior_quality_sensitivity}

\begin{figure}[h]
\begin{center}
\includegraphics[scale=0.8]{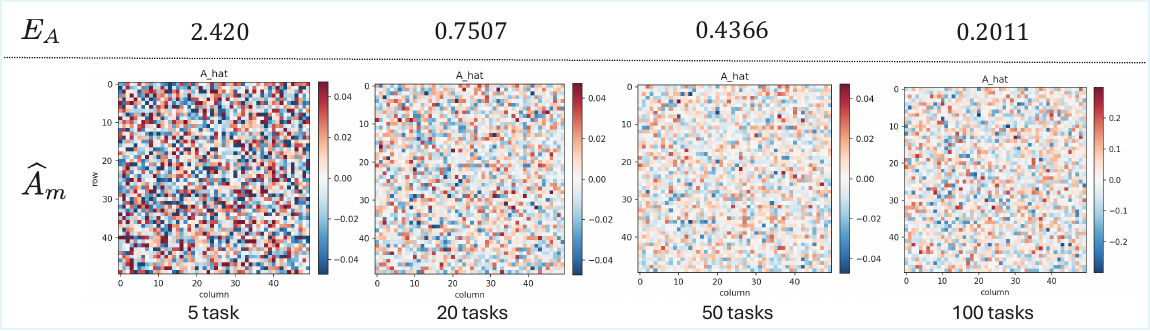}
\end{center}
\caption{Sensitivity of the learned prior to the number of meta-training tasks. The figure visualizes the estimated
transition matrices and the corresponding estimation errors for varying numbers of available training tasks. All tasks are
sampled from stable systems with spectral radius $\rho_0 = 0.95$ and dimension $d=50$.}
\label{fig:task_sensitivity}
\end{figure}

In PBML-LTI, task-level uncertainty is quantified by the posterior column covariance $V_m$ in
\eqref{eq:posterior_cov}. The fidelity of this uncertainty estimate therefore depends on how accurately the meta-training
stage recovers the shared parameters $(W,V,\sigma^2)$. At the same time, the PAC-Bayes analysis in
Section~\ref{theoretical_analysis} should be interpreted as conditional on a learned prior: it clarifies how prior
quality affects predictive adaptation through the empirical-fit/KL tradeoff, but it does not provide a second-stage bound
on how the quality of the learned prior improves as the number of meta-training tasks grows. We therefore examine that
question empirically in Figure~\ref{fig:task_sensitivity}.

As shown in Figure~\ref{fig:task_sensitivity}, when only five training tasks are available, the resulting estimates
$\widehat{A}_m$ are visibly more erratic, suggesting that the learned prior does not yet provide sufficiently reliable
shrinkage or uncertainty information. As the number of training tasks increases, the learned prior covariance $V$
becomes a more stable estimate of cross-task variability, and the induced posteriors produce transition estimates that
are progressively closer to the ground-truth dynamics. Quantitatively, increasing the number of training tasks from five
to one hundred reduces the transition-matrix estimation error by approximately $90\%$. This sensitivity study therefore
supports the interpretation that richer meta-level data improves prior quality and, in turn, improves downstream
few-shot adaptation.

\subsection{Empirical Magnitude of the PAC-Bayes-Derived Bounds}

To complement the theoretical analysis in Section~\ref{theoretical_analysis}, we report empirical magnitudes of the
PAC-Bayes-derived matrix and trajectory bounds on synthetic benchmarks. In addition to the stable and strongly unstable
settings, we include a marginally unstable regime, in which the spectral-radius parameter is set to $\rho_0=1+\epsilon$
with $\epsilon\in(0,1)$, to examine whether the derived bounds vary continuously as the dynamics cross the nominal
stability threshold $\rho=1$. This setting places the task dynamics just beyond the stability boundary and emulates the
near-critical operating conditions---slowly drifting or near-integrator modes---that are common in practical LTI
applications. The diagnostic is useful because the theory depends on finite-horizon growth and posterior predictive fit,
rather than on a discontinuous distinction between $\rho<1$ and $\rho>1$.

Table~\ref{tab:reply_pacbayes_extensions} shows a clear ordering across the three regimes. For every dimension and both
evaluation subsets, the derived bounds increase from stable to marginally unstable to strongly unstable. In the stable
setting, both $A_{\mathrm{bound}}$ and $\mathrm{Traj}_{\mathrm{bound}}$ are moderate in magnitude and nearly unchanged
between common-case and edge-case subsets. They also decrease with dimension, consistent with the fact that
lower-dimensional stable systems are easier to identify from the same support budget.

The marginally unstable regime provides a continuity check around the stability boundary. Its bounds are only mildly
larger than those in the stable regime: the matrix bound increases by less than $10\%$, while the trajectory bound
increases by roughly $20$--$25\%$ across dimensions. The qualitative behavior remains the same as in the stable regime:
the bounds remain moderate, decrease as the dimension decreases, and are similar between common-case and edge-case
subsets. Thus, crossing the nominal threshold $\rho=1$ does not by itself make the PAC-Bayes-derived quantities
uninformative. Instead, the degradation is gradual and governed by the finite-horizon growth of the learned dynamics.
This is precisely the continuity predicted by the support-to-query analysis of Section~\ref{theoretical_analysis}: the
alignment lower bound of Remark~\ref{rem:unstable_failure_mode} scales as $\rho^{2K_m}$ and the rollout factor
$H_{m,K_m}$ is continuous in the spectral radius of the learned dynamics, so for $\rho_0=1+\epsilon$, $\epsilon\in(0,1)$ at the
experimental horizons both remain controlled, exploding only under strong instability.\footnote{Additional marginal unstable system experiment results are included in Appendix~\ref{app:marginal_results}.}

\begin{table}[h]

\caption{Empirical values of the PAC-Bayes-derived matrix and trajectory bounds for PBML-LTI across the three synthetic regimes: stable ($\rho_0=0.95$), marginally unstable ($\rho_0=1+\epsilon$, $\epsilon\in(0,1)$), and strongly unstable ($\rho_0=4.95$).}
\label{tab:reply_pacbayes_extensions}
\begin{center}
\begin{tabular}{ccccc}
\toprule
Regime & Dimension & Setting & $A_{\mathrm{bound}}$ & $\mathrm{Traj}_{\mathrm{bound}}$ \\
\midrule
\multirow{6}{*}{Stable}
& \multirow{2}{*}{$d=50$} & Common-case & $0.410 \pm 0.059$ & $17.744 \pm 2.660$ \\
&  & Edge-case   & $0.418 \pm 0.063$ & $18.133 \pm 2.781$ \\
& \multirow{2}{*}{$d=25$} & Common-case & $0.172 \pm 0.028$ & $9.473 \pm 1.547$ \\
&  & Edge-case   & $0.178 \pm 0.019$ & $9.757 \pm 1.051$ \\
& \multirow{2}{*}{$d=10$} & Common-case & $0.066 \pm 0.008$ & $4.550 \pm 0.622$ \\
&  & Edge-case   & $0.066 \pm 0.011$ & $4.553 \pm 0.797$ \\
\midrule
\multirow{6}{*}{Marginally unstable}
& \multirow{2}{*}{$d=50$} & Common-case & $0.447 \pm 0.073$ & $21.128 \pm 3.450$ \\
&  & Edge-case   & $0.442 \pm 0.072$ & $20.906 \pm 3.446$ \\
& \multirow{2}{*}{$d=25$} & Common-case & $0.184 \pm 0.027$ & $11.719 \pm 1.671$ \\
&  & Edge-case   & $0.187 \pm 0.039$ & $11.880 \pm 2.504$ \\
& \multirow{2}{*}{$d=10$} & Common-case & $0.068 \pm 0.013$ & $5.674 \pm 1.141$ \\
&  & Edge-case   & $0.068 \pm 0.011$ & $5.704 \pm 0.975$ \\
\midrule
\multirow{6}{*}{Unstable}
& \multirow{2}{*}{$d=50$} & Common-case & $1.077 \pm 0.143$ & $528.807 \pm 70.972$ \\
&  & Edge-case   & $1.100 \pm 0.124$ & $539.841 \pm 60.204$ \\
& \multirow{2}{*}{$d=25$} & Common-case & $0.748 \pm 0.175$ & $1278.494 \pm 302.805$ \\
&  & Edge-case   & $0.770 \pm 0.135$ & $1306.887 \pm 227.728$ \\
& \multirow{2}{*}{$d=10$} & Common-case & $1.277 \pm 0.143$ & $3717.382 \pm 618.144$ \\
&  & Edge-case   & $1.371 \pm 0.344$ & $3832.271 \pm 628.807$ \\
\bottomrule
\end{tabular}
\end{center}
\end{table}

In contrast, the strongly unstable regime produces much looser bounds, especially for trajectory rollout. This reinforces
the interpretation of the unstable experiments as finite-horizon stress tests. The PAC-Bayes predictive-risk statement may
remain formally valid under the stated MGF condition, but the rollout consequence can become loose or effectively vacuous
because it contains the finite-horizon growth factor
\[
H_{m,K_m}
=
\sum_{r=0}^{K_m-1}\|\widehat A_m\|_2^r .
\]
When the learned dynamics are strongly unstable, the powers $\|\widehat A_m\|_2^r$ can grow rapidly, so even a finite
predictive bound can be substantially amplified under rollout. Consistent with this mechanism, the trajectory bound is
largest in the strongly unstable low-dimensional setting, where the learned dynamics exhibit the strongest finite-horizon
amplification. By contrast, the corresponding matrix bounds remain much closer in magnitude across dimensions, indicating
that the main source of looseness is rollout amplification rather than the one-step matrix-error consequence alone.

Overall, the table supports two conclusions. First, the PAC-Bayes-derived bounds behave continuously around the stability
threshold: the marginally unstable regime is a mild degradation of the stable regime rather than a qualitative breakdown.
Second, strong instability can make the trajectory bound loose because finite-horizon spectral amplification magnifies the
predictive-risk bound. Across all three regimes, the differences between common-case and edge-case subsets are small
relative to the differences between stability regimes, suggesting that the degree of instability is the dominant factor
controlling the practical tightness of the derived bounds in this experiment.

\section{Conclusion}

We studied few-shot identification of linear time-invariant dynamical systems in a meta-learning setting, where each task
provides only a short, temporally dependent trajectory and reliable adaptation to previously unseen systems is required.
Our proposed PBML-LTI framework learns a transferable prior over task-specific dynamics from multiple related systems and
performs task-level adaptation through closed-form Bayesian inference, yielding both accurate point estimates and
principled uncertainty quantification while remaining computationally efficient.

On the theoretical side, we developed a martingale PAC-Bayes analysis tailored to trajectory data with within-task
temporal dependence. The resulting support--query statement should be interpreted as a support-conditioned predictive-risk
diagnostic: it clarifies the role of the learned prior and posterior complexity, while still depending on the empirical
query fit on the held-out suffix. We further characterized when this empirical query term is expected to be controlled,
through support-to-query transfer conditions involving excitation alignment, posterior concentration, learned-prior
quality, and finite-horizon stability. Thus, the theory provides a PAC-Bayes motivation for the fit--KL surrogate and a
support--query predictive-risk perspective for few-shot adaptation, rather than a direct end-to-end guarantee for the full
experimental pipeline. We also showed how the support--query predictive bound induces problem-specific corollaries for
projected transition-matrix error, full Frobenius recovery under query excitation, and finite-horizon rollout error.

Empirically, PBML-LTI consistently improves data efficiency and robustness in challenging regimes, including
high-dimensional synthetic systems and real fMRI-derived LTI tasks, where linear dynamics only hold approximately. In
particular, the proposed method achieves strong multi-step prediction performance while adapting from limited support
data, highlighting its suitability for few-shot system identification.

Several directions remain for future work. Extending the framework to controlled systems, partial observability, or
nonlinear dynamics would broaden its applicability. Another important direction is to develop an explicit second-stage
meta-generalization analysis that quantifies how the quality of the learned prior improves with the number of
meta-training tasks. Incorporating richer structured priors is also a natural avenue for further research.

\paragraph{Code and Data Availability}
The code repository is available at \url{https://github.com/chenfeng-huang/PBML-LTI-TMLR-2026}. The synthetic dataset generation code is included in the code repository. The fMRI dataset is available at \url{https://openneuro.org/datasets/ds000244/versions/1.0.0}, provided by \cite{ds000244}.

\subsubsection*{Acknowledgments}
The work of GM was partially supported by NSF grants ATD 2319552 and DMS 2348640.

\bibliography{main}
\bibliographystyle{tmlr}

\clearpage
\appendix

\startcontents[appendix]

\printappendixtoc

\section{Technical Appendix}

\subsection{Proof of Theorem~\ref{thm:inst_nll}}
\label{sec:proof1}
\begin{proof}
Under the Gaussian conditional model
\begin{equation}
p_\sigma(x_t\mid x_{t-1},A)=\mathcal{N}(Ax_{t-1},\sigma^2 I_d),
\label{eq:app_gaussian_model}
\end{equation}
the conditional density is given by
\begin{equation}
p_\sigma(x_t\mid x_{t-1},A)
=(2\pi\sigma^2)^{-d/2}\exp\!\left(-\frac{1}{2\sigma^2}\|x_t-Ax_{t-1}\|_2^2\right).
\label{eq:app_gaussian_density}
\end{equation}
Taking the negative of the logarithm yields
\begin{equation}
\ell_t(A)
=\frac{1}{2\sigma^2}\|x_t-Ax_{t-1}\|_2^2+\frac{d}{2}\log(2\pi\sigma^2),
\label{eq:app_inst_nll}
\end{equation}
which proves \eqref{eq:inst_nll}.

Next, fix any $A$ and define
\begin{equation}
d_t(A):=\E[\ell_t(A)\mid \mathcal{F}_{t-1}]-\ell_t(A).
\label{eq:app_dt_def}
\end{equation}
By assumption, $\ell_t(A)$ is $\mathcal{F}_t$-measurable and integrable. Therefore,
\begin{equation}
\E[d_t(A)\mid \mathcal{F}_{t-1}]
=0,
\label{eq:app_dt_mds}
\end{equation}
which shows that $\{d_t(A)\}_{t=1}^T$ is a martingale difference sequence.
\end{proof}
 
\subsection{Proof of Theorem~\ref{thm:exp_supermg_psi}}
\label{sec:proof2}
\begin{proof}
Fix $A$ and $\lambda_{\mathrm{PB}}\in(0,1]$. Define
\begin{equation}
S_t(A):=\sum_{i=1}^t d_i(A), \qquad S_0(A)=0,
\label{eq:app_St_def}
\end{equation}
and
\begin{equation}
Z_t(A;\lambda_{\mathrm{PB}})
:=
\exp\!\left(
\lambda_{\mathrm{PB}} S_t(A)
-
\sum_{i=1}^t \psi_i(\lambda_{\mathrm{PB}})
\right),
\qquad
Z_0(A;\lambda_{\mathrm{PB}})=1.
\label{eq:app_Zt_def}
\end{equation}
Since $S_t(A)=S_{t-1}(A)+d_t(A)$,
\begin{equation}
Z_t(A;\lambda_{\mathrm{PB}})
=
Z_{t-1}(A;\lambda_{\mathrm{PB}})
\exp\!\left(
\lambda_{\mathrm{PB}}d_t(A)-\psi_t(\lambda_{\mathrm{PB}})
\right).
\label{eq:app_Zt_recursion}
\end{equation}
Taking conditional expectation and using the MGF condition gives
\begin{equation}
\E[Z_t(A;\lambda_{\mathrm{PB}})\mid \mathcal{F}_{t-1}]
\le
Z_{t-1}(A;\lambda_{\mathrm{PB}}),
\label{eq:app_supermg}
\end{equation}
so $\{Z_t(A;\lambda_{\mathrm{PB}})\}$ is a nonnegative supermartingale.
\end{proof}

\begin{lemma}[Donsker--Varadhan change of measure]
\label{lem:dv}
Let $P$ and $Q$ be probability measures on the same measurable space. If $Q\ll P$, then for any measurable function $f$
for which the expressions are well defined,
\begin{equation}
\E_{A\sim Q}[f(A)]
\le
\KL(Q\|P)
+
\log \E_{A\sim P}\!\left[e^{f(A)}\right].
\label{eq:dv}
\end{equation}
If $Q\not\ll P$, then $\KL(Q\|P)=+\infty$ and the inequality holds trivially.
\end{lemma}

\begin{proof}
Assume $Q\ll P$ and let $r:=dQ/dP$. Then
\[
\KL(Q\|P)=\E_Q[\log r].
\]
By Jensen's inequality,
\begin{align}
\E_Q[f(A)]-\KL(Q\|P)
&=
\E_Q\!\left[f(A)-\log r(A)\right] \nonumber\\
&=
\E_Q\!\left[
\log\!\left(\frac{e^{f(A)}}{r(A)}\right)
\right] \nonumber\\
&\le
\log \E_Q\!\left[
\frac{e^{f(A)}}{r(A)}
\right]
=
\log \int e^{f(a)}\,dP(a).
\end{align}
Rearranging gives the result.
\end{proof}

\subsection{Proof of Theorem~\ref{thm:mpb_with_psi}}
\label{sec:proof3}
\begin{proof}
Fix $\lambda_{\mathrm{PB}}\in(0,1]$. Define the mixture supermartingale
\begin{equation}
\overline{Z}_t(\lambda_{\mathrm{PB}}) := \E_{A\sim P}[Z_t(A;\lambda_{\mathrm{PB}})].
\label{eq:app_Zbar_def}
\end{equation}
Then
\begin{equation}
\E[\overline{Z}_t(\lambda)\mid\mathcal{F}_{t-1}]
\le \overline{Z}_{t-1}(\lambda),
\label{eq:app_Zbar_supermg}
\end{equation}
and hence
\begin{equation}
\E[\overline{Z}_T(\lambda)]\le 1.
\label{eq:app_Zbar_expectation}
\end{equation}

By Markov's inequality, with probability at least $1-\delta$,
\begin{equation}
\overline{Z}_T(\lambda)\le \frac{1}{\delta}.
\label{eq:app_event_mixture}
\end{equation}
Applying Lemma~\ref{lem:dv} yields
\begin{equation}
\E_{A\sim Q}[S_T(A)]
\le \frac{1}{\lambda}\left(\KL(Q\|P)+\log\frac{1}{\delta}+\sum_{t=1}^T\psi_t(\lambda)\right).
\label{eq:app_ST_bound}
\end{equation}
Substituting
\begin{equation}
S_T(A)=\sum_{t=1}^T\Big(\E[\ell_t(A)\mid\mathcal{F}_{t-1}]-\ell_t(A)\Big)
\label{eq:app_ST_expand}
\end{equation}
and rearranging gives \eqref{eq:mpb_with_psi}.
\end{proof}

\subsection{Proof of Corollary~\ref{cor:subg_excess}}
\label{sec:proof4}
\begin{proof}
Under Assumption~\ref{as4},
\begin{equation}
\sum_{t=1}^T \psi_t(\lambda_{\mathrm{PB}}) \le \frac{\lambda_{\mathrm{PB}}^2 v T}{2}.
\label{eq:app_sum_psi}
\end{equation}
Substituting this into Theorem~\ref{thm:mpb_with_psi} yields
\begin{equation}
L_T(Q)
\le
\widehat{L}_T(Q)
+
\frac{\KL(Q\|P)+\log(1/\delta)}{\lambda_{\mathrm{PB}} T}
+
\frac{\lambda_{\mathrm{PB}} v}{2}.
\label{eq:app_pb_bound}
\end{equation}
This proves \eqref{eq:pb_uniform_v_short}.

Now let
\[
a:=\KL(Q\|P)+\log(1/\delta),
\qquad
g(\lambda_{\mathrm{PB}})
:=
\frac{a}{\lambda_{\mathrm{PB}} T}
+
\frac{\lambda_{\mathrm{PB}} v}{2},
\qquad \lambda_{\mathrm{PB}}\in(0,1].
\]
Differentiating with respect to $\lambda_{\mathrm{PB}}$ gives
\[
g'(\lambda_{\mathrm{PB}})
=
-\frac{a}{\lambda_{\mathrm{PB}}^2 T}
+
\frac{v}{2},
\qquad
g''(\lambda_{\mathrm{PB}})
=
\frac{2a}{\lambda_{\mathrm{PB}}^3 T}
>0.
\]
Thus, $g$ is strictly convex on $(0,\infty)$. Its unconstrained minimizer satisfies
\[
g'(\lambda_{\mathrm{PB}})=0
\quad\Longrightarrow\quad
\lambda_{\mathrm{unc}}=\sqrt{\frac{2a}{vT}}.
\]
Therefore, the minimizer over the constrained interval $(0,1]$ is
\begin{equation}
\lambda^\star
=
\min\!\left\{
1,\;
\sqrt{\frac{2a}{vT}}
\right\}.
\label{eq:app_lambda_star}
\end{equation}

In the regime where $\lambda^\star<1$, we have
\[
\lambda^\star=\sqrt{\frac{2a}{vT}},
\]
and substituting this into \eqref{eq:app_pb_bound} gives
\begin{equation}
L_T(Q)
\le
\widehat{L}_T(Q)
+
\sqrt{\frac{2va}{T}},
\label{eq:app_pb_sqrt}
\end{equation}
which is exactly \eqref{eq:pb_rate_sqrtT_short}.
\end{proof}

\subsection{Proof of Corollary~\ref{cor:support_query_pb}}
\label{sec:proof_support_query_pb}
\begin{proof}
Fix a task $m$ and write
\[
\mathcal G_u := \mathcal F_{m,S_m+u}, \qquad u=0,\dots,K_m.
\]
Condition on the support sigma-field $\mathcal G_0=\mathcal F_{m,S_m}$. Under this conditioning, the posterior
\[
Q_{m,\phi}^{\mathrm{sup}} = p(A_m\mid D_m^{\mathrm{sup}},\phi)
\]
is fixed, and the prior $P_\phi$ is independent of the future query suffix. We apply
Corollary~\ref{cor:subg_excess} to the shifted query sequence.

Define, for $u=1,\dots,K_m$,
\begin{equation}
\widetilde \ell_u(A)
:=
-\log p_\sigma(x_{m,S_m+u}\mid x_{m,S_m+u-1},A).
\label{eq:shifted_query_logloss}
\end{equation}
By Theorem~\ref{thm:inst_nll},
\begin{equation}
\widetilde \ell_u(A)
=
\frac{1}{2\sigma^2}
\|x_{m,S_m+u}-Ax_{m,S_m+u-1}\|_2^2
+
\frac{d}{2}\log(2\pi\sigma^2).
\label{eq:shifted_query_logloss_expanded}
\end{equation}
Since Assumption~\ref{as4} is inherited by the shifted filtration $\{\mathcal G_u\}_{u=0}^{K_m}$,
Corollary~\ref{cor:subg_excess} gives, with conditional probability at least $1-\delta$ given $\mathcal G_0$,
\begin{equation}
\frac{1}{K_m}\sum_{u=1}^{K_m}
\E_{A\sim Q_{m,\phi}^{\mathrm{sup}}}
\!\left[
\E\!\left[\widetilde \ell_u(A)\mid \mathcal G_{u-1}\right]
\right]
\le
\frac{1}{K_m}\sum_{u=1}^{K_m}
\E_{A\sim Q_{m,\phi}^{\mathrm{sup}}}\!\left[\widetilde \ell_u(A)\right]
+
\frac{\KL(Q_{m,\phi}^{\mathrm{sup}}\|P_\phi)+\log(1/\delta)}{\lambda_{\mathrm{PB}} K_m}
+
\frac{\lambda_{\mathrm{PB}} v}{2}.
\label{eq:shifted_query_pb_logloss}
\end{equation}

Substituting \eqref{eq:shifted_query_logloss_expanded} into \eqref{eq:shifted_query_pb_logloss}, the left-hand side becomes
\begin{align}
&\frac{1}{K_m}\sum_{u=1}^{K_m}
\E_{A\sim Q_{m,\phi}^{\mathrm{sup}}}
\!\left[
\E\!\left[\widetilde \ell_u(A)\mid \mathcal G_{u-1}\right]
\right] \nonumber\\
&=
\frac{1}{2\sigma^2}
\frac{1}{K_m}\sum_{u=1}^{K_m}
\E_{A\sim Q_{m,\phi}^{\mathrm{sup}}}
\!\left[
\E\!\left[
\|x_{m,S_m+u}-Ax_{m,S_m+u-1}\|_2^2
\,\middle|\,
\mathcal G_{u-1}
\right]
\right]
+
\frac{d}{2}\log(2\pi\sigma^2) \nonumber\\
&=
\frac{1}{2\sigma^2}\,
R_{m,q}^{\mathrm{pred}}(Q_{m,\phi}^{\mathrm{sup}})
+
\frac{d}{2}\log(2\pi\sigma^2),
\label{eq:lhs_shifted_query_pb}
\end{align}
by definition \eqref{eq:support_query_pred_risk}. Similarly, the right-hand side becomes
\begin{align}
&\frac{1}{K_m}\sum_{u=1}^{K_m}
\E_{A\sim Q_{m,\phi}^{\mathrm{sup}}}\!\left[\widetilde \ell_u(A)\right] \nonumber\\
&=
\frac{1}{2\sigma^2}
\frac{1}{K_m}\sum_{u=1}^{K_m}
\E_{A\sim Q_{m,\phi}^{\mathrm{sup}}}
\!\left[
\|x_{m,S_m+u}-Ax_{m,S_m+u-1}\|_2^2
\right]
+
\frac{d}{2}\log(2\pi\sigma^2) \nonumber\\
&=
\frac{1}{2\sigma^2}
\widehat R_{m,q}^{\mathrm{pred}}(Q_{m,\phi}^{\mathrm{sup}})
+
\frac{d}{2}\log(2\pi\sigma^2),
\label{eq:rhs_shifted_query_pb}
\end{align}
since
\[
\sum_{u=1}^{K_m}
\|x_{m,S_m+u}-Ax_{m,S_m+u-1}\|_2^2
=
\|Y_{m,q}-AX_{m,q}\|_F^2.
\]

Cancelling the common additive term $\frac{d}{2}\log(2\pi\sigma^2)$ from both sides and multiplying by $2\sigma^2$ yields
\[
R_{m,q}^{\mathrm{pred}}(Q_{m,\phi}^{\mathrm{sup}})
\le
\widehat R_{m,q}^{\mathrm{pred}}(Q_{m,\phi}^{\mathrm{sup}})
+
\frac{2\sigma^2}{\lambda_{\mathrm{PB}} K_m}
\Big(
\KL(Q_{m,\phi}^{\mathrm{sup}}\|P_\phi)+\log(1/\delta)
\Big)
+
\sigma^2\lambda_{\mathrm{PB}} v
\]
which is exactly \eqref{eq:support_query_pb}.
\end{proof}

\subsection{Support-to-Query Transfer Proof Details}
\label{sec:proof_support_query_transfer}

\begin{proof}[Proof of Proposition~\ref{prop:support_query_transfer}]
By the closed form \eqref{eq:posterior_sq_loss_closed_form} applied to the query pair,
\[
\widehat R_{m,q}^{\mathrm{pred}}(Q_{m,\phi}^{\mathrm{sup}})
=
\frac{1}{K_m}\|Y_{m,q}-M_mX_{m,q}\|_F^2
+
d\,\mathrm{tr}(V_m\widehat\Sigma_q).
\]
Using $Y_{m,q}=A_mX_{m,q}+\Xi_{m,q}$ and
$\|U+V\|_F^2\le 2\|U\|_F^2+2\|V\|_F^2$, we obtain
\[
\frac{1}{K_m}\|Y_{m,q}-M_mX_{m,q}\|_F^2
\le
2\,\mathrm{tr}\big((A_m-M_m)\widehat\Sigma_q(A_m-M_m)^\top\big)
+
\frac{2}{K_m}\|\Xi_{m,q}\|_F^2 .
\]
If $\widehat\Sigma_q\preceq \gamma_m\widehat\Sigma_{\mathrm{sup}}$, then for any matrix $B$,
\[
\mathrm{tr}(B\widehat\Sigma_qB^\top)
\le
\gamma_m\,\mathrm{tr}(B\widehat\Sigma_{\mathrm{sup}}B^\top),
\]
and also
\[
\mathrm{tr}(V_m\widehat\Sigma_q)
\le
\gamma_m\,\mathrm{tr}(V_m\widehat\Sigma_{\mathrm{sup}}).
\]
Taking $B=A_m-M_m$ gives
\[
\mathrm{tr}\big((A_m-M_m)\widehat\Sigma_q(A_m-M_m)^\top\big)
\le
\gamma_m
\frac{1}{S_m}\|(A_m-M_m)X_m^{\mathrm{sup}}\|_F^2 .
\]
Since
\[
Y_m^{\mathrm{sup}}=A_mX_m^{\mathrm{sup}}+\Xi_m^{\mathrm{sup}},
\]
we have
\[
(A_m-M_m)X_m^{\mathrm{sup}}
=
(Y_m^{\mathrm{sup}}-M_mX_m^{\mathrm{sup}})-\Xi_m^{\mathrm{sup}},
\]
and hence
\[
\frac{1}{S_m}\|(A_m-M_m)X_m^{\mathrm{sup}}\|_F^2
\le
\frac{2}{S_m}\|Y_m^{\mathrm{sup}}-M_mX_m^{\mathrm{sup}}\|_F^2
+
\frac{2}{S_m}\|\Xi_m^{\mathrm{sup}}\|_F^2 .
\]
Combining the preceding inequalities yields
\[
\widehat R_{m,q}^{\mathrm{pred}}(Q_{m,\phi}^{\mathrm{sup}})
\le
\frac{4\gamma_m}{S_m}\|Y_m^{\mathrm{sup}}-M_mX_m^{\mathrm{sup}}\|_F^2
+
\gamma_m d\,\mathrm{tr}(V_m\widehat\Sigma_{\mathrm{sup}})
+
\frac{4\gamma_m}{S_m}\|\Xi_m^{\mathrm{sup}}\|_F^2
+
\frac{2}{K_m}\|\Xi_{m,q}\|_F^2 .
\]
Since
\[
\widehat R_{\mathrm{sup}}^{\mathrm{pred}}
=
\frac{1}{S_m}\|Y_m^{\mathrm{sup}}-M_mX_m^{\mathrm{sup}}\|_F^2
+
d\,\mathrm{tr}(V_m\widehat\Sigma_{\mathrm{sup}}),
\]
we can upper bound the first two terms by
$4\gamma_m\widehat R_{\mathrm{sup}}^{\mathrm{pred}}$, proving
\eqref{eq:support_query_transfer}.

For the posterior-width cap, recall from \eqref{eq:posterior_cov} that
\[
V_m^{-1}
=
V^{-1}
+
\frac{S_m}{\sigma^2}\widehat\Sigma_{\mathrm{sup}}.
\]
Thus $V_m^{-1}\succeq V^{-1}$, which implies $V_m\preceq V$, and
$V_m^{-1}\succeq (S_m/\sigma^2)\widehat\Sigma_{\mathrm{sup}}$. Therefore
\[
\mathrm{tr}(V_m\widehat\Sigma_{\mathrm{sup}})
\le
\mathrm{tr}(V\widehat\Sigma_{\mathrm{sup}})
\]
and
\[
\mathrm{tr}(V_m\widehat\Sigma_{\mathrm{sup}})
\le
\frac{\sigma^2}{S_m}\mathrm{tr}(V_mV_m^{-1})
=
\frac{\sigma^2 d}{S_m}.
\]
Combining this with
$\mathrm{tr}(V_m\widehat\Sigma_q)\le\gamma_m\mathrm{tr}(V_m\widehat\Sigma_{\mathrm{sup}})$ proves
\eqref{eq:posterior_width_transfer_cap}.
\end{proof}

\begin{proof}[Proof of Proposition~\ref{prop:regularized_support_query_transfer}]
Since $\widehat\Sigma_{\mathrm{sup}}+\alpha I_d\succ0$, the definition of $\gamma_{m,\alpha}$ is equivalent to
\[
\widehat\Sigma_q
\preceq
\gamma_{m,\alpha}(\widehat\Sigma_{\mathrm{sup}}+\alpha I_d).
\]
Repeating the proof of Proposition~\ref{prop:support_query_transfer} with
$\widehat\Sigma_{\mathrm{sup}}+\alpha I_d$ in place of $\widehat\Sigma_{\mathrm{sup}}$ gives
\[
\mathrm{tr}\big((A_m-M_m)\widehat\Sigma_q(A_m-M_m)^\top\big)
\le
\gamma_{m,\alpha}
\mathrm{tr}\big((A_m-M_m)\widehat\Sigma_{\mathrm{sup}}(A_m-M_m)^\top\big)
+
\gamma_{m,\alpha}\alpha\|A_m-M_m\|_F^2
\]
and
\[
\mathrm{tr}(V_m\widehat\Sigma_q)
\le
\gamma_{m,\alpha}\mathrm{tr}(V_m\widehat\Sigma_{\mathrm{sup}})
+
\gamma_{m,\alpha}\alpha\mathrm{tr}(V_m).
\]
Combining these inequalities with the same support residual identity as above yields
\[
\widehat R_{m,q}^{\mathrm{pred}}(Q_{m,\phi}^{\mathrm{sup}})
\le
4\gamma_{m,\alpha}\widehat R_{\mathrm{sup}}^{\mathrm{pred}}
+
\frac{4\gamma_{m,\alpha}}{S_m}\|\Xi_m^{\mathrm{sup}}\|_F^2
+
\frac{2}{K_m}\|\Xi_{m,q}\|_F^2
+
2\gamma_{m,\alpha}\alpha\|M_m-A_m\|_F^2
+
\gamma_{m,\alpha}\alpha d\,\mathrm{tr}(V_m).
\]
Since
\[
\E_{A\sim Q_{m,\phi}^{\mathrm{sup}}}\|A-A_m\|_F^2
=
\|M_m-A_m\|_F^2+d\,\mathrm{tr}(V_m),
\]
the final two terms are upper bounded by
\[
2\gamma_{m,\alpha}\alpha
\E_{A\sim Q_{m,\phi}^{\mathrm{sup}}}\|A-A_m\|_F^2,
\]
which proves \eqref{eq:regularized_support_query_transfer}.
\end{proof}

\begin{remark}[Properties of the regularized alignment constant and control of the remainder term]
\label{rem:regularized_remainder_control}
(i) On the event $\{\gamma_m<\infty\}$,
$\widehat\Sigma_q\preceq\gamma_m\widehat\Sigma_{\mathrm{sup}}\preceq\gamma_m(\widehat\Sigma_{\mathrm{sup}}+\alpha I_d)$
gives $\gamma_{m,\alpha}\le\gamma_m$, and
$\widehat\Sigma_q\preceq\lambda_{\max}(\widehat\Sigma_q)I_d\preceq(\lambda_{\max}(\widehat\Sigma_q)/\alpha)(\widehat\Sigma_{\mathrm{sup}}+\alpha I_d)$
gives $\gamma_{m,\alpha}\le\lambda_{\max}(\widehat\Sigma_q)/\alpha$; moreover $\alpha\mapsto\gamma_{m,\alpha}$ is
nonincreasing, since $\alpha\mapsto\widehat\Sigma_{\mathrm{sup}}+\alpha I_d$ is nondecreasing in the semidefinite order.
With the prior-matched choice that replaces $\alpha I_d$ by $(\sigma^2/S_m)V^{-1}$, the regularized Gram equals
$(\sigma^2/S_m)V_m^{-1}$ by \eqref{eq:posterior_cov}, the corresponding alignment diagnostic is
$\widetilde\gamma_m=(S_m/\sigma^2)\lambda_{\max}(V_m^{1/2}\widehat\Sigma_qV_m^{1/2})$, and the posterior-width
contribution obeys the unconditional cap
$d\,\mathrm{tr}(V_m\widehat\Sigma_q)\le\widetilde\gamma_m\,\sigma^2d^2/S_m$, since
$\mathrm{tr}(V_m^{1/2}\widehat\Sigma_qV_m^{1/2})\le d\,\lambda_{\max}(V_m^{1/2}\widehat\Sigma_qV_m^{1/2})$.
(ii) The remainder term in \eqref{eq:regularized_support_query_transfer} decomposes as
$\E_{A\sim Q_{m,\phi}^{\mathrm{sup}}}\|A-A_m\|_F^2=\|M_m-A_m\|_F^2+d\,\mathrm{tr}(V_m)$, an observable posterior-width
part plus a mean-error part. From \eqref{eq:posterior_cov}--\eqref{eq:posterior_mean} and
$Y_m^{\mathrm{sup}}=A_mX_m^{\mathrm{sup}}+\Xi_m^{\mathrm{sup}}$,
\[
M_m-A_m
=
\left[\frac{1}{\sigma^2}\Xi_m^{\mathrm{sup}}(X_m^{\mathrm{sup}})^\top+(W-A_m)V^{-1}\right]V_m,
\]
so that
\[
\|M_m-A_m\|_F^2
\le
\frac{2}{\sigma^4}\big\|\Xi_m^{\mathrm{sup}}(X_m^{\mathrm{sup}})^\top V_m\big\|_F^2
+
2\,\lambda_{\max}(V_m)\,\Delta_m(\phi),
\]
using $V^{-1}V_m^2V^{-1}\preceq\lambda_{\max}(V_m)V^{-1}$ (a consequence of $V_m\preceq V$) for the second term, with
$\Delta_m(\phi)$ the learned-prior quality of \eqref{eq:prior_quality_basic_inequality}. The first term is a
self-normalized martingale quantity controlled by standard self-normalized concentration for dependent regressors
\citep{abbasi2011online}. Thus, in weakly excited directions the remainder is controlled by posterior contraction and
learned-prior quality, as claimed.
\end{remark}

\begin{proof}[Proof sketch of Lemma~\ref{lem:noise_energy_concentration}]
Let
\[
Z_t
:=
\|w_{m,t}\|_2^2
-
\E[\|w_{m,t}\|_2^2\mid \mathcal F_{m,t-1}].
\]
Then $\{Z_t\}$ is a martingale difference sequence. The conditional sub-Gaussian assumption in
Assumption~\ref{as2} implies that $\|w_{m,t}\|_2^2$ is conditionally sub-exponential, with scale controlled by
$\sigma_w^2 d$. Moreover,
\[
\E[\|w_{m,t}\|_2^2\mid \mathcal F_{m,t-1}]
=
\mathrm{tr}\!\left(
\E[w_{m,t}w_{m,t}^\top\mid \mathcal F_{m,t-1}]
\right)
\le
\mathrm{tr}(\Sigma_w).
\]
Applying a Bernstein--Freedman inequality for conditionally sub-exponential martingale differences
\citep{freedman1975tail} gives the stated bound.
\end{proof}

\begin{proof}[Proof of Corollary~\ref{cor:conditional_support_query_certificate}]
Corollary~\ref{cor:support_query_pb} holds with conditional probability at least $1-\delta$ given
$\mathcal F_{m,S_m}$, and hence unconditionally by the tower property. Lemma~\ref{lem:noise_energy_concentration} is
applied once to the support noise energy and once (conditionally on $\mathcal F_{m,S_m}$, hence again unconditionally) to
the query noise energy. On the intersection of these three events, Proposition~\ref{prop:support_query_transfer} bounds
the empirical query term in Corollary~\ref{cor:support_query_pb}. Substituting the two noise-energy concentration bounds
into \eqref{eq:support_query_transfer} yields \eqref{eq:conditional_support_query_certificate}. A union bound gives total
failure probability at most $3\delta$.
\end{proof}

\paragraph{Corollary chain for the matrix and rollout consequences.}
Let $\mathcal C_m^{\mathrm{pred}}(\bar\gamma,\delta)$ denote the right-hand side of
\eqref{eq:conditional_support_query_certificate}. The consequences \eqref{eq:reply_proj_A_bound},
\eqref{eq:reply_full_A_bound}, and \eqref{eq:reply_traj_bound} use the support--query predictive bound only through the
upper bound $\mathcal B_m^{\mathrm{pred}}$ on $R_{m,q}^{\mathrm{pred}}(Q_{m,\phi}^{\mathrm{sup}})$, so on the events of
Corollary~\ref{cor:conditional_support_query_certificate} the same chaining applies with
$\mathcal C_m^{\mathrm{pred}}(\bar\gamma,\delta)$ in place of $\mathcal B_m^{\mathrm{pred}}$. Concretely, on the event
$\widehat\Sigma_q\preceq\bar\gamma\,\widehat\Sigma_{\mathrm{sup}}$, with probability at least $1-3\delta$,
\begin{gather*}
E_{A,\mathrm{proj}}(m)\le \mathcal C_m^{\mathrm{pred}}(\bar\gamma,\delta),
\qquad
E_A(m)\le \frac{\mathcal C_m^{\mathrm{pred}}(\bar\gamma,\delta)}{\kappa_m}
\ \text{under \eqref{eq:reply_excitation_cond}},
\\
\E\!\left[E_{\mathrm{traj}}(m)\mid\mathcal F_{m,S_m}\right]\le K_m^2\,H_{m,K_m}^2\,\mathcal C_m^{\mathrm{pred}}(\bar\gamma,\delta).
\end{gather*}
Thus the projected transition-matrix, full Frobenius, and rollout consequences inherit the conditional support-to-query
form: apart from the alignment level and noise constants, their right-hand sides are determined by the support-adapted
posterior and the support prefix.

\subsection{Functional Magnetic Resonance Imaging (fMRI) Dataset}
\label{sec:fmri_dataset}

We construct a real-data LTI benchmark from the OpenNeuro dataset ds000244 (Individual Brain Charting)
\citep{ds000244}. The dataset is publicly available at
\url{https://openneuro.org/datasets/ds000244/versions/1.0.0}. The raw data follow the Brain Imaging Data Structure
(BIDS) convention, together with metadata and tabular confound files.

\paragraph{Schaefer parcellation.}
To obtain a multivariate state sequence from each fMRI run, we convert voxel-level Blood-Oxygen-Level-Dependent (BOLD)
signals into a $d$-dimensional ROI time series by averaging voxel signals within atlas-defined parcels using the
Schaefer parcellation \citep{schaefer2018local}. Let $y_v(t)$ denote the BOLD signal at voxel $v$ and time index $t$.
For parcel (ROI) $j$ with voxel set $\Omega_j$, we define the parcel-averaged signal
\begin{equation}
x_t[j]
~:=~
\frac{1}{|\Omega_j|}
\sum_{v\in\Omega_j} y_v(t),
\qquad j=1,\dots,d,
\label{eq:roi_average}
\end{equation}
yielding a multivariate state vector $x_t\in\R^d$ at each time $t$. In our experiments we use a Schaefer atlas with
$d=200$ parcels, so each fMRI run becomes a length-$T$ trajectory $\{x_t\}_{t=0}^{T}$ in $\R^{200}$.

\paragraph{Run-level preprocessing.}
For each run-level ROI time series, we apply standard time-series preprocessing to improve comparability across runs and
to reduce nuisance variation. Concretely, we (i) standardize each ROI time series over time, (ii) regress out motion
confounds when a confounds TSV is available, and (iii) mean-center the resulting ROI signals within each run. If we write
the ROI matrix as $Z\in\R^{T\times d}$ with rows $x_t^\top$, and let $C\in\R^{T\times q}$ denote a subspace spanned by
the confound regressors, we apply
\begin{equation}
Z \;\leftarrow\; \big(I - C(C^\top C)^{\dagger}C^\top\big)Z,
\label{eq:confound_regression}
\end{equation}
followed by per-run mean-centering $Z \leftarrow Z - \mathbf{1}\mu^\top$, where
$\mu=\frac{1}{T}\sum_{t=1}^T x_t$.

\paragraph{Window-level LTI task construction.}
To obtain many related identification tasks from each run, we segment each ROI time series into overlapping temporal
windows of fixed length. Each window is treated as a distinct task instance $m$ with states
$\{x_{m,t}\}_{t=0}^{T_m}$, and we form the one-step regression matrices
\begin{equation}
X_m := [x_{m,0},\dots,x_{m,T_m-1}] \in \R^{d\times T_m},
\qquad
Y_m := [x_{m,1},\dots,x_{m,T_m}] \in \R^{d\times T_m}.
\label{eq:fmri_XY}
\end{equation}
These matrices define the window-level matrix regression view $Y_m \approx A_m X_m$ used throughout the paper. In our
released split, each window has $d=200$ and $T_m=100$ transitions.

The resulting windows are \emph{related} task instances because they are not drawn from arbitrary unrelated sources.
Rather, they come from the same pool of subjects, sessions, and experimental conditions, and overlapping windows from the
same run can be interpreted as nearby local dynamical regimes within a shared subject- and condition-specific
environment. This gives rise to a heterogeneous but structured family of tasks, which is well matched to the
meta-learning objective of PBML-LTI.

\paragraph{Reference transition matrix for evaluation.}
Unlike synthetic data, fMRI does not provide a physical ground-truth transition matrix. We therefore define a reference
transition matrix $A_m^{\mathrm{ref}}$ for each window by fitting a ridge-regularized one-step linear map:
\begin{equation}
A_m^{\mathrm{ref}}
~:=~
\arg\min_{A\in\R^{d\times d}}
\ \|Y_m - A X_m\|_F^2 \;+\; \lambda_{\mathrm{ref}}\|A\|_F^2,
\end{equation}
which has the closed-form solution
\begin{equation}
A_m^{\mathrm{ref}}
~=~
Y_m X_m^\top\Big(X_m X_m^\top + \lambda_{\mathrm{ref}} I_d\Big)^{-1}.
\end{equation}
We emphasize that $A_m^{\mathrm{ref}}$ is used only as a consistent evaluation reference to compute $E_A$; all methods
are trained and adapted from trajectories $(X_m,Y_m)$.

\paragraph{Train--test split and task labels.}
We use a predefined train--test split constructed at the window level. Windows from different subjects, sessions, and
experimental conditions are distributed across training and test according to this split. The original BIDS task labels
are retained only as metadata for grouping and reporting, and are not used by the learning algorithms.

BIDS stands for \emph{Brain Imaging Data Structure}, a standard format for organizing neuroimaging datasets. In this
format, each fMRI run carries a \texttt{task} label identifying the experimental paradigm under which the data were
collected. These labels therefore describe the source cognitive or behavioral condition of a run, rather than providing a
supervisory target for learning. In our benchmark, they are used only to characterize the diversity of the data and to
report how many experimental conditions are represented in the train/test split.

In the split used in our experiments, the training set contains 141 window-level task instances spanning 17 task labels,
while the test set contains 14 window-level instances spanning 8 task labels. Across the full split, the task labels
include:
\begin{itemize}
\item \textbf{Archi:} ArchiEmotional, ArchiSocial, ArchiSpatial, ArchiStandard.
\item \textbf{HCP-style:} HcpEmotion, HcpGambling, HcpLanguage, HcpMotor,
HcpRelational, HcpSocial, HcpWm.
\item \textbf{RSVP language:} RSVPLanguage00, RSVPLanguage01, RSVPLanguage02,
RSVPLanguage03, RSVPLanguage04, RSVPLanguage05.
\end{itemize}

\section{Additional Results}

\subsection{Support--Query Diagnostics}
\label{app:support_query_diagnostics}

The support--query PAC-Bayes bound in Corollary~\ref{cor:support_query_pb} contains the empirical query predictive term
$\widehat R_{m,q}^{\mathrm{pred}}$. This term is computed on the held-out query suffix and should not be interpreted as a
support-only quantity. To assess when this term is expected to be small, we compare predictive errors on the support,
inner validation, and held-out query portions of the trajectory.

All diagnostics are computed after meta-training and use the learned prior $P_\phi$ fixed. For each task, the held-out
query suffix is never used to choose the adaptation prefix; it is used only for post-selection evaluation. Candidate
prefixes are selected using only the fit-prefix posterior and the validation suffix inside the support window.
Unless otherwise stated, reported group values are averages over tasks in the corresponding regime and dimension.

For the support--query diagnostics in Tables~\ref{tab:val_query_corr} and~\ref{tab:pred_terms}, we use the following
notation. For a candidate prefix length $s$, let $Q_{m,\phi}^{(s)}$ denote the posterior formed by adapting the learned
prior to the first $s$ transitions of task $m$, with parameters $(M_m(s),V_m(s))$ given by
\eqref{eq:posterior_cov}--\eqref{eq:posterior_mean}. Let
$(X_{m,\mathrm{sup}}^{(s)},Y_{m,\mathrm{sup}}^{(s)})$,
$(X_{m,\mathrm{val}},Y_{m,\mathrm{val}})$, and
$(X_{m,\mathrm{qry}},Y_{m,\mathrm{qry}})$ denote the regression pairs of the fit prefix, the validation suffix inside
the support window, and the held-out query suffix of length $K_m$. We define the posterior-predictive squared errors
\[
\widehat R_{m,\mathrm{sup}}^{\mathrm{pred}}(s)
:=
\frac{1}{s}
\E_{A\sim Q_{m,\phi}^{(s)}}\!\left[
\|Y_{m,\mathrm{sup}}^{(s)}-AX_{m,\mathrm{sup}}^{(s)}\|_F^2
\right],
\]
\[
\widehat R_{m,\mathrm{val}}^{\mathrm{pred}}(s)
:=
\frac{1}{T_{\mathrm{val}}}
\E_{A\sim Q_{m,\phi}^{(s)}}\!\left[
\|Y_{m,\mathrm{val}}-AX_{m,\mathrm{val}}\|_F^2
\right],
\qquad
\widehat R_{m,\mathrm{qry}}^{\mathrm{pred}}(s)
:=
\frac{1}{K_m}
\E_{A\sim Q_{m,\phi}^{(s)}}\!\left[
\|Y_{m,\mathrm{qry}}-AX_{m,\mathrm{qry}}\|_F^2
\right].
\]
Thus, $\widehat R_{m,\mathrm{qry}}^{\mathrm{pred}}(s)$ instantiates the empirical predictive risk
\eqref{eq:support_query_emp_pred_risk} at the posterior $Q_{m,\phi}^{(s)}$. By
\eqref{eq:posterior_sq_loss_closed_form}, each expectation evaluates in closed form as
\[
\|Y^{\mathrm{seg}}-M_m(s)X^{\mathrm{seg}}\|_F^2
+
d\,\mathrm{tr}\!\left(
V_m(s)\,X^{\mathrm{seg}}(X^{\mathrm{seg}})^\top
\right),
\]
normalized by the segment length.
Let $\mathcal S_m$ denote the grid of candidate fit-prefix lengths considered for task $m$. The validation-selected
prefix is
\[
s_m^{\mathrm{fit}}
\in
\arg\min_{s\in\mathcal S_m}
\widehat R_{m,\mathrm{val}}^{\mathrm{pred}}(s).
\]
The final reported query metric is then evaluated at $s=s_m^{\mathrm{fit}}$, after this selection has been made. In Tables~\ref{tab:val_query_corr} and~\ref{tab:pred_terms}, the task index $m$ is suppressed, and
Table~\ref{tab:pred_terms} reports these quantities at $s=s_m^{\mathrm{fit}}$, abbreviated
$\widehat R_{\mathrm{sup}}^{\mathrm{pred}}$, $\widehat R_{\mathrm{val}}^{\mathrm{pred}}$, and
$\widehat R_{q}^{\mathrm{pred}}$.

To connect these diagnostics with the transfer analysis in Section~\ref{theoretical_analysis}, we also define the
regularized support-window alignment proxy
\[
\gamma_{m,\alpha}^{\mathrm{val}}(s)
:=
\lambda_{\max}\!\left[
(\widehat\Sigma_{m,\mathrm{fit}}^{(s)}+\alpha I_d)^{-1/2}
\widehat\Sigma_{m,\mathrm{val}}
(\widehat\Sigma_{m,\mathrm{fit}}^{(s)}+\alpha I_d)^{-1/2}
\right],
\]
where
\[
\widehat\Sigma_{m,\mathrm{fit}}^{(s)}
=
\frac{1}{s}
X_{m,\mathrm{sup}}^{(s)}(X_{m,\mathrm{sup}}^{(s)})^\top,
\qquad
\widehat\Sigma_{m,\mathrm{val}}
=
\frac{1}{T_{\mathrm{val}}}
X_{m,\mathrm{val}}X_{m,\mathrm{val}}^\top .
\]
This quantity is the support-window analogue of the regularized support-to-query alignment constant
\eqref{eq:regularized_gamma_alignment}, with the validation Gram replacing the held-out query Gram.

In the diagnostics, $\alpha$ is chosen to match the prior-scaled regularization used in the posterior covariance,
namely $\alpha=\sigma^2/s$ when using the isotropic proxy, or equivalently the prior-matched matrix regularization
$(\sigma^2/s)V^{-1}$ when reporting the posterior-scaled alignment quantity.

Table~\ref{tab:val_query_corr} reports the correlation between inner validation predictive error and held-out query
predictive error across candidate support prefixes and tasks. The correlations are strongly positive overall (Spearman
$\rho=0.902$, $n=720$) and remain positive within every regime and dimension, supporting the use of validation error as an empirical proxy for held-out query behavior.
The quantity $\gamma_{m,\alpha}^{\mathrm{val}}(s)$ provides a complementary support-window analogue of the theoretical
alignment constant, although it is not itself reported in Table~\ref{tab:val_query_corr}, see\ Remark~\ref{rem:necessity}. Table~\ref{tab:pred_terms} reports support, validation, and query posterior-predictive squared errors at the
validation-selected prefix.\footnote{The ratio $\widehat R_q^{\mathrm{pred}}/\widehat R_{\mathrm{sup}}^{\mathrm{pred}}$
is invariant to the overall normalization of the empirical predictive risk \eqref{eq:support_query_emp_pred_risk}, so
the comparison applies whether the per-step risk is reported per coordinate or in aggregate.} In stable systems, the
query term is comparable to the support term. In the unstable $d=10$
setting, however, the query term becomes much larger, confirming that the support--query bound can become loose when
unstable spectral amplification produces large held-out residuals.

These ratios are quantitatively consistent with the transfer analysis of Section~\ref{theoretical_analysis}.
Propositions~\ref{prop:support_query_transfer} and~\ref{prop:regularized_support_query_transfer} imply
$\widehat R_q^{\mathrm{pred}}/\widehat R_{\mathrm{sup}}^{\mathrm{pred}}\lesssim 4\gamma_{m,\alpha}$ up to noise-floor
and prior-quality terms. The stable-regime ratios in Table~\ref{tab:pred_terms} ($\approx 1.69$--$2.73$ across
$d\in\{10,25,50\}$) are consistent with mild alignment, $\gamma_{m,\alpha}=O(1)$, as the common-stationary-regime
sufficient condition of Section~\ref{theoretical_analysis} (which gives $\gamma_m\le3$) implies for the stable
generator ($\rho_0=0.95$) at the reported horizons. The unstable $d=10$ ratio ($\approx 1.07\times10^{4}$) matches the
geometric degradation of Remark~\ref{rem:unstable_failure_mode}, in which the alignment constants grow on the order of
$\rho^{2K_m}$ along unstable directions. The milder unstable $d=25$ and $d=50$ ratios ($0.861$ and $0.095$) suggest that the generated trajectories in these
settings experience less effective query-horizon amplification than the unstable $d=10$ setting. This is consistent with
the longer validation-selected prefixes in those settings, which improve support excitation before evaluation
(Table~\ref{tab:sweep_A_mse_unstable}).

\begin{table}[h]
\centering
\caption{Correlation between inner validation predictive error $\widehat R_{\mathrm{val}}^{\mathrm{pred}}(s)$ and held-out query predictive error $\widehat R_q^{\mathrm{pred}}(s)$ across candidate prefixes and tasks (PBML-LTI).}
\label{tab:val_query_corr}
\begin{tabular}{lccccc}
\toprule
Group & $n$ & Spearman $\rho$ & Spearman $p$ & Pearson $r$ & Pearson $p$ \\
\midrule
all & 720 & 0.9020 & 1.6300e-264 & 0.8040 & 2.1300e-164 \\
stable & 360 & 0.8350 & 4.2400e-95 & 0.7090 & 2.7700e-56 \\
stable\_d10 & 120 & 0.6210 & 3.9900e-14 & 0.6700 & 5.5200e-17 \\
stable\_d25 & 120 & 0.6660 & 1.0500e-16 & 0.5910 & 1.2600e-12 \\
stable\_d50 & 120 & 0.7310 & 2.6300e-21 & 0.5260 & 6.5700e-10 \\
unstable & 360 & 0.7050 & 2.0200e-55 & 0.8020 & 3.0000e-82 \\
unstable\_d10 & 120 & 0.9320 & 5.5900e-54 & 0.7830 & 4.0900e-26 \\
unstable\_d25 & 120 & 0.8080 & 8.0400e-29 & 0.6790 & 1.4600e-17 \\
unstable\_d50 & 120 & 0.5250 & 7.3100e-10 & 0.4270 & 1.1800e-06 \\
\bottomrule
\end{tabular}
\end{table}

\begin{table}[h]
    \centering
    \caption{Support, validation, and query posterior-predictive squared errors at the validation-selected prefix
    (PBML-LTI). The query term is comparable to the support term in the stable regime but can be much larger under unstable
    dynamics.}
    \label{tab:pred_terms}
   
    \begin{tabular}{lcccc}
    \toprule
    Group
    & $\widehat R_{\mathrm{sup}}^{\mathrm{pred}}$
    & $\widehat R_{\mathrm{val}}^{\mathrm{pred}}$
    & $\widehat R_q^{\mathrm{pred}}$
    & $\widehat R_q^{\mathrm{pred}} / \widehat R_{\mathrm{sup}}^{\mathrm{pred}}$ \\
    \midrule
    stable        & 0.001  & 0.004  & 0.003    & 2.331 \\
    stable\_d10   & 0.001  & 0.001  & 0.001    & 1.687 \\
    stable\_d25   & 0.001  & 0.003  & 0.003    & 1.955 \\
    stable\_d50   & 0.002  & 0.009  & 0.006    & 2.726 \\
    unstable      & 0.1764 & 0.9076 & 384.237  & 2178.214 \\
    unstable\_d10 & 0.1073 & 2.6094 & 1152.347 & $1.074{\times}10^{4}$ \\
    unstable\_d25 & 0.164  & 0.085  & 0.141    & 0.861 \\
    unstable\_d50 & 0.257  & 0.028  & 0.025    & 0.095 \\
    \bottomrule
    \end{tabular}
    
    \end{table}
\FloatBarrier

\subsection{Fixed-Prefix Support Sweeps}
\label{app:support_sweep_diagnostics}

To clarify the role of adaptive support selection, we also evaluate each method at fixed support lengths. These sweeps
remove the prefix-selection step and show how performance changes as the number of support transitions increases.

Tables~\ref{tab:sweep_A_mse_stable} and~\ref{tab:sweep_A_mse_unstable} report transition-matrix error across fixed
support lengths in the stable and unstable regimes, respectively. Tables~\ref{tab:sweep_traj_mse_stable}
and~\ref{tab:sweep_traj_mse_unstable} report the corresponding rollout errors. In stable systems, PBML-LTI reaches low
transition-matrix error with very short prefixes, supporting the interpretation that the learned prior provides genuine
few-shot data efficiency. In unstable systems, PBML-LTI often benefits from longer prefixes, showing that the method does
not mechanically prefer the shortest support length. Instead, the validation-selected prefix adapts to the amount of
calibration data needed for the task.

\begin{table}[h]
    
    \caption{Fixed-prefix support sweep ($E_A$, stable regime). ``Adaptive'' is the metric at the
    validation-selected effective support length $\bar S$. Lower is better for $E_A$; $\bar S$ reports the average selected
    support length. Boldface marks the lowest $E_A$ values within each dimension.}
    \label{tab:sweep_A_mse_stable}
    \begin{center}
    \resizebox{\textwidth}{!}{%
    \begin{tabular}{llcccccccc}
    \toprule
    Dimension & Method & $s=1$ & $s=2$ & $s=3$ & $s=5$ & $s=10$ & $s=15$ & Adaptive & $\bar S$ \\
    \midrule
    \multirow{6}{*}{$d=10$}
    & PBML-LTI          & $\mathbf{0.0052}$ & $\mathbf{0.0053}$ & $\mathbf{0.0057}$ & $\mathbf{0.0064}$ & $\mathbf{0.0073}$ & $\mathbf{0.0076}$ & $\mathbf{0.0053}$ & 1.05 \\
    & OLS               & 0.5827 & 0.5473 & 0.5216 & 0.4968 & 0.4794 & 0.5232 & 0.4449 & 2.65 \\
    & Ridge             & 0.5683 & 0.5396 & 0.5128 & 0.4937 & 0.5014 & 0.5083 & 0.4447 & 2.65 \\
    & Pooled-prior Ridge & 0.0224 & 0.0217 & 0.0203 & 0.0216 & 0.0218 & 0.0215 & 0.0887 & 1.80 \\
    & SharedSubspace    & 0.0456 & 0.0384 & 0.0378 & 0.0367 & 0.0359 & 0.0354 & 0.1088 & 7.95 \\
    & MAML-LTI          & 0.0316 & 0.0314 & 0.0327 & 0.0323 & 0.0336 & 0.0338 & 0.0357 & 1.55 \\
    \midrule
    \multirow{6}{*}{$d=25$}
    & PBML-LTI          & $\mathbf{0.0446}$ & $\mathbf{0.0438}$ & $\mathbf{0.0427}$ & $\mathbf{0.0403}$ & $\mathbf{0.0506}$ & $\mathbf{0.0554}$ & $\mathbf{0.0394}$ & 1.20 \\
    & OLS               & 0.8037 & 0.7764 & 0.7628 & 0.7516 & 0.8173 & 0.8467 & 0.7140 & 2.00 \\
    & Ridge             & 0.7765 & 0.7618 & 0.7496 & 0.7432 & 0.8024 & 0.8193 & 0.7140 & 2.00 \\
    & Pooled-prior Ridge & 0.9476 & 0.9263 & 0.9127 & 0.8874 & 0.9038 & 0.9196 & 0.8544 & 2.00 \\
    & SharedSubspace    & 1.0027 & 0.9776 & 0.9584 & 0.9467 & 0.9413 & 0.9526 & 0.7731 & 13.95 \\
    & MAML-LTI          & 1.7628 & 1.7596 & 1.7683 & 1.7837 & 1.8046 & 1.8123 & 1.7566 & 1.85 \\
    \midrule
    \multirow{6}{*}{$d=50$}
    & PBML-LTI          & $\mathbf{0.1776}$ & $\mathbf{0.1748}$ & $\mathbf{0.1714}$ & $\mathbf{0.1678}$ & $\mathbf{0.1697}$ & $\mathbf{0.1956}$ & $\mathbf{0.2068}$ & 1.35 \\
    & OLS               & 1.0038 & 0.9776 & 0.9583 & 0.9416 & 0.9327 & 0.9524 & 0.8411 & 2.00 \\
    & Ridge             & 1.0476 & 1.0197 & 0.9926 & 0.9684 & 0.9573 & 0.9818 & 0.8410 & 2.00 \\
    & Pooled-prior Ridge & 0.7238 & 0.7086 & 0.6974 & 0.6887 & 0.7264 & 0.7526 & 0.6743 & 2.55 \\
    & SharedSubspace    & 0.8774 & 0.8667 & 0.8576 & 0.8593 & 0.8974 & 0.9197 & 0.8572 & 10.25 \\
    & MAML-LTI          & 15.3726 & 15.2278 & 15.2084 & 15.2936 & 15.4263 & 15.4872 & 14.8930 & 2.90 \\
    \bottomrule
    \end{tabular}
    }
    \end{center}
    \end{table}

    \begin{table}[h]
    \caption{Fixed-prefix support sweep ($E_A$, unstable regime). PBML-LTI performs better at \emph{large} $s$;
    Pooled-prior Ridge/SharedSubspace recover $A$ better than PBML-LTI in $d=10$ (see Table~\ref{tab:sweep_traj_mse_unstable}).
    ``Adaptive'' is the metric at the validation-selected effective support length $\bar S$. Lower is better for $E_A$;
    $\bar S$ reports the average selected support length. Boldface marks the lowest $E_A$ values within each dimension.}
    \label{tab:sweep_A_mse_unstable}
    \begin{center}
    \resizebox{\textwidth}{!}{%
    \begin{tabular}{llcccccccc}
    \toprule
    Dimension & Method & $s=1$ & $s=2$ & $s=3$ & $s=5$ & $s=10$ & $s=15$ & Adaptive & $\bar S$ \\
    \midrule
    \multirow{6}{*}{$d=10$}
    & PBML-LTI          & 0.2236 & 0.2074 & 0.1968 & 0.1873 & 0.1847 & 0.1796 & 0.1390 & 14.35 \\
    & OLS               & 1.1027 & 1.0476 & 0.9964 & 0.9773 & 0.9726 & 0.9487 & 0.8450 & 13.35 \\
    & Ridge             & 0.8476 & 0.8194 & 0.8027 & 0.7768 & 0.7873 & 0.8026 & 0.6970 & 14.30 \\
    & Pooled-prior Ridge & $\mathbf{0.0186}$ & $\mathbf{0.0174}$ & 0.0178 & $\mathbf{0.0167}$ & $\mathbf{0.0143}$ & $\mathbf{0.0146}$ & $\mathbf{0.0070}$ & 9.25 \\
    & SharedSubspace    & 0.0187 & 0.0176 & $\mathbf{0.0174}$ & 0.0168 & 0.0167 & 0.0154 & 0.0130 & 3.20 \\
    & MAML-LTI          & 1.9027 & 1.8476 & 1.9538 & 2.0014 & 2.2967 & 2.1036 & 0.1473 & 8.35 \\
    \midrule
    \multirow{6}{*}{$d=25$}
    & PBML-LTI          & $\mathbf{0.0406}$ & $\mathbf{0.0387}$ & $\mathbf{0.0374}$ & $\mathbf{0.0368}$ & $\mathbf{0.0367}$ & $\mathbf{0.0356}$ & $\mathbf{0.0290}$ & 6.05 \\
    & OLS               & 11.4726 & 11.1967 & 10.9874 & 10.8126 & 10.8738 & 10.9846 & 10.3440 & 14.10 \\
    & Ridge             & 10.9763 & 10.6738 & 10.4876 & 10.3124 & 10.3767 & 10.4923 & 9.9070 & 14.90 \\
    & Pooled-prior Ridge & 1.6476 & 1.6037 & 1.5796 & 1.6238 & 1.6794 & 1.7026 & 1.5870 & 14.75 \\
    & SharedSubspace    & 37.8467 & 36.9726 & 36.4873 & 35.9768 & 35.4926 & 35.9874 & 35.2680 & 14.35 \\
    & MAML-LTI          & 2.4876 & 2.4738 & 2.4817 & 2.5086 & 2.5423 & 2.5567 & 2.5229 & 2.05 \\
    \midrule
    \multirow{6}{*}{$d=50$}
    & PBML-LTI          & $\mathbf{0.1726}$ & $\mathbf{0.1647}$ & $\mathbf{0.1618}$ & $\mathbf{0.1596}$ & $\mathbf{0.1743}$ & $\mathbf{0.1797}$ & $\mathbf{0.1560}$ & 4.15 \\
    & OLS               & 17.4768 & 17.1876 & 16.9847 & 16.7926 & 17.4637 & 17.9824 & 16.9230 & 9.75 \\
    & Ridge             & 16.4726 & 16.1874 & 15.9846 & 15.7938 & 15.6727 & 15.9876 & 15.9020 & 15.00 \\
    & Pooled-prior Ridge & 8.5764 & 8.4676 & 8.3927 & 8.2416 & 8.0574 & 7.9648 & 7.5380 & 15.00 \\
    & SharedSubspace    & 7.0286 & 6.5274 & 6.4378 & 6.4036 & 6.3978 & 6.3994 & 6.4190 & 14.20 \\
    & MAML-LTI          & 7.2637 & 7.1846 & 7.1814 & 7.2276 & 7.3028 & 7.3356 & 8.9787 & 2.55 \\
    \bottomrule
    \end{tabular}
    }
    \end{center}
    \end{table}

    \begin{table}[h]
    \caption{Fixed-prefix support sweep ($E_{\mathrm{traj}}$, stable regime). ``Adaptive'' is the metric at the
    validation-selected effective support length $\bar S$. Lower is better for $E_{\mathrm{traj}}$; $\bar S$ reports the
    average selected support length. Boldface marks the lowest $E_{\mathrm{traj}}$ values within each dimension.}
    \label{tab:sweep_traj_mse_stable}
    \begin{center}
    \resizebox{\textwidth}{!}{%
    \begin{tabular}{llcccccccc}
    \toprule
    Dimension & Method & $s=1$ & $s=2$ & $s=3$ & $s=5$ & $s=10$ & $s=15$ & Adaptive & $\bar S$ \\
    \midrule
    \multirow{6}{*}{$d=10$}
    & PBML-LTI          & 0.0126 & 0.0117 & 0.0114 & 0.0116 & 0.0118 & 0.0123 & $\mathbf{0.0055}$ & 1.05 \\
    & OLS               & 0.0087 & 0.0076 & $\mathbf{0.0064}$ & 0.0067 & 0.0073 & 0.0076 & 0.0056 & 2.65 \\
    & Ridge             & $\mathbf{0.0076}$ & 0.0074 & 0.0067 & 0.0068 & $\mathbf{0.0063}$ & 0.0076 & 0.0056 & 2.65 \\
    & Pooled-prior Ridge & $\mathbf{0.0076}$ & $\mathbf{0.0064}$ & 0.0067 & $\mathbf{0.0066}$ & 0.0068 & $\mathbf{0.0067}$ & 0.0057 & 1.80 \\
    & SharedSubspace    & 0.0116 & 0.0117 & 0.0114 & 0.0116 & 0.0118 & 0.0113 & 0.0056 & 7.95 \\
    & MAML-LTI          & 0.0126 & 0.0123 & 0.0127 & 0.0124 & 0.0126 & 0.0127 & 0.0112 & 1.55 \\
    \midrule
    \multirow{6}{*}{$d=25$}
    & PBML-LTI          & 0.0286 & 0.0287 & 0.0276 & 0.0256 & 0.0257 & 0.0256 & $\mathbf{0.0121}$ & 1.20 \\
    & OLS               & 0.0186 & 0.0176 & 0.0167 & 0.0157 & 0.0167 & 0.0166 & 0.0133 & 2.00 \\
    & Ridge             & 0.0176 & 0.0167 & 0.0166 & 0.0157 & 0.0156 & 0.0155 & 0.0133 & 2.00 \\
    & Pooled-prior Ridge & $\mathbf{0.0167}$ & $\mathbf{0.0166}$ & $\mathbf{0.0156}$ & $\mathbf{0.0154}$ & $\mathbf{0.0146}$ & $\mathbf{0.0156}$ & 0.0134 & 2.00 \\
    & SharedSubspace    & 0.0206 & 0.0196 & 0.0197 & 0.0186 & 0.0187 & 0.0186 & 0.0134 & 13.95 \\
    & MAML-LTI          & 0.0346 & 0.0347 & 0.0336 & 0.0306 & 0.0307 & 0.0306 & 0.0338 & 1.85 \\
    \midrule
    \multirow{6}{*}{$d=50$}
    & PBML-LTI          & 0.0616 & 0.0607 & 0.0596 & 0.0556 & 0.0567 & 0.0546 & $\mathbf{0.0257}$ & 1.35 \\
    & OLS               & 0.0306 & 0.0296 & 0.0297 & 0.0286 & 0.0287 & 0.0286 & 0.0258 & 2.00 \\
    & Ridge             & 0.0296 & 0.0297 & 0.0286 & 0.0287 & 0.0276 & 0.0286 & 0.0258 & 2.00 \\
    & Pooled-prior Ridge & 0.0296 & 0.0326 & 0.0356 & 0.0406 & 0.0306 & 0.0286 & 0.0258 & 2.55 \\
    & SharedSubspace    & $\mathbf{0.0286}$ & $\mathbf{0.0287}$ & $\mathbf{0.0276}$ & $\mathbf{0.0276}$ & $\mathbf{0.0266}$ & $\mathbf{0.0276}$ & 0.0258 & 10.25 \\
    & MAML-LTI          & 0.0806 & 0.0856 & 0.0786 & 0.0766 & 0.0746 & 0.0756 & 0.1081 & 2.90 \\
    \bottomrule
    \end{tabular}
    }
    \end{center}
    \end{table}

    \begin{table}[h]
    
    \caption{Fixed-prefix support sweep ($E_{\mathrm{traj}}$, unstable regime). MAML-LTI's rollout has elevated error for
    large $s$ at $d=10$. ``Adaptive'' is the metric at the validation-selected effective support length $\bar S$. Lower is
    better for $E_{\mathrm{traj}}$; $\bar S$ reports the average selected support length. Boldface marks the lowest
    $E_{\mathrm{traj}}$ values within each dimension.}
    \label{tab:sweep_traj_mse_unstable}
    \begin{center}
    \resizebox{\textwidth}{!}{%
    \begin{tabular}{llcccccccc}
    \toprule
    Dimension & Method & $s=1$ & $s=2$ & $s=3$ & $s=5$ & $s=10$ & $s=15$ & Adaptive & $\bar S$ \\
    \midrule
    \multirow{6}{*}{$d=10$}
    & PBML-LTI          & $\mathbf{65.1476}$ & $\mathbf{119.8347}$ & $\mathbf{89.7626}$ & $\mathbf{69.9136}$ & $\mathbf{54.8767}$ & $\mathbf{47.9138}$ & 57.9190 & 14.35 \\
    & OLS               & $2.2037{\times}10^{5}$ & $2.0046{\times}10^{5}$ & $1.8038{\times}10^{5}$ & $1.7047{\times}10^{5}$ & $1.6536{\times}10^{5}$ & $1.6087{\times}10^{5}$ & 193.7040 & 13.35 \\
    & Ridge             & $7.0036{\times}10^{4}$ & $6.5047{\times}10^{4}$ & $6.0078{\times}10^{4}$ & $5.8036{\times}10^{4}$ & $5.6074{\times}10^{4}$ & $5.5068{\times}10^{4}$ & $\mathbf{57.6280}$ & 14.30 \\
    & Pooled-prior Ridge & $2.5036{\times}10^{4}$ & $2.8047{\times}10^{4}$ & $2.6038{\times}10^{4}$ & $2.2067{\times}10^{4}$ & $1.9036{\times}10^{4}$ & $1.8078{\times}10^{4}$ & $2.1935{\times}10^{4}$ & 9.25 \\
    & SharedSubspace    & $8.0036{\times}10^{5}$ & $7.5047{\times}10^{5}$ & $7.2038{\times}10^{5}$ & $7.0067{\times}10^{5}$ & $6.8036{\times}10^{5}$ & $6.5078{\times}10^{5}$ & $7.2189{\times}10^{5}$ & 3.20 \\
    & MAML-LTI          & $2.5036{\times}10^{7}$ & $2.0047{\times}10^{7}$ & $2.8038{\times}10^{7}$ & $3.0067{\times}10^{7}$ & $3.5036{\times}10^{7}$ & $2.2078{\times}10^{7}$ & $9.7390{\times}10^{6}$ & 8.35 \\
    \midrule
    \multirow{6}{*}{$d=25$}
    & PBML-LTI          & 0.1236 & 0.1027 & 0.0876 & 0.0796 & 0.0767 & 0.0716 & 0.0800 & 6.05 \\
    & OLS               & 0.1476 & 0.1287 & 0.1196 & 0.1086 & 0.0976 & 0.0967 & 0.0880 & 14.10 \\
    & Ridge             & 0.1196 & 0.1087 & 0.0976 & 0.0876 & 0.0796 & 0.0716 & 0.0640 & 14.90 \\
    & Pooled-prior Ridge & $\mathbf{0.0616}$ & $\mathbf{0.0576}$ & $\mathbf{0.0567}$ & $\mathbf{0.0556}$ & $\mathbf{0.0567}$ & $\mathbf{0.0556}$ & $\mathbf{0.0530}$ & 14.75 \\
    & SharedSubspace    & 29.8766 & 27.9137 & 26.9846 & 25.9767 & 25.4876 & 24.9826 & 27.1510 & 14.35 \\
    & MAML-LTI          & 0.2196 & 0.1976 & 0.1796 & 0.1686 & 0.1576 & 0.1496 & 1.1349 & 2.05 \\
    \midrule
    \multirow{6}{*}{$d=50$}
    & PBML-LTI          & $\mathbf{0.0376}$ & $\mathbf{0.0377}$ & $\mathbf{0.0376}$ & $\mathbf{0.0376}$ & $\mathbf{0.0376}$ & $\mathbf{0.0376}$ & $\mathbf{0.0370}$ & 4.15 \\
    & OLS               & 0.0426 & 0.0427 & 0.0416 & 0.0416 & 0.0416 & 0.0426 & 0.0400 & 9.75 \\
    & Ridge             & 0.0426 & 0.0427 & 0.0416 & 0.0416 & 0.0416 & 0.0416 & 0.0400 & 15.00 \\
    & Pooled-prior Ridge & 0.0396 & 0.0386 & 0.0386 & 0.0386 & 0.0386 & 0.0386 & 0.0380 & 15.00 \\
    & SharedSubspace    & 0.0386 & 0.0386 & 0.0386 & 0.0386 & 0.0386 & 0.0386 & 0.0380 & 14.20 \\
    & MAML-LTI          & 0.0406 & 0.0406 & 0.0406 & 0.0406 & 0.0406 & 0.0406 & 0.0414 & 2.55 \\
    \bottomrule
    \end{tabular}
    }
    \end{center}
    \end{table}
    
\FloatBarrier

\subsection{Ridge Tuning Without the Stability-Oriented Criterion}
\label{app:rho_target_diagnostic}

The Ridge and Pooled-prior Ridge baselines use a stability-oriented tuning rule in the main experiments. In synthetic
experiments, the threshold is fixed as $\rho_{\mathrm{target}}=\rho_0$ using the known regime parameter. To verify that
the qualitative conclusions do not depend on this criterion, Table~\ref{tab:no_stability_tuning} reports a diagnostic
variant in which Ridge and Pooled-prior Ridge are tuned directly by validation rollout error, without using the
$\rho_{\mathrm{target}}$ constraint.

The results show that the Ridge-type estimates remain conservative in several settings even without the explicit
stability-oriented criterion. The fitted spectral radii also vary by regime and dimension; for example, unstable $d=10$
selects $\rho(\widehat A)\approx 1.47$, whereas unstable $d=50$ selects substantially smaller spectral radii. Thus, the
visual conservativeness of Ridge-type estimates is due to the interaction between ridge shrinkage, validation tuning, and
rollout stability, rather than only to enforcing a hard spectral-radius bound.

\begin{table}[h]
    
    \caption{Validation-rollout tuning without stability-oriented regularization. Regularization is selected directly by
    validation rollout error rather than by a spectral-radius criterion. $E_A$ denotes transition-matrix error,
    $E_{\mathrm{traj}}$ denotes rollout error, $\rho(\widehat A)$ reports the fitted spectral radius, and $\lambda$ is the
    selected ridge penalty.}
    \label{tab:no_stability_tuning}
    \setlength{\tabcolsep}{4pt}
    \renewcommand{\arraystretch}{0.7}
    \begin{center}
    \begin{tabular}{ccccccc}
    \toprule
    Regime & $d$ & Method & $E_A$ & $E_{\mathrm{traj}}$ & $\rho(\widehat A)$ & $\lambda$ \\
    \midrule
    \multirow{6}{*}{Stable}
    & \multirow{2}{*}{10}
    & Ridge             & $0.5083 \pm 0.0472$ & $0.0073 \pm 0.0014$ & 0.9383 & $10^{-3}$ \\
    & & Pooled-prior Ridge & $0.0214 \pm 0.0037$ & $0.0067 \pm 0.0016$ & 0.9496 & $10^{-1}$ \\
    & \multirow{2}{*}{25}
    & Ridge             & $0.8194 \pm 0.0826$ & $0.0153 \pm 0.0024$ & 0.9139 & $10^{-2}$ \\
    & & Pooled-prior Ridge & $0.9196 \pm 0.0227$ & $0.0156 \pm 0.0023$ & 0.9399 & $10^{-1}$ \\
    & \multirow{2}{*}{50}
    & Ridge             & $0.9814 \pm 0.0406$ & $0.0283 \pm 0.0027$ & 0.9415 & $10^{-2}$ \\
    & & Pooled-prior Ridge & $0.7526 \pm 0.0154$ & $0.0286 \pm 0.0023$ & 0.9473 & $10^{-2}$ \\
    \midrule
    \multirow{6}{*}{Unstable}
    & \multirow{2}{*}{10}
    & Ridge             & $0.8024 \pm 0.4016$ & $(5.5347 \pm 2.4683){\times}10^{4}$ & 1.4720 & $10^{-4}$ \\
    & & Pooled-prior Ridge & $0.0143 \pm 0.0036$ & $(1.8036 \pm 1.0047){\times}10^{4}$ & 1.4720 & $10^{-1}$ \\
    & \multirow{2}{*}{25}
    & Ridge             & $10.4926 \pm 0.7014$ & $0.0713 \pm 0.0304$ & 1.0420 & $10^{-2}$ \\
    & & Pooled-prior Ridge & $1.7026 \pm 0.1024$ & $0.0553 \pm 0.0206$ & 1.0470 & $10^{-1}$ \\
    & \multirow{2}{*}{50}
    & Ridge             & $15.9874 \pm 0.4026$ & $0.0413 \pm 0.0024$ & 0.6588 & $10^{-2}$ \\
    & & Pooled-prior Ridge & $7.9643 \pm 0.1216$ & $0.0384 \pm 0.0023$ & 0.6929 & $10^{-1}$ \\
    \bottomrule
    \end{tabular}
    \end{center}
    \end{table}

\FloatBarrier
\subsection{Paired Significance Tests}

\label{app:paired_significance}

\subsubsection{Synthetic Data}

\label{app:synthetic_significance}
To support the interpretation of boldface entries in the main result tables, we conduct paired two-sided Wilcoxon
signed-rank tests on matched test tasks. For each method and setting, we run the experiment with 10 random seeds under
the same data split, task-generation configuration, and hyperparameter setting. We average each task-level metric across
the 10 seeds and then perform a paired test over the 20 matched test tasks. This evaluates method differences
task-by-task under identical experimental conditions while reducing variability due to random initialization and
optimization.

Table~\ref{tab:significance_d10_baselines} reports representative tests at dimension $d=10$, where the distinction
between transition-matrix error and rollout error is especially informative. The results show that PBML-LTI's
transition-matrix error advantages in the stable setting are statistically significant against all classical baselines.
For rollout error in stable systems, the numerical differences are very small and are often not statistically
significant. In unstable systems, the tests highlight the decoupling between $E_A$ and $E_{\mathrm{traj}}$: some methods
can achieve lower Frobenius transition-matrix error while still suffering from much larger rollout error because of
spectral amplification. Thus, boldface should be read as indicating the lowest empirical mean, while the significance
tests provide the statistical qualification for those comparisons.

\begin{table}[h]
\scriptsize
\caption{Paired two-sided Wilcoxon signed-rank tests comparing PBML-LTI against the classical baselines at dimension
$d=10$, across the stable $(\rho_0=0.95)$ and unstable $(\rho_0=4.95)$ regimes and the common-case and edge-case
settings. For each method and cell, we run 10 random seeds under the same experimental configuration, average each
task-level metric across seeds, and then run a paired test over the 20 matched test tasks. We regard $p<0.05$ as
significant, $p<0.01$ as strongly significant, and $p<0.001$ as highly significant. The ``Lower error'' column reports
which method attains the smaller mean, since all metrics are lower-is-better.}
\label{tab:significance_d10_baselines}
\begin{center}
\resizebox{\linewidth}{!}{
\begin{tabular}{ccclcc}
\toprule
Regime & Setting & Metric & Comparison & $p$-value & Lower error \\
\midrule

\multirow{16}{*}{\shortstack[c]{Stable}}
& \multirow{8}{*}{Common-case}
& \multirow{4}{*}{$E_A$}
& PBML-LTI vs OLS & $1.91{\times}10^{-6}$ & PBML-LTI \\
& & & PBML-LTI vs Ridge & $1.91{\times}10^{-6}$ & PBML-LTI \\
& & & PBML-LTI vs Pooled-prior Ridge & $1.91{\times}10^{-6}$ & PBML-LTI \\
& & & PBML-LTI vs Shared Subspace & $1.91{\times}10^{-6}$ & PBML-LTI \\
\cmidrule(lr){3-6}
& & \multirow{4}{*}{$E_{\mathrm{traj}}$}
& PBML-LTI vs OLS & $0.0973$ & PBML-LTI \\
& & & PBML-LTI vs Ridge & $0.0696$ & PBML-LTI \\
& & & PBML-LTI vs Pooled-prior Ridge & $0.0897$ & PBML-LTI \\
& & & PBML-LTI vs Shared Subspace & $0.3884$ & PBML-LTI \\
\cmidrule(lr){2-6}

& \multirow{8}{*}{Edge-case}
& \multirow{4}{*}{$E_A$}
& PBML-LTI vs OLS & $1.91{\times}10^{-6}$ & PBML-LTI \\
& & & PBML-LTI vs Ridge & $1.91{\times}10^{-6}$ & PBML-LTI \\
& & & PBML-LTI vs Pooled-prior Ridge & $1.91{\times}10^{-6}$ & PBML-LTI \\
& & & PBML-LTI vs Shared Subspace & $1.91{\times}10^{-6}$ & PBML-LTI \\
\cmidrule(lr){3-6}
& & \multirow{4}{*}{$E_{\mathrm{traj}}$}
& PBML-LTI vs OLS & $0.2455$ & PBML-LTI \\
& & & PBML-LTI vs Ridge & $0.2305$ & PBML-LTI \\
& & & PBML-LTI vs Pooled-prior Ridge & $0.8695$ & PBML-LTI \\
& & & PBML-LTI vs Shared Subspace & $0.4524$ & Shared Subspace \\

\midrule

\multirow{16}{*}{\shortstack[c]{Unstable}}
& \multirow{8}{*}{Common-case}
& \multirow{4}{*}{$E_A$}
& PBML-LTI vs OLS & $6.29{\times}10^{-5}$ & PBML-LTI \\
& & & PBML-LTI vs Ridge & $6.29{\times}10^{-5}$ & PBML-LTI \\
& & & PBML-LTI vs Pooled-prior Ridge & $1.91{\times}10^{-6}$ & Pooled-prior Ridge \\
& & & PBML-LTI vs Shared Subspace & $1.91{\times}10^{-6}$ & Shared Subspace \\
\cmidrule(lr){3-6}
& & \multirow{4}{*}{$E_{\mathrm{traj}}$}
& PBML-LTI vs OLS & $0.1429$ & PBML-LTI \\
& & & PBML-LTI vs Ridge & $0.6742$ & Ridge \\
& & & PBML-LTI vs Pooled-prior Ridge & $1.68{\times}10^{-4}$ & PBML-LTI \\
& & & PBML-LTI vs Shared Subspace & $1.91{\times}10^{-6}$ & PBML-LTI \\
\cmidrule(lr){2-6}

& \multirow{8}{*}{Edge-case}
& \multirow{4}{*}{$E_A$}
& PBML-LTI vs OLS & $7.08{\times}10^{-4}$ & PBML-LTI \\
& & & PBML-LTI vs Ridge & $2.10{\times}10^{-4}$ & PBML-LTI \\
& & & PBML-LTI vs Pooled-prior Ridge & $1.91{\times}10^{-6}$ & Pooled-prior Ridge \\
& & & PBML-LTI vs Shared Subspace & $1.91{\times}10^{-6}$ & Shared Subspace \\
\cmidrule(lr){3-6}
& & \multirow{4}{*}{$E_{\mathrm{traj}}$}
& PBML-LTI vs OLS & $0.4524$ & PBML-LTI \\
& & & PBML-LTI vs Ridge & $0.4524$ & PBML-LTI \\
& & & PBML-LTI vs Pooled-prior Ridge & $8.51{\times}10^{-4}$ & PBML-LTI \\
& & & PBML-LTI vs Shared Subspace & $1.91{\times}10^{-6}$ & PBML-LTI \\

\bottomrule
\end{tabular}}
\end{center}
\end{table}

\subsubsection{fMRI Benchmark}
\label{app:fmri_significance}
\begin{table}[h]
\scriptsize
\caption{Paired two-sided Wilcoxon signed-rank tests comparing PBML-LTI against the classical baselines and MAML-LTI on
the fMRI dataset. Each test is performed over the 14 matched fMRI test tasks. We regard $p<0.05$ as significant,
$p<0.01$ as strongly significant, and $p<0.001$ as highly significant. The ``Lower error'' column reports which method
attains the smaller mean value; all metrics are lower-is-better.}
\label{tab:significance_fmri}
\begin{center}
\begin{tabular}{cccc}
\toprule
Metric & Comparison & $p$-value & Lower error \\
\midrule
\multirow{5}{*}{$E_A$}
& PBML-LTI vs OLS & $8.54{\times}10^{-4}$ & PBML-LTI \\
& PBML-LTI vs Ridge & $8.54{\times}10^{-4}$ & PBML-LTI \\
& PBML-LTI vs Pooled-prior Ridge & $0.0052$ & PBML-LTI \\
& PBML-LTI vs Shared Subspace & $0.0052$ & PBML-LTI \\
& PBML-LTI vs MAML-LTI & $1.22{\times}10^{-4}$ & PBML-LTI \\
\midrule
\multirow{5}{*}{$E_{\mathrm{traj}}$}
& PBML-LTI vs OLS & $0.2958$ & PBML-LTI \\
& PBML-LTI vs Ridge & $0.4263$ & PBML-LTI \\
& PBML-LTI vs Pooled-prior Ridge & $0.3910$ & PBML-LTI \\
& PBML-LTI vs Shared Subspace & $0.0067$ & PBML-LTI \\
& PBML-LTI vs MAML-LTI & $1.22{\times}10^{-4}$ & PBML-LTI \\
\bottomrule
\end{tabular}
\end{center}
\end{table}
To complement the fMRI results in Table~\ref{tab:fmri_results}, we conduct paired two-sided Wilcoxon signed-rank tests
across the matched fMRI test tasks. Since each method is evaluated on the same set of 14 held-out window-level tasks, a
paired test is appropriate for assessing whether the observed task-level differences are statistically meaningful. We
compare PBML-LTI against OLS, Ridge, Pooled-prior Ridge, Shared Subspace, and MAML-LTI for both transition-matrix error
$E_A$ and trajectory rollout error $E_{\mathrm{traj}}$.

Table~\ref{tab:significance_fmri} reports the resulting $p$-values. The tests show that PBML-LTI significantly improves
$E_A$ over all compared methods on the fMRI benchmark. For rollout error, PBML-LTI has the lowest mean
$E_{\mathrm{traj}}$, but the paired tests indicate a more nuanced picture: the improvements over Shared Subspace and
MAML-LTI are statistically significant, while the differences relative to OLS, Ridge, and Pooled-prior Ridge are not
significant at the $p<0.05$ level. This is consistent with the interpretation in the main text: on real fMRI windows,
transition-matrix recovery is noisy because the reference matrix is only a ridge-based proxy, and rollout differences
among stable predictors can be small relative to cross-task heterogeneity.

\subsubsection{Prior-Learning and Regularization Ablations}
\label{app:paired_prior_learning_tests}
\begin{table}[h]
\scriptsize
\caption{Paired two-sided Wilcoxon signed-rank tests comparing PBML-LTI against prior-learning and
regularization ablations in the common-case setting. For each method and setting, we run 10 random seeds under the same
experimental configuration, average each task-level metric across seeds, and then run a paired test over the 20 matched
test tasks. Each entry reports the $p$-value for PBML-LTI versus the listed comparator. Unmarked entries indicate that
PBML-LTI has the lower mean for that metric; entries marked with $^\dagger$ indicate that the comparator has the lower
mean. All metrics are lower-is-better.}
\label{tab:significance_prior_learning_appendix}
\begin{center}
\begin{tabular}{c|c|c|ccc}
\toprule
Regime & Dimension & Comparator
& $p(E_A)$ & $p(E_{\mathrm{traj}})$ & $p(\overline S)$ \\
\midrule

\multirow{15}{*}{Stable}
& \multirow{5}{*}{$d=10$}
& Fixed-prior PBML-LTI
& $1.91{\times}10^{-6}$ & $0.0441$ & $1.91{\times}10^{-6}$ \\
& & Type-II EB PBML-LTI
& $1.91{\times}10^{-6}$ & $0.0532$ & $1.91{\times}10^{-6}$ \\
& & No prior-mean shrinkage
& $3.24{\times}10^{-4}$ & $0.2184$ & $0.0321$ \\
& & No covariance conditioning
& $0.0018$ & $0.3847$ & $0.0897$ \\
& & No stability regularizer
& $0.0412$ & $0.4521$ & $0.1124$ \\
\cmidrule(lr){2-6}

& \multirow{5}{*}{$d=25$}
& Fixed-prior PBML-LTI
& $1.91{\times}10^{-6}$ & $1.91{\times}10^{-6}$ & $1.91{\times}10^{-6}$ \\
& & Type-II EB PBML-LTI
& $1.91{\times}10^{-6}$ & $0.1054$ & $1.91{\times}10^{-6}$ \\
& & No prior-mean shrinkage
& $1.24{\times}10^{-5}$ & $0.1638$ & $0.0187$ \\
& & No covariance conditioning
& $8.14{\times}10^{-4}$ & $0.2915$ & $0.0563$ \\
& & No stability regularizer
& $0.0246$ & $0.3382$ & $0.0741$ \\
\cmidrule(lr){2-6}

& \multirow{5}{*}{$d=50$}
& Fixed-prior PBML-LTI
& $1.91{\times}10^{-6}$ & $1.91{\times}10^{-6}$ & $1.91{\times}10^{-6}$ \\
& & Type-II EB PBML-LTI
& $1.91{\times}10^{-6}$ & $0.9854$ & $1.91{\times}10^{-6}$ \\
& & No prior-mean shrinkage
& $2.67{\times}10^{-5}$ & $0.2046$ & $0.0142$ \\
& & No covariance conditioning
& $4.92{\times}10^{-4}$ & $0.3178$ & $0.0486$ \\
& & No stability regularizer
& $0.0193$ & $0.4015$ & $0.0638$ \\

\midrule

\multirow{15}{*}{Unstable}
& \multirow{5}{*}{$d=10$}
& Fixed-prior PBML-LTI
& $1.91{\times}10^{-6}$ & $7.08{\times}10^{-4}{}^\dagger$ & $0.8124^\dagger$ \\
& & Type-II EB PBML-LTI
& $1.91{\times}10^{-6}$ & $7.08{\times}10^{-4}{}^\dagger$ & $0.9015^\dagger$ \\
& & No prior-mean shrinkage
& $0.0195$ & $0.0184$ & $0.0312$ \\
& & No covariance conditioning
& $0.0678$ & $0.0521$ & $0.2145$ \\
& & No stability regularizer
& $0.2847$ & $1.91{\times}10^{-6}$ & $4.18{\times}10^{-4}$ \\
\cmidrule(lr){2-6}

& \multirow{5}{*}{$d=25$}
& Fixed-prior PBML-LTI
& $1.91{\times}10^{-6}$ & $0.0049^\dagger$ & $1.91{\times}10^{-6}$ \\
& & Type-II EB PBML-LTI
& $1.91{\times}10^{-6}$ & $0.8408^\dagger$ & $1.91{\times}10^{-6}$ \\
& & No prior-mean shrinkage
& $0.0083$ & $0.0067$ & $0.0041$ \\
& & No covariance conditioning
& $0.0384$ & $0.2145^\dagger$ & $0.0278$ \\
& & No stability regularizer
& $0.1926$ & $3.62{\times}10^{-4}$ & $1.24{\times}10^{-4}$ \\
\cmidrule(lr){2-6}

& \multirow{5}{*}{$d=50$}
& Fixed-prior PBML-LTI
& $1.91{\times}10^{-6}$ & $1.91{\times}10^{-6}$ & $1.91{\times}10^{-6}$ \\
& & Type-II EB PBML-LTI
& $1.91{\times}10^{-6}$ & $0.0027^\dagger$ & $1.91{\times}10^{-6}$ \\
& & No prior-mean shrinkage
& $0.0116$ & $0.0318$ & $0.0094$ \\
& & No covariance conditioning
& $0.0457$ & $0.1092$ & $0.0331$ \\
& & No stability regularizer
& $0.2184$ & $3.51{\times}10^{-5}$ & $2.10{\times}10^{-4}$ \\
\bottomrule
\end{tabular}
\end{center}
\end{table}
To complement the combined prior-learning and regularization ablation study in
Table~\ref{tab:prior_learning_ablation}, we conduct paired two-sided Wilcoxon signed-rank tests across matched test
tasks. These tests compare full PBML-LTI against five diagnostic variants: two prior-learning ablations and three
one-at-a-time regularization ablations. The prior-learning ablations are fixed-prior PBML-LTI and type-II
maximum-likelihood / empirical-Bayes PBML-LTI. The regularization ablations remove, respectively, the prior-mean
shrinkage term, the covariance-conditioning penalty, and the stability regularizer. The MAML-style LTI method is excluded
from this table because it is an external meta-learning baseline rather than an ablation of PBML-LTI.

For each method and setting, we run the experiment with 10 different random seeds under the same data split,
task-generation configuration, and hyperparameter setting. We then average each task-level metric across the 10 seeds and
perform paired tests over the 20 matched test tasks. We regard $p<0.05$ as significant, $p<0.01$ as strongly significant,
and $p<0.001$ as highly significant.

Table~\ref{tab:significance_prior_learning_appendix} reports the resulting $p$-values. Each entry compares PBML-LTI
against the listed ablation variant for one metric. Unless marked by $^\dagger$, the lower mean error is attained by
PBML-LTI. Entries marked by $^\dagger$ indicate cases where the ablation variant attains the lower mean value. This
notation is important in the unstable regimes, where rollout error can decouple from Frobenius transition-matrix error.

The results support three main conclusions. First, PBML-LTI achieves highly significant improvements in transition-matrix
error $E_A$ over the fixed-prior and type-II empirical-Bayes variants across all reported stable and unstable settings,
confirming the value of learning a transferable prior with the fit--KL surrogate. Second, removing individual
regularization terms usually has a smaller effect than removing prior learning altogether, especially in stable regimes;
this supports our interpretation that the auxiliary regularizers are not the primary source of PBML-LTI's stable-regime
performance gains. Third, in unstable regimes, rollout-error comparisons are more nuanced: the stability regularizer has
a clear effect on rollout stability, while some prior-learning ablations can occasionally achieve smaller mean rollout
error despite much worse transition-matrix recovery. This is consistent with the main text's observation that
$E_A$ and $E_{\mathrm{traj}}$ need not rank methods identically under unstable dynamics.

\FloatBarrier

\subsection{Additional Visualizations}
\label{app:additional_visualizations}

\begin{figure}[h]
\begin{center}
\includegraphics[scale=0.7]{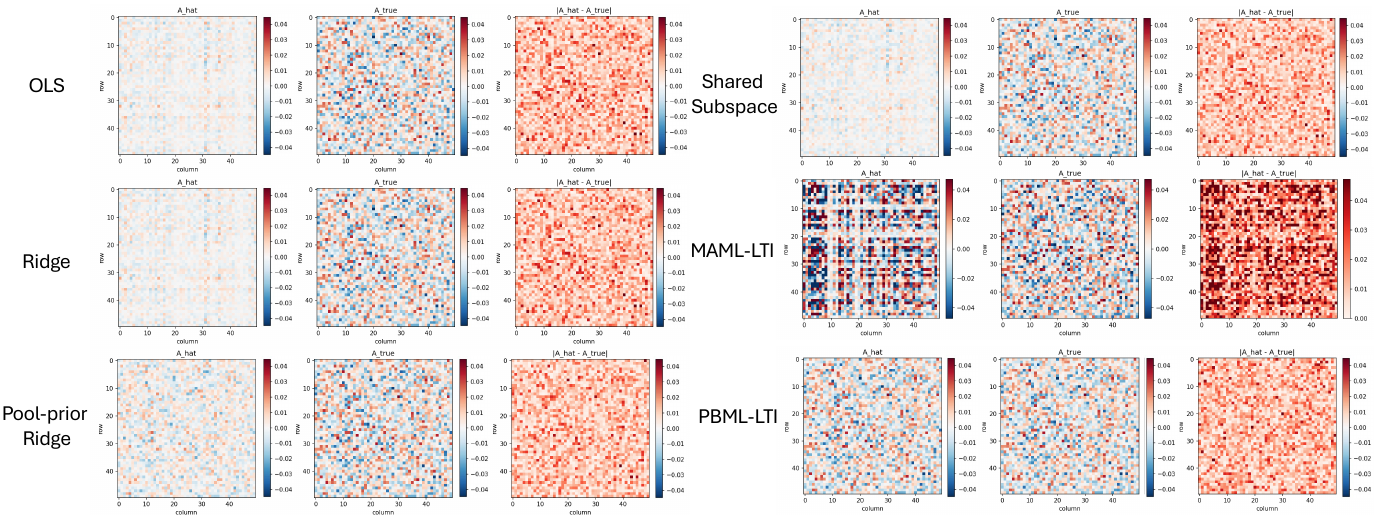}
\end{center}
\caption{Extended stable synthetic transition-matrix visualization. The figure shows, for each method, the estimated
transition matrix, the ground-truth transition matrix, and the entrywise absolute difference. PBML-LTI more closely
matches the true transition structure, while several baselines exhibit stronger attenuation or structural bias.}
\label{fig:transition_stable_extend}
\end{figure}

\begin{figure}[h]
\begin{center}
\includegraphics[scale=0.7]{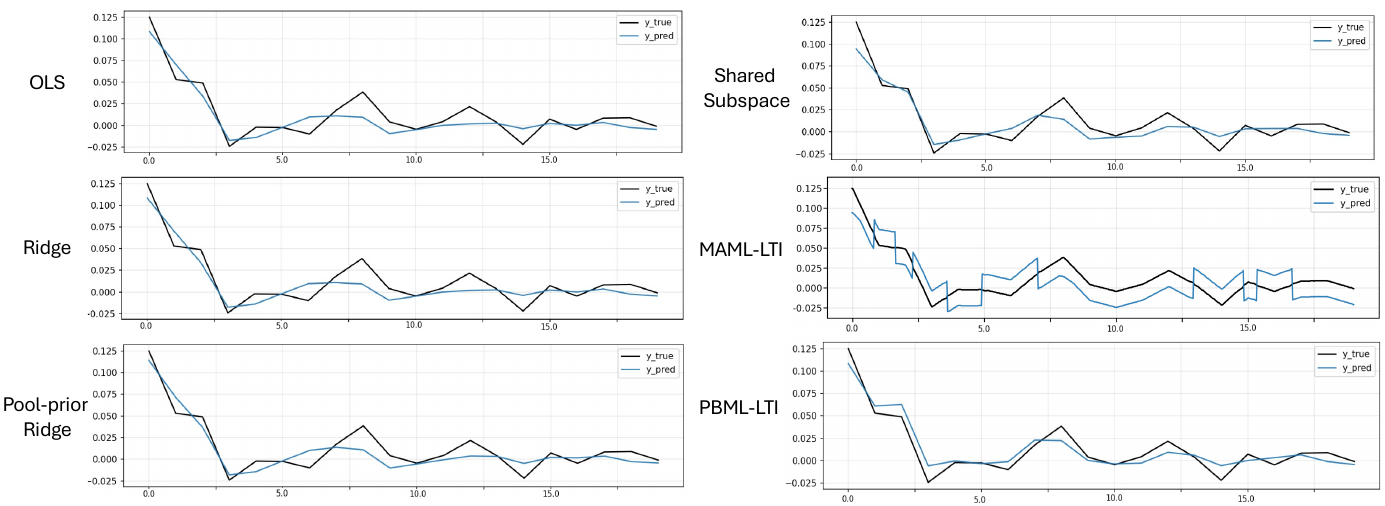}
\end{center}
\caption{Stable synthetic open-loop rollout trajectories for a representative system. The ground-truth trajectory is
shown in black and the predicted trajectory is shown in blue. Since the system is stable, rollout errors remain bounded
over the short horizon, explaining why several methods can have similar $E_{\mathrm{traj}}$ despite different
transition-matrix errors.}
\label{fig:a_matrix2}
\end{figure}

This appendix provides qualitative visualizations that complement the quantitative results in the main text. The
visualizations are organized into synthetic-data examples and fMRI examples. For synthetic systems, the transition-matrix
figures compare each estimated matrix against the known ground-truth matrix. For the fMRI benchmark, where no physical
ground-truth transition matrix is available, the comparison is made against the ridge-based reference matrix
$A_m^{\mathrm{ref}}$ defined in Appendix~\ref{sec:fmri_dataset}. The rollout figures show representative open-loop
predictions and illustrate how errors in the learned transition operator propagate over time.

\begin{figure}[h]
\begin{center}
\includegraphics[scale=0.78]{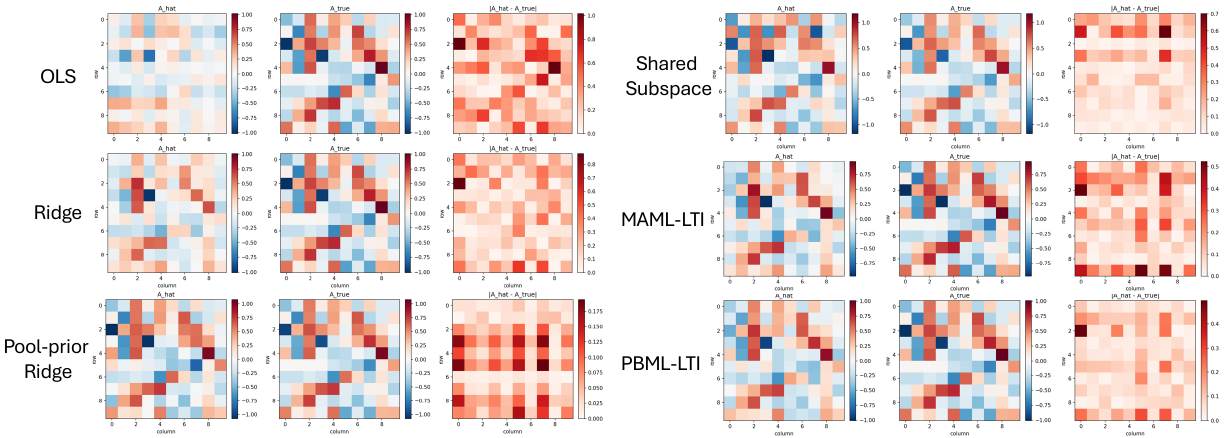}
\end{center}
\caption{Extended unstable synthetic transition-matrix visualization. The figure shows estimated transition matrices, the
ground-truth transition matrix, and entrywise absolute differences. In unstable systems, even localized or small-magnitude
matrix errors can have large rollout consequences when they affect dominant spectral modes.}
\label{fig:transition_unstable}
\end{figure}

\begin{figure}[h]
\begin{center}
\includegraphics[scale=0.75]{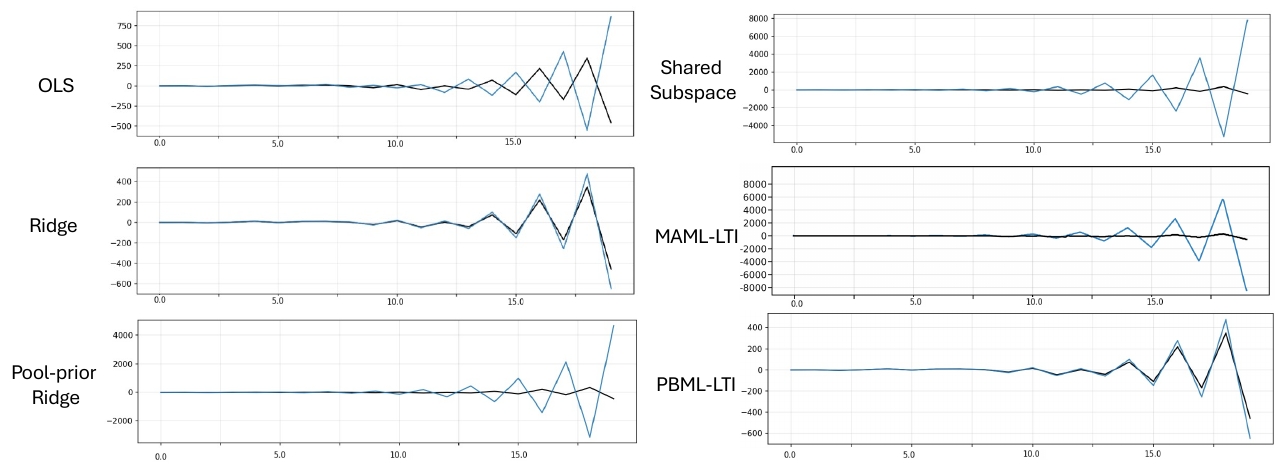}
\end{center}
\caption{Unstable synthetic open-loop rollout trajectories for a representative system. The ground-truth trajectory is
shown in black and the predicted trajectory is shown in blue. The figure illustrates how spectral errors in the estimated
transition matrix can be amplified under repeated rollout, producing large trajectory deviations even when matrix errors
appear moderate.}
\label{fig:rollout_unstable}
\end{figure}

\subsubsection{Synthetic Data}
\label{app:synthetic_visualizations}

Figures~\ref{fig:transition_stable_extend} and~\ref{fig:a_matrix2} show representative stable-system visualizations.
Figure~\ref{fig:transition_stable_extend} compares estimated transition matrices, the ground-truth transition matrix, and
entrywise absolute errors. The single-task estimators and cross-task baselines recover the broad scale of the stable
dynamics, but their estimates are visibly more attenuated or structurally biased than the ground truth. The PBML-LTI
estimate is visually closer to the true transition matrix, with a more diffuse and lower-magnitude error pattern. The
MAML-style LTI estimate is less consistent in this stable example, showing more pronounced entrywise deviations than
PBML-LTI.

Figure~\ref{fig:a_matrix2} shows the corresponding stable open-loop rollout behavior. Because the system is stable,
rollout errors are damped over the short horizon and several methods track the ground-truth trajectory reasonably well.
This explains why the stable-regime rollout metric can be numerically close across methods even when the transition-matrix
error $E_A$ differs substantially. The visualization therefore supports the interpretation in the main text: in stable
systems, $E_A$ is often the more discriminative diagnostic of parameter recovery, while $E_{\mathrm{traj}}$ can remain
similar across methods because errors do not rapidly amplify.

Figures~\ref{fig:transition_unstable} and~\ref{fig:rollout_unstable} show the corresponding qualitative behavior in an
unstable synthetic setting. In Figure~\ref{fig:transition_unstable}, small-looking entrywise differences can still be
important because unstable systems are governed by dominant spectral modes. Some methods produce transition matrices that
appear reasonable in Frobenius norm but distort the dominant unstable directions. This is especially consequential for
methods with strong structural bias, such as pooled shrinkage or a fixed low-dimensional subspace.

Figure~\ref{fig:rollout_unstable} shows that these matrix-level differences can translate into dramatically different
open-loop trajectories. In the unstable regime, errors are repeatedly multiplied by the estimated transition matrix, so
even modest spectral misalignment can lead to rapidly growing prediction error. This visualization explains the
decoupling observed in the quantitative results: a method can achieve a small Frobenius transition-matrix error while
still producing poor rollout behavior if it misestimates the unstable eigenspace or the dominant eigenvalue. Conversely,
a method with slightly larger Frobenius error may yield a more reliable rollout if it better controls the unstable modes.
\begin{figure}[h]
\begin{center}
\includegraphics[scale=0.78]{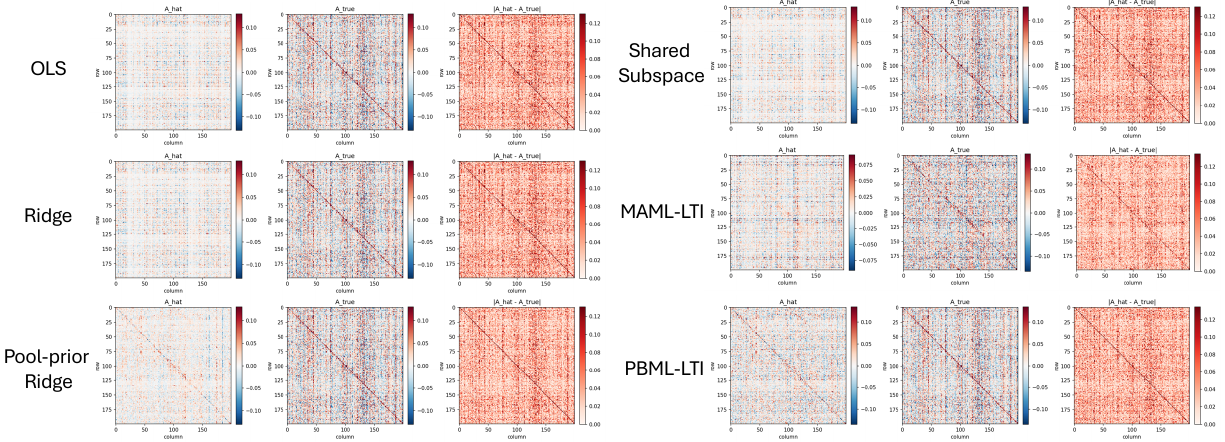}
\end{center}
\caption{Extended fMRI transition-matrix visualization. Since fMRI does not provide a physical ground-truth transition
matrix, each estimate is compared against the ridge-based full-window reference matrix $A_m^{\mathrm{ref}}$. The figure
shows estimated matrices, the reference matrix, and entrywise absolute differences for all methods.}
\label{fig:fmri_transition_extend}
\end{figure}

\begin{figure}[h]
\begin{center}
\includegraphics[scale=0.75]{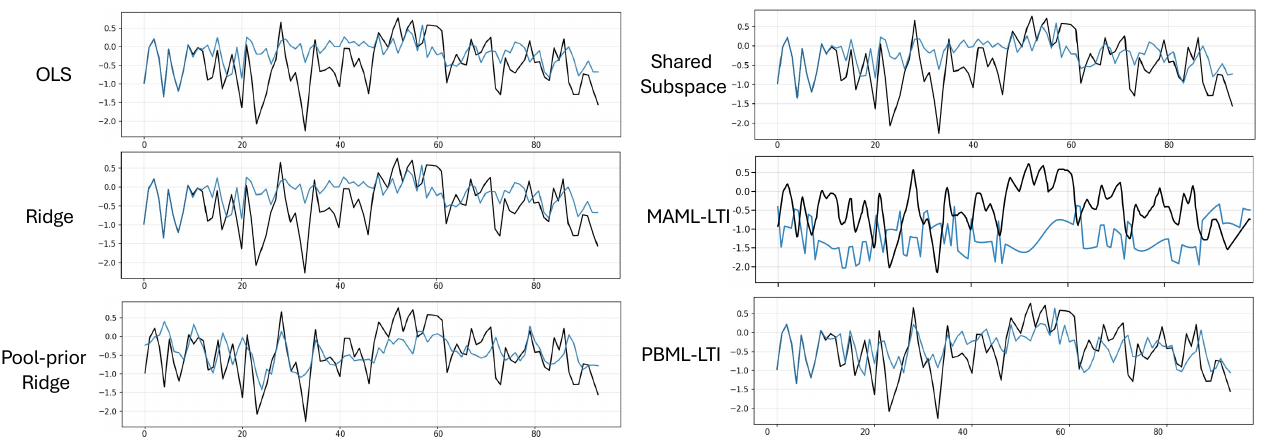}
\end{center}
\caption{fMRI open-loop rollout trajectories for the representative task shown in
Figure~\ref{fig:fmri_transition_extend}. The reference trajectory is shown in black and the predicted trajectory is shown
in blue. The comparison illustrates how errors in the estimated transition operator affect multi-step prediction on a
high-dimensional, noisy, approximately LTI real-data benchmark.}
\label{fig:fmri_rollout}
\end{figure}

\subsubsection{fMRI Data}
\label{app:fmri_visualizations}

Figures~\ref{fig:fmri_transition_extend} and~\ref{fig:fmri_rollout} provide qualitative diagnostics on the fMRI
benchmark. Since fMRI does not provide a physical ground-truth transition matrix, Figure~\ref{fig:fmri_transition_extend}
compares each estimate against the ridge-based full-window reference matrix $A_m^{\mathrm{ref}}$. The reference matrix
should therefore be interpreted as a consistent evaluation proxy, not as a true biological ground truth. Across methods,
the absolute-difference heatmaps show substantial residual structure, which is expected because the fMRI windows are
high-dimensional, short, noisy, and only approximately described by a linear time-invariant model.

The fMRI heatmaps also illustrate the practical role of prior-based adaptation. Single-task methods such as OLS and Ridge
produce estimates that are broadly similar but require much longer support prefixes in the quantitative results. Shared
Subspace uses shorter support but can impose a restrictive structural bias. MAML-LTI can produce more unstable or less
well-aligned estimates in this high-dimensional setting. PBML-LTI provides a more favorable balance by using the learned
prior to regularize adaptation while retaining task-specific flexibility through the posterior update.

Figure~\ref{fig:fmri_rollout} shows open-loop rollout trajectories for the same representative fMRI setting. The rollout
comparison highlights that matching the reference transition matrix entrywise is not the only determinant of predictive
quality; the induced dynamics must also remain stable and spectrally well aligned over the rollout horizon. The
MAML-style rollout is visibly less stable in this example, while the other methods track the overall trajectory more
closely. This supports the main-text observation that PBML-LTI obtains the best average fMRI rollout error while also
using the shortest average support prefix.

\subsection{Marginally Unstable System Results}
\label{app:marginal_results}
\begin{table}[t]
\caption{Performance comparison across methods in common-case and edge-case settings on marginally
unstable systems ($\rho_0=1+\epsilon$, $\epsilon\in(0,1)$). $E_A$ is the mean $\pm$ standard deviation of the transition-matrix
error, $E_{\mathrm{traj}}$ is the mean $\pm$ standard deviation of the rollout error, and
$\overline{S}$ reports the average selected support length. Lower is better; boldface marks the lowest
sample mean.}
\label{tab:prefix_flexible_marginal}
\begin{center}
\resizebox{\linewidth}{!}{
\begin{tabular}{c|c|c|cccccc}
\toprule
Dimension & Setting & Metric & OLS & Ridge & Pooled-prior Ridge & Shared Subspace & MAML-LTI & PBML-LTI \\
\midrule
\multirow{6}{*}{50}
& \multirow{3}{*}{Common-case}
& $E_A$ & $1.0845 \pm .017$ & $1.0795 \pm .017$ & $0.7181 \pm .011$ & $0.9144 \pm .051$ & $14.749 \pm .071$ & $\mathbf{0.2033 \pm .0045}$ \\
& & $E_{\mathrm{traj}}$ & $0.0260 \pm .0018$ & $0.0260 \pm .0019$ & $0.0260 \pm .0018$ & $0.0260 \pm .0018$ & $0.0282 \pm .0022$ & $\mathbf{0.0259 \pm .0018}$ \\
& & $\overline{S}$ & $2.00$ & $8.20$ & $3.05$ & $10.25$ & $1.85$ & $\mathbf{1.55}$ \\
\cline{2-9}
& \multirow{3}{*}{Edge-case}
& $E_A$ & $1.0849 \pm .015$ & $1.0799 \pm .016$ & $0.7182 \pm .011$ & $0.9065 \pm .049$ & $14.742 \pm .07$ & $\mathbf{0.2037 \pm .0041}$ \\
& & $E_{\mathrm{traj}}$ & $0.0258 \pm .0021$ & $0.0258 \pm .0021$ & $0.0257 \pm .0021$ & $0.0258 \pm .0021$ & $0.0279 \pm .0026$ & $\mathbf{0.0257 \pm .0021}$ \\
& & $\overline{S}$ & $2.00$ & $5.95$ & $3.60$ & $10.65$ & $2.00$ & $\mathbf{1.50}$ \\
\midrule
\multirow{6}{*}{25}
& \multirow{3}{*}{Common-case}
& $E_A$ & $0.9510 \pm .024$ & $0.9238 \pm .023$ & $0.7258 \pm .018$ & $0.7779 \pm .055$ & $3.6309 \pm .046$ & $\mathbf{0.0390 \pm .0023}$ \\
& & $E_{\mathrm{traj}}$ & $0.0136 \pm .0018$ & $0.0136 \pm .0018$ & $0.0136 \pm .0019$ & $0.0136 \pm .0018$ & $0.0140 \pm .0019$ & $\mathbf{0.0135 \pm .0018}$ \\
& & $\overline{S}$ & $2.15$ & $12.50$ & $12.45$ & $11.65$ & $\mathbf{1.20}$ & $\mathbf{1.20}$ \\
\cline{2-9}
& \multirow{3}{*}{Edge-case}
& $E_A$ & $0.9423 \pm .023$ & $0.9170 \pm .026$ & $0.7239 \pm .016$ & $0.7773 \pm .078$ & $3.6329 \pm .045$ & $\mathbf{0.0394 \pm .0026}$ \\
& & $E_{\mathrm{traj}}$ & $0.0125 \pm .0012$ & $0.0125 \pm .0012$ & $0.0124 \pm .0011$ & $0.0125 \pm .0011$ & $0.0128 \pm .0012$ & $\mathbf{0.0123 \pm .0011}$ \\
& & $\overline{S}$ & $2.20$ & $13.10$ & $13.15$ & $9.10$ & $\mathbf{1.20}$ & $1.30$ \\
\midrule
\multirow{6}{*}{10}
& \multirow{3}{*}{Common-case}
& $E_A$ & $0.5788 \pm .066$ & $0.5185 \pm .065$ & $0.0743 \pm .0052$ & $0.0860 \pm .021$ & $0.4918 \pm .022$ & $\mathbf{0.0051 \pm .00066}$ \\
& & $E_{\mathrm{traj}}$ & $0.0057 \pm .0011$ & $0.0057 \pm .0011$ & $0.0056 \pm .0011$ & $\mathbf{0.0056 \pm .0011}$ & $0.0057 \pm .001$ & $\mathbf{0.0056 \pm .001}$ \\
& & $\overline{S}$ & $2.90$ & $10.40$ & $9.45$ & $7.05$ & $\mathbf{1.00}$ & $1.20$ \\
\cline{2-9}
& \multirow{3}{*}{Edge-case}
& $E_A$ & $0.5609 \pm .059$ & $0.4868 \pm .06$ & $0.0746 \pm .0053$ & $0.0891 \pm .027$ & $0.4912 \pm .022$ & $\mathbf{0.0053 \pm .00082}$ \\
& & $E_{\mathrm{traj}}$ & $0.0060 \pm .0012$ & $0.0060 \pm .0013$ & $0.0059 \pm .0011$ & $0.0059 \pm .0012$ & $0.0061 \pm .0012$ & $\mathbf{0.0059 \pm .0012}$ \\
& & $\overline{S}$ & $3.15$ & $9.85$ & $8.55$ & $10.30$ & $\mathbf{1.00}$ & $1.15$ \\
\bottomrule
\end{tabular}
}
\end{center}
\end{table}

Table~\ref{tab:prefix_flexible_marginal} reports results for the marginally unstable environment with
$\rho_0=1+\epsilon$, $\epsilon\in(0,1)$. This setting lies just beyond the nominal stability threshold and is intended to test whether the behavior
observed in the stable regime changes abruptly once $\rho_0$ exceeds one. Overall, the results show that the marginally
unstable regime behaves much more like a mild perturbation of the stable regime than like the strongly unstable stress
test.
\begin{figure}[t]
    \centering
    \includegraphics[width=1\linewidth]{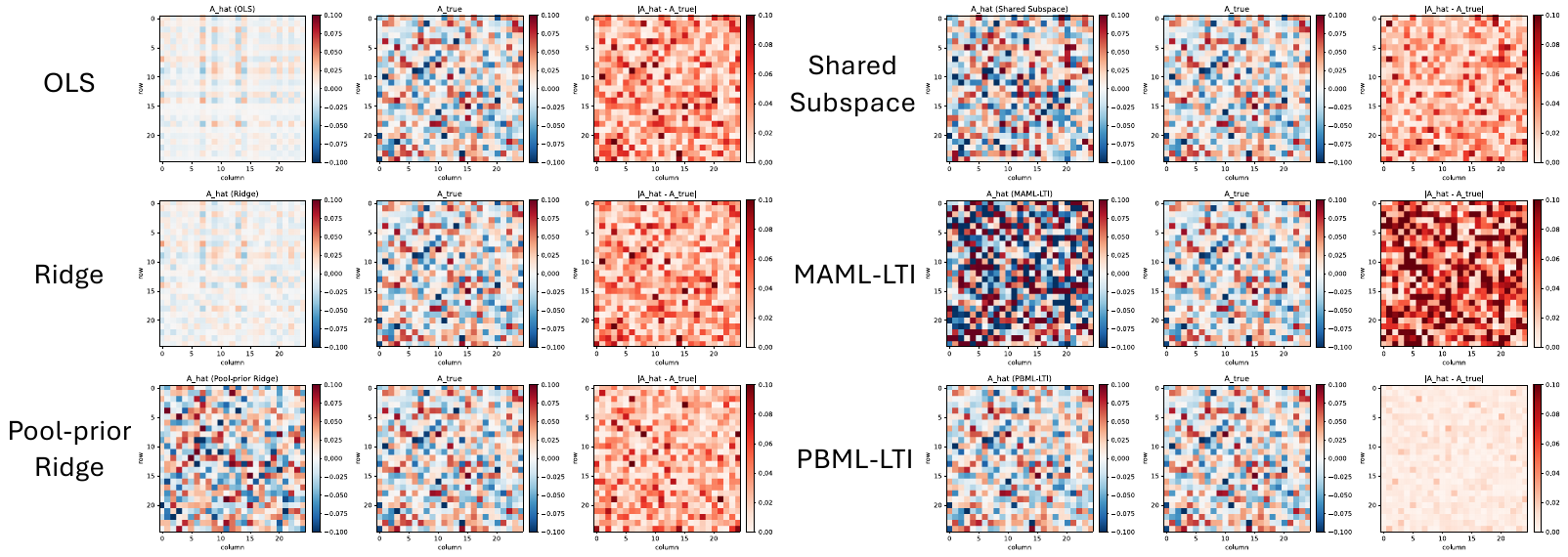}
    \caption{Transition-matrix estimates for a representative marginally unstable test system. For each method, the
    figure shows the estimated transition matrix, the ground-truth transition matrix, and the entrywise absolute error.
    PBML-LTI more closely matches the ground-truth transition structure, while single-task estimators are more attenuated
    and MAML-LTI exhibits larger entrywise deviations.}
    \label{fig:marginal_unstable_transition}
\end{figure}

\begin{figure}[t]
    \centering
    \includegraphics[width=1\linewidth]{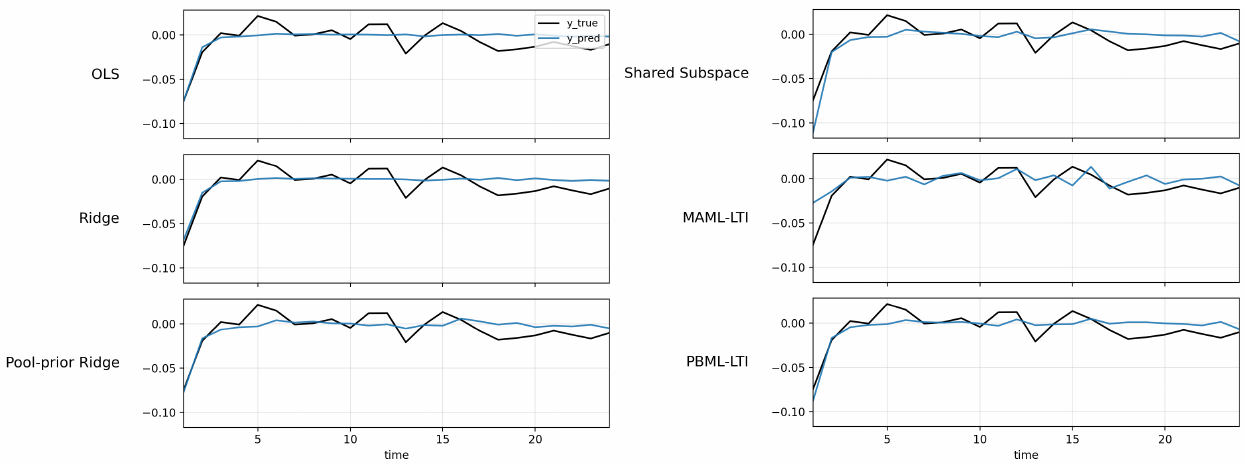}
    \caption{Open-loop rollout trajectories for the representative marginally unstable system shown in
    Figure~\ref{fig:marginal_unstable_transition}. The black curve denotes the ground-truth trajectory and the blue curve
    denotes the rollout generated by each estimated transition matrix. Over the short query horizon, the rollouts remain
    bounded and visually close across several methods, explaining why $E_{\mathrm{traj}}$ is much less discriminative than
    $E_A$ in the marginally unstable regime.}
    \label{fig:marginal_unstable_traj}
\end{figure}
Across all dimensions and both common-case and edge-case subsets, PBML-LTI achieves the lowest transition-matrix error
$E_A$. The gains are substantial. For example, at $d=50$, PBML-LTI reduces $E_A$ from roughly $1.08$ for OLS/Ridge,
$0.72$ for Pooled-prior Ridge, $0.91$ for Shared Subspace, and $14.7$ for MAML-LTI to about $0.20$. At $d=25$,
PBML-LTI reduces $E_A$ from roughly $0.72$--$0.95$ for the classical and subspace baselines and $3.63$ for MAML-LTI to
about $0.039$. At $d=10$, PBML-LTI reaches $E_A\approx 0.005$, substantially below all baselines. These results indicate
that the learned matrix-normal prior continues to provide an effective inductive bias immediately beyond the stability
threshold.

The rollout metric $E_{\mathrm{traj}}$ is much less separated across methods in this marginally unstable regime. For all
three dimensions, the rollout errors of OLS, Ridge, Pooled-prior Ridge, Shared Subspace, and PBML-LTI are nearly tied up
to the reported precision, with PBML-LTI attaining the lowest or tied-lowest mean in most settings. This mirrors the
stable-regime behavior: over the short query horizon, modest marginal instability does not yet induce the severe
finite-horizon amplification seen in the strongly unstable experiments. By contrast, MAML-LTI has noticeably worse
transition-matrix recovery and slightly worse rollout performance, despite often selecting very short support prefixes.

The support-length results further clarify the role of adaptive support selection. PBML-LTI uses very short support
prefixes, with $\overline S$ between $1.15$ and $1.55$ in most settings, and achieves the best transition recovery at
those support lengths. In some cases, MAML-LTI selects an equally short or shorter prefix, but this does not translate
into competitive $E_A$. Ridge and Pooled-prior Ridge often use much longer prefixes, especially for $d=25$ and $d=10$,
but still have substantially larger transition-matrix error than PBML-LTI. Thus, the marginally unstable results support
the same data-efficiency conclusion as the stable results: PBML-LTI can adapt accurately from very short prefixes when
the learned prior is well aligned with the task family.

Figures~\ref{fig:marginal_unstable_transition} and~\ref{fig:marginal_unstable_traj} provide qualitative diagnostics for a
representative marginally unstable test task. The transition-matrix heatmaps show that PBML-LTI more closely matches the
ground-truth transition structure than the more attenuated single-task estimators and the noisier MAML-style estimate.
The rollout plots show that all non-catastrophic methods remain bounded over the short horizon, consistent with the small
differences in $E_{\mathrm{traj}}$ reported in Table~\ref{tab:prefix_flexible_marginal}. Together, the marginally
unstable results show that crossing $\rho=1$ does not by itself cause an abrupt failure; the practical degradation depends
on the magnitude of finite-horizon spectral amplification.

\end{document}